\documentclass[10pt,fullpage,letterpaper]{article}

\usepackage{amsmath,amsfonts,amsthm,amssymb,amsopn,bm,color,mathtools,booktabs,multirow,appendix,caption, dsfont}
\usepackage{algorithm,algorithmic}

\usepackage[utf8]{inputenc} 
\usepackage[T1]{fontenc}    
\usepackage{hyperref}       
\usepackage{url}            
\usepackage{booktabs}       
\usepackage{amsfonts}       
\usepackage{nicefrac}       
\usepackage{microtype}      
\usepackage{xcolor}         
\usepackage{enumitem}
\allowdisplaybreaks

\newtheorem{definition}{Definition}
\newtheorem{lemma}{Lemma}
\newtheorem{theorem}{Theorem}

\newtheorem{proposition}{Proposition}

\def\E{\mathbb{E}}
\def\P{\mathbb{P}}
\def\R{\mathbb{R}}
\def\1{\bm{1}}
\def\HighProbConst{c_0}

\newcommand{\mc}[1]{\mathcal{#1}}

\newcommand{\Sp}[1]{\left(#1\right)}
\newcommand{\Mp}[1]{\left[#1\right]}
\newcommand{\Bp}[1]{\left\{#1\right\}}
\newcommand{\abs}[1]{\left|#1\right|}
\newcommand{\Norm}[1]{\left\|#1\right\|}

\newcommand{\inner}[1]{\left\langle#1\right\rangle}
\newcommand{\matenv}[1]{\left[\begin{matrix}#1\end{matrix}\right]}

\DeclareMathOperator*{\argmin}{argmin}
\DeclareMathOperator*{\argmax}{argmax}

\title{Selective Sampling for Online Best-arm Identification}

\author{%
Romain Camilleri\thanks{Equal contribution. Alphabetical order.} , 
Zhihan Xiong\footnotemark[\value{footnote}] , Maryam Fazel, Lalit Jain, Kevin Jamieson\\
University of Washington, Seattle, WA\\
\url{{camilr, zhihanx, mfazel, lalitj, jamieson}@uw.edu}
}
\date{}

\begin{document}

\maketitle

\begin{abstract}
  This work considers the problem of \emph{selective-sampling for best-arm identification}. Given a set of potential options $\mathcal{Z}\subset\mathbb{R}^d$, a learner aims to compute with probability greater than $1-\delta$, $\arg\max_{z\in \mathcal{Z}} z^{\top}\theta_{\ast}$ where $\theta_{\ast}$ is unknown. At each time step, a potential measurement $x_t\in \mathcal{X}\subset\mathbb{R}^d$ is drawn IID and the learner can either choose to take the measurement, in which case they observe a noisy measurement of $x^{\top}\theta_{\ast}$, or to abstain from taking the measurement and wait for a potentially more informative point to arrive in the stream. 
  Hence the learner faces a fundamental trade-off between the number of labeled samples they take and when they have collected enough evidence to declare the best arm and stop sampling. The main results of this work precisely characterize this trade-off between labeled samples and stopping time and provide an algorithm that nearly-optimally achieves the minimal label complexity given a desired stopping time. In addition, we show that the optimal decision rule has a simple geometric form based on deciding whether a point is in an ellipse or not. Finally, our framework is general enough to capture binary classification improving upon previous works. 
\end{abstract}

\section{Introduction}
\setcounter{footnote}{0} 
In this work we consider \emph{selective sampling for online best-arm identification}. In this setting, at every time step $t=1,2,\dots$, Nature reveals a potential measurement $x_t\in \mc{X}\subset \mathbb{R}^d$ to the learner. The learner can choose to either \textit{query} $x_t$ ($\xi_t = 1$) or abstain ($\xi_t = 0$) and immediately move on to the next time. If the learner chooses to take a query ($\xi_t=1$), then Nature reveals a noisy linear measurement of an unknown $\theta_{\ast}\in \mathbb{R}^d$, i.e.  $y_t = \langle x_t, \theta_* \rangle + \epsilon_t$ where $\epsilon_t$ is mean zero sub-Gaussian noise. Before the start of the game, the learner has knowledge of a set $\mc{Z}\subset \mathbb{R}^d$. The objective of the learner is to identify $z_* := \arg\max_{z \in \mc{Z}} \langle z, \theta_* \rangle$ with probability at least $1-\delta$ at a learner specified stopping time $\mc{U}$. 
It is desirable to minimize both the stopping time $\mc{U}$ which counts the total number of unlabeled or labeled queries and the number of labeled queries requested $\mc{L} := \sum_{t=1}^{\mc{U}} \1\{ \xi_t = 1 \}$.
In this setting, at each time $t$ the learner must make the decision of whether to accept the available measurement $x_t$, or abstain and wait for an even more informative measurement. 
While abstention may result in a smaller total labeled sample complexity $\mc{L}$, the stopping time  $\mc{U}$ may be very large. 
This paper characterizes the set of feasible pairs $(\mc{U},\mc{L})$ that are necessary and sufficient to identify $z_*$ with probability at least $1-\delta$ when $x_t$ are drawn IID at each time $t$ from a distribution $\nu$.
Moreover, we propose an algorithm that nearly obtains the minimal information theoretic label sample complexity $\mc{L}$ for any desired unlabeled sample complexity $\mc{U}$.

While characterizing the sample complexity of selective sampling for online best arm identification is the primary theoretical goal of this work, the study was initially motivated by fundamental questions about how to optimally trade-off the value of information versus time. 
Even for this idealized linear setting, it is far from obvious a priori what an optimal decision rule $\xi_t$ looks like and if it can even be succinctly described, or if it is simply the solution to an opaque optimization problem. 
Remarkably, we show that for every feasible, optimal operating pair $(\mc{U},\mc{L})$ there exists a matrix $A \in \R^{d \times d}$
such that the optimal decision rule takes on the form $\xi_t = \1\{ x^\top A x \geq 1 \}$ when $x_t \sim \nu$ iid. 
The fact that for any smooth distribution $\nu$ the decision rule is a hard decision equivalent to $x_t$ falling outside a fixed ellipse or not, and not a stochastic rule that varies complementarily with the density of $\nu$ over space is perhaps unexpected.

To motivate the problem description, suppose on each day $t=1,2,\dots$ a food blogger posts the \emph{Cocktail of the Day} with a recipe described by a feature vector $x_t \in \R^d$. 
You have the ingredients (and skills) to make any possible cocktail in the space of all cocktails $\mc{Z}$, but you don't know which one you'd like the most, i.e., $z_* := \arg\max_{z \in \mc{Z}} \langle z, \theta_* \rangle$, where $\theta_{\ast}$ captures your preferences over cocktail recipes. 
You decide to use the \emph{Cocktail of the Day} to inform your search.
That is, each day you are presented with the cocktail recipe $x_t \in \R^d$, and if you choose to make it ($\xi_t = 1$) you observe your preference for the cocktail $y_t$ with $\E[y_t] = \langle x_t,\theta_*\rangle$.  Of course, making cocktails can get costly, so you don't want to make each day's cocktail, but rather you will only make the cocktail if $x_t$ is informative about $\theta_{\ast}$ (e.g., uses a new combination of ingredients). At the same time, waiting too many days before making the next cocktail of the day may mean that you never get to learn (and hence drink) the cocktail $z_{\ast}$ you like best. 
The setting above is not limited to cocktails, but rather naturally generalizes to discovering the efficacy of drugs and other therapeutics where blood and tissue samples come to the clinic in a stream and the researcher has to choose whether to take a potentially costly measurement.

Our results hold for arbitrary $\theta_* \in \R^d$, sets $\mc{X} \subset \R^d$ and $\mc{Z} \subset \R^d$, and measures $\nu \in \triangle_{\mc{X}}$\footnote{We denote the set of probability measures over $\mc{X}$ as $\triangle_{\mc{X}}$.} for which we assume $x_t \sim \nu$ is drawn IID.
The assumption that each $x_t$ is IID allows us to make very strong statements about optimality. 
To summarize, our contributions are as follows:
\begin{itemize}[leftmargin=5pt]
\setlength\itemsep{.1em}
            \item We present fundamental limits on the trade-off between the amount of unlabelled data and labelled data in the form of (the first) information theoretic lower bounds for selective sampling problems that we are aware of. 
            Naturally, they say that there is an absolute minimum amount of unlabelled data that is necessary to solve the problem, but then for any amount of unlabelled data beyond this critical value, the bounds say that the amount of labelled data must exceed some value as a function of the unlabelled data used. 
            \item We propose an algorithm that nearly matches the lower bound at all feasible trade-off points in the sense that given any unlabelled data budget that exceeds the critical threshold, the algorithm takes no more labels than the lower bound suggests. Thus, the upper and lower bounds sketch out a curve of all possible operating points, and the algorithm achieves any point on this curve.
            \item We characterize the optimal decision rule of whether to take a sample or not, based on any critical point is a simple test: Accept $x_t \in \mathbb{R}^d$ if $x_t^{\top} A x_t \geq 1$ for some matrix $A$ that depends on the desired operating point and geometry of the task. Geometrically, this is equivalent to $x_t$ falling inside or outside an ellipsoid. 
            \item Our framework is also general enough to capture binary classification, and consequently, we prove results there that improve upon state of the art. 
        \end{itemize}

\subsection{Related Work}

\textbf{Selective Sampling in the Streaming Setting:} Online prediction, the setting in which the selective sampling framework was introduced, is a closely related problem to the one studied in this paper and enjoys a much more developed literature \cite{cesa2009robust, dekel2012selective, agarwal2013selective, chen2021active}. 
In the linear online prediction setting, for $t=1,2,\dots$ Nature reveals $x_t \in \R^d$, the learner predicts $\widehat{y}_t$ and incurs a loss $\ell(\widehat{y_t},y_t)$, and then the learner decides whether to observe $y_t$ (i.e., $\xi_t=1$)  or not ($\xi_t=0$), where $y_t$ is a label generated by a composition of a known link function with a linear function of $x_t$. 
For example, in the classification setting \cite{agarwal2013selective, cesa2009robust, dekel2012selective}, one setting assumes $y_t \in \{-1,1\}$ with $\E[y_t | x_t] = \langle x_t ,\theta_* \rangle$ for some unknown $\theta_* \in \R^d$, and $\ell(\widehat{y}_t,y_t) = \1\{ \widehat{y}_t \neq y_t \}$.
In the regression setting \cite{chen2021active}, one observes $y_t \in [-1,1]$ with $\E[y_t | x_t] = \langle x_t ,\theta_* \rangle$ again, and $\ell(\widehat{y}_t, y_t) = (\widehat{y}_t - y_t)^2$.
After any amount of time $\mc{U}$, the learner is incentivized to minimize both the amount of requested labels $\sum_{t=1}^{\mc{U}} \1\{ \xi_t = 1 \}$ and the cumulative loss $\sum_{t=1}^{\mc{U}} \ell(y_t, \widehat{y}_t)$ (or some measure of regret which compares to predictions using the unknown $\theta_*$).
If every label $y_t$ is requested then $\mc{L}=\mc{U}$ and this is just the classical online learning setting. 

These works give a guarantee on the regret and labeled points taken in terms of the hardness of the stream relative to a learner which would see the label at every time. 
Most do not give the learner the ability to select an operating point that provides a trade-off between the amount of unlabeled versus labeled data taken.
Those few works that propose algorithms that do provide this functionality do not provide lower bounds that match their given upper bounds, leaving it unclear whether their algorithm optimally negotiates this trade-off. 
In contrast, our work fully characterizes the trade-off between the amount of unlabeled and labeled data through an information-theoretic lower bound and a matching upper bound. 
Specifically, our algorithm includes a tuning parameter, call it $\tau$, that controls the trade-off between the evaluation metric of interest (for us, the quality of the recommended $z \in \mc{Z}$), the label complexity $\mc{L}$, and the amount of unlabelled data $\mc{U}$ that is necessary before the metric of interest can be non-trivial.
We prove that each possible setting of $\tau$ parametrizes \emph{all} possible trade-offs between unlabeled and labeled data. 




Our work is perhaps closest to the streaming setting for agnostic active classification \cite{dasgupta2008agnostic, huang2015efficient} where each $x_s$ is drawn i.i.d. from an underlying distribution $\nu$ on $\mc{X}$, and indeed our results can be specialized to this setting as we discuss in Section~\ref{sec:active_classification}. These papers also evaluate themselves at a single point on the tradeoff curve, namely the number of samples needed in passive supervised learning to obtain a learner with excess risk at most $\epsilon$. They provide minimax guarantees on the amount of labeled data needed in terms of the disagreement coefficient \cite{hanneke2014theory}. In contrast, again, our results characterize the full trade-off between the amount of unlabeled data seen, and the amount of labeled data needed to achieve the target excess risk $\epsilon$. We note that using online-to-batch conversion methods, \cite{dekel2012selective, agarwal2013selective, cesa2009robust} also provide results on the amount of labeled data needed but they assume a very specific parametric form to their label distribution unlike our setting which is agnostic. 
Other works have characterized selective sampling for classification in the realizable setting that assumes there exists a classifer among the set under consideration that perfectly labels every $y_t$ \cite{hanneke2021toward}--our work addresses the agnostic setting where no such assumption is made.  
Finally, our results apply under the more general setting of \textit{domain adaptation under covariate shift} where we are observing data drawn from the stream $\nu$, but we will evaluate the excess risk of our resulting classifier on a different stream $\pi$ \cite{, rai2010domain, saha2011active, xiao2013online}. 

\textbf{Best-Arm Identification and Online Experimental Design.} Our techniques are based on experimental design methods for best-arm identification in linear bandits, see \cite{soare2014best, fiez2019sequential,camilleri2021highdimensional}. 
In the setting of these works, there exists a pool of examples $\mc{X}$ and at each time any $x \in \mc{X}$ can be selected with replacement.
The goal is to identify the best arm using as few total selections (labels) as possible. Their algorithms are based on arm-elimination. Specifically, they select examples with probability proportional to an approximate $G$-optimal design with respect to the current remaining arms. Then, during each round after taking measurements, those arms with high probability of being suboptimal will be eliminated. Remarkably, near-optimal sample complexity has been achieved under this setting. While we apply these techniques of arm-elimination and sampling through $G$-optimal design, the major difference is that we are facing a stream instead of a pool of examples.
Finally, \cite{eghbali2018competitive}
considers a different online experiment design setup where  (adversarially chosen) experiments arrive sequentially and a primal-dual algorithm decides whether to choose each, subject to a total budget. \cite{eghbali2018competitive} studies the competitive ratio of such algorithms (in the manner of online packing algorithms) for problems such as $D$-optimal experiment design. 

\section{Selective Sampling for Best Arm Identification}
\noindent Consider the following game: 
Given known $\mc{X},\mc{Z} \subset \R^d$ and unknown $\theta_* \in \R^d$ at each time $t=1,2,\dots$:
\begin{enumerate}
    \item Nature reveals $x_t \overset{iid}{\sim} \nu$ with $\text{support}(\nu)=\mc{X}$
    \item Player chooses $Q_t \in \{0,1\}$. If $Q_t=1$ then nature reveals $y_t$ with $\E[y_t] = \langle x_t, \theta_* \rangle$
    \item Player optionally decides to stop at time $t$ and output some $\widehat{z} \in \mc{Z}$
\end{enumerate}
If the player stops at time $\mc{U}$ after observing $\mc{L}=\sum_{t=1}^{\mc{U}} Q_t$ labels, the objective is to identify $z_* = \arg\max_{z \in \mc{Z}} \langle z, \theta_* \rangle$ with probability at least $1-\delta$ while minimizing a trade-off of $\mc{U},\mc{L}$.

This paper studies the relationship between $\mc{U}$ and $\mc{L}$ in the context of necessary and sufficient conditions to identify $z_*$ with probability at least $1-\delta$.
Clearly $\mc{U}$ must be ``large enough'' for $z_*$ to be identifiable even if all labels are requested (i.e., $\mc{L}=\mc{U}$). 
But if $\mc{U}$ is very large, the player can start to become more picky with their decision to observe the label or not.
Indeed, one can easily imagine scenarios in which it is advantageous for a player to forgo requesting the label of the current example in favor of waiting for a more informative example to arrive later if they wished to minimize $\mc{L}$ alone.
Intuitively, $\mc{L}$ should decrease as $\mc{U}$ increases, but how?

Any selective sampling algorithm for the above protocol at time $t$ is defined by 1) a selection rule $P_t:\mc{X} \rightarrow [0,1]$ where $Q_t \sim \text{Bernoulli}(P_t(x_t))$, 2) a stopping rule $\mc{U}$, and 3) a recommendation rule $\widehat{z} \in \mc{Z}$.
The algorithm's behavior at time $t$ can use all information collected up to time $t$ 

\begin{definition}
For any $\delta \in (0,1)$ we say a selective sampling algorithm is $\delta$-PAC for $\nu \in \triangle_{\mc{X}}$ if for all $\theta \in \R^d$ the algorithm terminates at time $\mc{U}$ which is finite almost surely and outputs $\arg\max_{z \in \mc{Z}} \langle z, \theta \rangle$ with probability at least $1-\delta$.
\end{definition}

\subsection{Optimal design}
Before introducing our own algorithm, let us consider a seemingly optimal procedure.
For any $\lambda \in \triangle_{\mc{X}} = \{ p : \sum_{x \in \mc{X}} p_x =1, \, p_x \geq 0 \, \, \forall x\in\mc{X} \}$ define 
\begin{align}\label{eq:rhodefn}
    \rho(\lambda) := \max_{z \in \mc{Z} \setminus \{z_*\}} \frac{\|z-z_*\|_{\E_{X \sim \lambda}[X X^\top]^{-1}}^2}{\langle \theta_*, z_* - z \rangle^2}.
\end{align}
Intuitively, $\rho(\lambda)$ captures the number of labeled examples drawn from distribution $\lambda$ to identify $z_*$. 
Specifically, for any $\tau \geq \rho(\lambda) \log(|\mc{Z}|/\delta)$, if $x_1,\dots,x_\tau \sim \lambda$ and $y_i = \langle x_i, \theta_* \rangle + \epsilon_i$ where $\epsilon_i$ is iid $1$ sub-Gaussian noise, then there exists an estimator $\widehat{\theta} := \widehat{\theta}( \{ (x_i,y_i) \}_{i=1}^\tau )$ such that $\langle \widehat{\theta} , z_* \rangle > \max_{z \in \mc{Z} \setminus z_*}\langle  \widehat{\theta} , z \rangle$ with probability at least $1-\delta$ \cite{fiez2019sequential}.
In particular, $\tau \geq  \rho(\lambda) \log(|\mc{Z}|/\delta)$ samples suffice to guarantee that $\arg\max_{z \in \mc{Z}} \langle  \widehat{\theta} , z \rangle= \arg\max_{z \in \mc{Z}} \langle  \theta_* , z \rangle =:z_*$.


Thus, if our $\tau$ samples are coming from $\nu$, we would expect any reasonable algorithm to require at least $\rho(\nu) \log(|\mc{Z}|/\delta)$ examples and labels. However, since we only want to take informative examples, we instead choose to select the $t$th example $x_t=x$ according to a probability $P(x)$ so that our final labeled samples are coming from the distribution $\lambda$ where $\lambda(x) \propto P(x)\nu(x)$. In particular, $P(x)$ should be chosen according to the following optimization problem
\begin{align} \label{eq:optimal_design}
P^* = &\argmin_{P: \mc{X} \rightarrow [0,1]} \tau\E_{X \sim \nu}[P(X)] \quad \text{ subject to } \max_{z \in \mc{\mc{Z}}\setminus \{z_*\}} \frac{\|z_*-z\|_{\E_{X\sim\nu}[ \tau P(X) X X^\top]^{-1}}^2}{\langle z_* - z, \theta_* \rangle^2} \beta_\delta \leq 1
\end{align}
for $\beta_\delta = \log(|\mc{Z}|/\delta)$ where the objective captures the number of samples we select using $P^*$, and the constraint captures the fact that we have solved the problem. 
Remarkably, we can reparametrize this result in terms of an optimization problem over $\lambda\in \Delta_{\mc{X}}$ instead of $P^*: \mc{X} \rightarrow [0,1]$ as
\begin{align*}
    \tau \E_{X \sim \nu}[ P^*(X) ] =  \min_{\lambda \in \triangle_{\mc{X}}} \rho(\lambda) \beta_\delta \quad \text{ subject to }\quad \tau \geq \|\lambda/\nu\|_{\infty} \rho(\lambda) \beta_\delta
\end{align*}
where $\|\lambda/\nu\|_{\infty} = \max_{x \in \mc{X}} \lambda(x)/\nu(x)$, as shown in Proposition~\ref{prop:reparameterization}.
Note that as $\tau \rightarrow \infty$ the constraint becomes inconsequential.
Also notice that $\rho(\nu) \beta_\delta$ appears to be a necessary amount of labels to solve the problem even if $P(x)\equiv 1$ (albeit, by arguing about minimizing the upperbound of above). 

\subsection{Main results}

In this section we formally justify the sketched argument of the previous section, showing nearly matching upper and lower bounds.

\begin{theorem}[Lower bound]\label{thm:lower_bound_tau_input}
Fix any $\delta \in (0,1)$, $\mc{X},\mc{Z} \subset \R^d$, and $\theta_* \in \R^d$.
Any selective sampling algorithm that is $\delta$-PAC for $\nu \in \triangle_{\mc{X}}$ and terminates after drawing $\mc{U}$ unlabelled examples from $\nu$ and requests the labels of just $\mc{L}$ of them satisfies
\begin{itemize}
    \item $\E[\mc{U} ] \geq \rho(\nu) \log(1/\delta)$, and
        \item $\displaystyle \E[\mc{L}] \geq \min_{\lambda \in \triangle_{\mc{X}}} \rho(\lambda) \log(1/\delta) \quad \text{ subject to }\quad \E[\mc{U}] \geq \|\lambda/\nu\|_{\infty} \rho(\lambda) \log(1/\delta)$.
    \end{itemize}
\end{theorem}
The first part of the theorem quantifies the number of rounds or unlabelled draws $\mc{U}$ that \emph{any} algorithm must observe before it could hope to stop and output $z_*$ correctly. 
The second part describes a trade-off between $\mc{U}$ and $\mc{L}$.
One extreme is if $\E[\mc{U}] \rightarrow \infty$, which effectively removes the constraint so that the number of observed labels must scale like $\min_{\lambda \in \triangle_{\mc{X}}} \rho(\lambda) \log(1/\delta)$.
Note that this is precisely the number of labels required in the pool-based setting where the agent can choose \emph{any} $x \in \mc{X}$ that she desires at each time $t$ (e.g. \cite{fiez2019sequential}).
In the other extreme, $\E[\mc{U}] = \rho(\nu) \log(1/\delta)$ so that the constraint in the label complexity $\E[\mc{L}]$ is equivalent to $\rho(\nu) \geq \|\lambda/\nu\|_{\infty} \rho(\lambda)$. This implies that the minimizing $\lambda$ must either stay very close to $\nu$, or must obtain a substantially smaller value of $\rho(\lambda)$ relative to $\rho(\nu)$ to account for the inflation factor $\|\lambda / \nu\|_\infty$. 
In some sense, this latter extreme is the most interesting point on the trade-off curve because its asking the algorithm to stop as quickly as the algorithm that observes all labels, but after requesting a minimal number of labels.
Note that this lower bound holds even for algorithms that known $\nu$ exactly. The proof of Theorem~\ref{thm:lower_bound_tau_input} relies on standard techniques from best arm identification lower bounds (see e.g. \cite{kaufmann2016complexity, fiez2019sequential}).

Remarkably, every point on the trade-off suggested by the lower bound is nearly achievable.
\begin{theorem}[Upper bound]\label{thm:upper_bound_tau_input}
Fix any $\delta \in (0,1)$, $\mc{X},\mc{Z} \subset \R^d$, and $\theta_* \in \R^d$. Let $\Delta = \min_{z \in \mc{Z} \setminus \Bp{z_*}} \langle z_* - z, \theta_* \rangle$ and $\beta_\delta \propto \log(\log(\tfrac{1}{\Delta})|\mc{Z}|/\delta)$ where the precise constant is given in the appendix.
For any $\tau \geq \rho(\nu) \beta_\delta$ there exists a $\delta$-PAC selective sampling algorithm that observes $\mc{U}$ unlabeled examples and requests just $\mc{L}$ labels that satisfies with probability at least $1-\delta$
\begin{itemize}
    \item $\mc{U} \leq  \log_2(\tfrac{4}{\Delta}) \, \tau$, and 
    \item $\displaystyle \mc{L} \leq 3 \log_2(\tfrac{4}{\Delta}) \, \min_{\lambda \in \triangle_{\mc{X}}}\rho(\lambda) \, \beta_\delta \quad\text{ subject to }\quad \tau \geq \|\lambda / \nu \|_\infty \rho(\lambda) \, \beta_\delta$.
\end{itemize}
\end{theorem}
\noindent Aside from the $\log(\tfrac{1}{\Delta})$ factor and the $\log(|\mc{Z}|)$ that appears in the $\beta_\delta$ term, this nearly matches the lower bound.
Note that the parameter $\tau$ parameterizes the algorithm and makes the trade-off between $\mc{U}$ and $\mc{L}$ explicit. 
The next section describes the algorithm that achieves this theorem.

\subsection{Selective Sampling Algorithm}

Algorithm~\ref{algo:best_arm_meta} contains the pseudo-code of our selective sampling algorithm for best-arm identification.
Note that it takes a confidence level $\delta \in (0,1)$ and a parameter $\tau$ that controls the unlabeled-labeled budget trade-off as input.
The algorithm is effectively an elimination style algorithm and closely mirrors the RAGE algorithm for the pool-based setting of best-arm identification problem \cite{fiez2019sequential}. 
The key difference, of course, is that instead of being able to plan over the pool of measurements, this algorithm must plan over the $x$'s that the algorithm may \emph{potentially} see and account for the case that it might not see the $x$'s it wants.

\begin{algorithm}[ht]
\caption{Selective Sampling for Best-arm Identification}
\label{algo:best_arm_meta}
\begin{algorithmic}[1]
\STATE{\textbf{Input} $\mc{Z} \subset \R^d$, $\delta \in (0,1), \tau$}
\WHILE{$|\mc{Z}_{\ell}|\geq 1$}
\STATE{Let $\widehat{P}_\ell, \widehat{\Sigma}_{\widehat{P}_\ell} \leftarrow $\textsc{OptimizeDesign}$(\mc{\mc{Z}_\ell},2^{-\ell},\tau)$ \hfill {\color{blue}//\texttt{ $\widehat{\Sigma}_{\widehat{P}_\ell}$ approximates $\E_{X \sim \nu}[ \widehat{P}_\ell(X) X X^\top ]$}}} \label{line:optimizedesign}
\FOR{$t=(\ell-1)\tau+1,\dots,\ell \tau$}
\STATE{Nature reveals $x_t$ drawn iid from $\nu$ (with support $\R^d$)}
\STATE{Sample $Q_t(x_t) \sim \text{Bernoulli}(\widehat{P}_\ell(x_t))$. If $Q_t=1$ then observe $y_t$ \hfill {\color{blue}//\texttt{ $\E[y_t | x_t] = \langle \theta_*, x_t \rangle$} }}
\ENDFOR
\STATE{Let $\widehat{\theta}_\ell\leftarrow$\textsc{RIPS}($\{ \widehat{\Sigma}_{\widehat{P}_\ell}^{-1} Q_s(x_s) x_s y_s \}_{s=(\ell-1)\tau+1}^{\ell \tau}$, $\mc{Z} \times \mc{Z}$) \hfill \color{blue}{// \texttt{$\widehat{\theta}_\ell$ approximates $\theta_*$}} } \label{line:rips}
\STATE{$\displaystyle\mc{Z}_{\ell+1} = \mc{Z}_{\ell} \setminus \{ z \in \mc{Z}_{\ell} : \max_{z' \in \mc{Z}_\ell} \langle z' - z , \widehat{\theta}_\ell \rangle \geq 2^{-\ell} \}$}
\ENDWHILE
\end{algorithmic}
\end{algorithm}

In round $\ell$, the algorithm maintains an active set $\mc{Z}_\ell \subseteq \mc{Z}$ with the guarantee that each remaining $z\in \mc{Z}_{\ell}$ satisfies, $\langle z_{\ast}-z, \theta_{\ast}\rangle \leq 8\cdot 2^{-\ell}$. In each round, on Line~\ref{line:optimizedesign} of the algorithm, it calls out to a sub-routine \textsc{OptimizeDesign}$(\mc{Z},\epsilon,\tau)$ that is trying to approximate the ideal optimal design of ~\eqref{eq:optimal_design}. 
In particular, the ideal response to \textsc{OptimizeDesign}$(\mc{Z},\epsilon,\tau)$ would return a $P^*_\epsilon$ and $\Sigma_{P^*_\epsilon} = \E_{X \sim \nu}[ P_{\epsilon}^*(X) X X^\top ]$ where $P^*_\epsilon$ is the solution to Equation~\ref{eq:optimal_design} with the one exception that the denominator of the constraint is replaced with $\max\{ \epsilon^2, \langle \theta_*, z_*-z\rangle^2 \}$.
Of course, $\theta_*$ is unknown so we cannot solve Equation~\ref{eq:optimal_design} (as well as other outstanding issues that we will address shortly).
Consequently, our implementation will aim to \emph{approximate} the optimization problem of Equation~\ref{eq:optimal_design}. 
But assuming our sample complexity is not too far off from this ideal, each round should not request more labels than the number of labels requested by the ideal program with $\epsilon=0$. Thus, the total number of samples should be bounded by the ideal sample complexity times the number of rounds, which is $O(\log(\Delta^{-1}))$.
We will return to implementation issues in the next section.

Assuming we are returned $(\widehat{P}_\ell,\widehat{\Sigma}_{\widehat{P}_\ell})$ that approximate their ideals as just described, the algorithm then proceeds to process the incoming stream of $x_t \sim \nu$. 
As described above, the decision to request the label of $x_t$ is determined by a coin flip coming up heads with probability $\widehat{P}_\ell(x_t)$--otherwise we do not request the label.
Given the collected dataset $\{(x_t,y_t,Q_t,\widehat{P}_\ell(x_t))\}_t$, line~\ref{line:rips} then computes an estimate $\widehat{\theta}_\ell$ of $\theta_*$ using the RIPS estimator of \cite{camilleri2021highdimensional} which will satisfy
$$|\langle z_{\ast}-z , \widehat{\theta}_\ell - \theta_{\ast} \rangle| \leq O\left(\|z_{\ast} -z\|_{\E_{X \sim \nu}[ \tau \widehat{P}_\ell(X) X X^\top]^{-1}}\sqrt{\log(2\ell^2|\mc{Z}|^2/\delta)}\right) \leq 2^{-\ell}$$ 
for all $z\in \mc{Z}_{\ell}$ simultaneously with probability at least $1-\delta$. 
Thus, the final line of the algorithm eliminates any $z \in \mc{Z}_\ell$ such that there exists another $z' \in \mc{Z}_\ell$ (think $z_*$) that satisfies $\langle \widehat{\theta}_\ell, z' - z \rangle > 2^{-\ell}$.
The process continues until $\mc{Z}_\ell = \{z_*\}$.

\subsection{Implementation of \textsc{OptimizeDesign}}
For the subroutine \textsc{OptimizeDesign} passed $(\mc{Z}_\ell,\epsilon,\tau)$ the next best thing to computing Equation~\ref{eq:optimal_design} with the denominator of the constraint replaced with $\max\{ \epsilon^2, \langle \theta_*, z_*-z\rangle^2 \}$, is to compute
\begin{align}
    \displaystyle 
    P_{\epsilon} = \argmin_{P: \mc{X} \rightarrow [0,1]} \E_{X \sim \nu}[P(X)]\text{ subject to }
    \max_{z,z' \in \mc{\mc{Z}_\ell}} \frac{\|z-z'\|^2_{\E_{X\sim\nu}[ \tau P(X) X X^\top]^{-1}}}{\epsilon^2} \beta_\delta \leq  1 \label{eqn:ideal_opt_problem}
\end{align} 
and $\Sigma_{P_{\epsilon}} = \E_{X\sim \nu}[P_{\epsilon}(X)XX^{\top}]$ for an appropriate choice of $\beta_\delta = \Theta( \log(|\mc{Z}|/\delta) )$. 
To see this, firstly, any $z \in \mc{Z}$ with gap $\langle \theta_*, z_* - z \rangle$  that we could accurately estimate would not be included in $\mc{Z}_\ell$, thus we don't need it in the $\max$ of the denominator.
Secondly, to get rid of $z_*$ in the numerator (which is unknown, of course), we note that for any norm $\max_{z,z'} \|z-z'\| \leq \max_{z} 2 \|z - z_*\| \leq  \max_{z,z'} 2 \|z-z'\|$.
Assuming we could solve this directly and compute $\Sigma_{P_{\epsilon}} = \E_{X\sim \nu}[P_{\epsilon}(X)XX^{\top}]$, we can obtain the result of Theorem 2 (proven in the Appendix). 

However, even if we knew $\nu$ exactly, the optimization problem of Equation~\ref{eqn:ideal_opt_problem} is quite daunting as it is a potentially infinite dimensional optimization problem over $\mc{X}$.
Fortunately, after forming the Lagrangian with dual variables for each $z-z' \in \mc{Z}\times \mc{Z}$, optimizing the dual amounts to a finite dimensional optimization problem over the finite number of dual variables. 
Moreover, this optimization problem is maximizing a simple expectation with respect to $\nu$ and thus we can apply standard stochastic gradient ascent and results from stochastic approximation \cite{nemirovskistochastic}.
Given the connection to stochastic approximation, instead of sampling a fresh $\widetilde{x} \sim \nu$ each iteration, it suffices to ``replay'' a sequence of $\widetilde{x}$'s from historical data.
Summing up, this construction allows us to compute a satisfactory $P_\epsilon$ and avoid both an infinite-dimensional optimization problem and requiring knowledge of $\nu$ (as long as historical data is available).

Meanwhile, with historical data, we can also empirically compute $\E_{X \sim \nu}[ P_\epsilon(X) X X^\top ]$. Historical data could mean offline samples from $\nu$ or just samples from previous rounds. In this setting, Theorem 2 still holds albeit with larger constants. Theorem~\ref{thm:upper_bound_tau_input_unknown_nu} in the appendix characterizes the necessary amount of historical data needed. Unfortunately (in full disclosure) the theoretical guarantees on the amount of historical data needed is absurdly large, though we suspect this arises from a looseness in our analysis. Similar assumptions and approaches to historical or offline data have been used in other works in the streaming setting e.g. \cite{huang2015efficient}. 

\section{Selective Sampling for Binary Classification}\label{sec:active_classification}
We now review streaming Binary Classification in the agnostic setting \cite{dasgupta2008agnostic,hanneke2014theory,huang2015efficient} and show that our approach can be adapted to this setting. Consider a binary classification problem where $\mc{X}$ is the example space and $\mc{Y} = \{-1,1\}$ is the label space.
Fix a hypothesis class $\mc{H}$ such that each $h \in \mc{H}$ is a classifier $h: \mc{X} \rightarrow \mc{Y}$.
Assume there exists a fixed regression function $\eta : \mc{X} \rightarrow [0,1]$ such that the label of $x$ is Bernoulli with probability $\eta(x) = \P(Y = 1 | X=x)$.
Being in the agnostic setting, we make no assumption on the relationship between $\mc{H}$ and $\eta$.
Finally, fix any $\nu \in \triangle_{\mc{X}}$ and $\pi \in \triangle_{\mc{X}}$. Given known $\mc{X},\mc{H}$ and unknown regression function $\eta$, at each time $t=1,2,\dots$:
\begin{enumerate}[leftmargin=8pt]
    \item Nature reveals $x_t \sim \nu$ 
    \item Player chooses $Q_t \in \{0,1\}$. If $Q_t=1$ then nature reveals $y_t \sim \text{Bernoulli}(\eta(x_t)) \in \{-1,1\}$
    \item Player optionally decides to stop at time $t$ and output some $\widehat{h} \in \mc{H}$.
\end{enumerate}
Define the \emph{risk} of any $h \in \mc{H}$ as $R_\pi(h):=\P_{X \sim \pi,Y \sim \eta(X)}( Y \neq h(X) )$.
If the player stops at time $\mc{U}$ after observing $\mc{L}=\sum_{t=1}^{\mc{U}} Q_t$ labels, the objective is to identify $h_* = \arg\min_{h \in \mc{H}} R_{\pi}(h)$ with probability at least $1-\delta$ while minimizing a trade-off of $\mc{U},\mc{L}$.
Note that $h_*$ is the true risk minimizer with respect to distribution $\pi$ but we observe samples $x_t \sim \nu$; $\pi$ is not necessarily equal to $\nu$.
While we have posed the problem as identifying the potentially unique $h^*$, our setting naturally generalizes to identifying an $\epsilon$-good $h$ such that $R_\pi(h)-R_\pi(h_*) \leq \epsilon$.

We will now reduce selective sampling for binary classification problem to selective sampling for best arm identification, and thus immediately obtain a result on the sample complexity. 
For simplicity, assume that $\mc{X}$ and $\mc{H}$ are finite. Enumerate $\mc{X}$ and for each $h \in \mc{H}$ define a vector $z^{(h)} \in [0,1]^{|\mc{X}|}$ such that $z_x^{(h)} := \pi(x) \1\{h(x)=1\}$ for $z^{(h)} = [z_x^{(h)}]_{x \in \mc{X}}$.
Moreover, define $\theta^* := [ \theta^*_x ]_{x \in \mc{X}}$ where $\theta_x^* := 2 \eta(x) - 1$.
Then
\begin{align*}
    R_\pi(h) 
    &= \E_{X \sim \pi, Y\sim \eta(X)}[\1\{Y\neq h(X)\}]\! =\! \sum_{x \in \mc{X}} \pi(x) (\eta(x) \1\{h(x)\neq 1\} \!+\! (1 - \eta(x))\1\{h(x)\neq 0\})\\
    &= \sum_{x \in \mc{X}} \pi(x) \eta(x) + \sum_{x \in \mc{X}} \pi(x) (1 - 2\eta(x))\1\{h(x)=1\} = c - \langle z^{(h)} , \theta^* \rangle
\end{align*}
where $c = \sum_{x \in \mc{X}} \pi(x) \eta(x)$ does not depend on $h$. Thus, if $\mc{Z} :=\{ z^{(h)}\}_{h \in \mc{H}}$ then identifying $h_* = \arg\min_{h\in \mc{H}}R_\pi(h)$ is equivalent to identifying $z_* = \arg\max_{z\in \mc{Z}} \langle z, \theta^* \rangle$. 
We can now apply Theorem~\ref{thm:upper_bound_tau_input} to obtain a result describing the sample complexity trade-off. 
First define, 
\begin{align*}
    \rho_{\pi}(\lambda, \varepsilon) := \max_{z \in \mc{Z} \setminus \{z_*\}} \frac{\|z-z_*\|_{\E_{X \sim \lambda}[X X^\top]^{-1}}^2}{\max\{\langle \theta_*, z_* - z \rangle^2,\varepsilon^2\}}
    =
    \max_{h \in \mc{H} \setminus \{h_*\}} \frac{\E_{X\sim\pi}\left[\bm{1}\{h(X) \neq h'(X)\}\frac{\pi(X)}{\lambda(X)}\right]}{\max\{(R_\pi(h)-R_\pi(h^*))^2, \varepsilon^2\}} 
\end{align*}

An important case of the above setting is when $X\sim \nu$ and $\pi = \nu$, i.e. we are evaluating the performance of a classifier relative to the same distribution our samples are drawn from. This is the setting of \cite{dasgupta2008agnostic, huang2015efficient, hanneke2014theory}. The following theorem shows that the sample complexity obtained by our algorithm is at least as good as the results they present.

\begin{theorem}\label{thm:upper_bound_tau_input_classification}
Fix any $\delta\in (0,1)$, domain $\mc{X}$ with distribution $\nu$, finite hypothesis class $\mc{H}$, regression function $\eta:\mc{X}\rightarrow [0,1]$. Set $\epsilon \geq 0$ and $\beta_{\delta} = 2048\log(4\log_2^2(4/\epsilon)|\mc{H}|/\delta)$. Then for $\tau \geq \rho_{\pi}(\nu, \epsilon)\beta_{\delta}$ there exists a selective sampling algorithm that returns $h\in \mc{H}$ satisfying $R_{\pi}(h) - R_{\pi}(h^{\ast})\leq \epsilon$ by observing $\mc{U}$ unlabeled examples and requesting just $\mc{L}$ labels such that
\begin{itemize}
    \item $\mc{U} \leq \log_2(4/\epsilon)\tau$
        \item $\displaystyle \mc{L} \leq 3 \log_2(\tfrac{4}{\varepsilon}) \, \min_{\lambda \in \triangle_{\mc{X}}}\rho_{\pi}(\lambda, \varepsilon) \beta_\delta \quad\text{ s.t. }\quad \tau \geq \|\lambda / \nu \|_\infty \rho_{\pi}(\lambda, \varepsilon)\beta_{\delta} $
\end{itemize}
with probability at least $1-\delta$. Furthermore when $\nu=\pi$ and if $\tau\geq 16\rho(\nu, \epsilon)\beta_{\delta}$ we have that 
\[\mc{L} \leq 36\log_2(4/\epsilon) \left(\tfrac{R_{\nu}(h^{\ast})^2}{\epsilon^2}+4\right)\sup_{\xi\geq \epsilon }\theta^{\ast}(2R_{\nu}(h^{\ast})+\xi,\nu)\beta_{\delta}\]
where $\theta^{\ast}(u,\nu)$ is the disagreement coefficient, defined in Appendix~\ref{sec:classification}.
\end{theorem}
Note that if $\tau$ is sufficiently large then the labeled sample complexity we obtain $\min_{\lambda\in \Delta_{\mc{X}}} \rho(\lambda, \epsilon)$ could be significantly smaller than previous results in the streaming setting, e.g. see \cite{katzsamuels2021improved}. The proof of Theorem~\ref{thm:upper_bound_tau_input_classification} can be found in Appendix~\ref{sec:classification}.

\section{Solving the Optimization Problem}

Recall that in Algorithm \ref{algo:best_arm_meta}, during round $\ell$, we need to solve optimization problem \eqref{eqn:ideal_opt_problem}.
Solving this optimization problem is not trivial because the number of variables can potentially be infinite if $\mc{X}$ is an infinite set. In this section, we will demonstrate how to reduce it to a finite-dimensional problem by considering its dual problem. To simplify the notation, let $\mc{Y}_\ell=\Bp{z-z':z, z'\in\mc{Z}_\ell, z\neq z'}$, and  rewrite the problem as follows, where $c_\ell>0$ is a constant that may depend on round $\ell$.
\begin{equation}
    \label{equ:opt_original}
    \begin{array}{rl}
        \min_{P} & \E_{X\sim\nu}\Mp{P(X)} \\
        \text{subject to} & y^\top \E_{X\sim\nu}\Mp{P(X)XX^\top}^{-1}y\leq c_\ell^2,\quad\forall y\in\mc{Y}_\ell,\\
        & 0\leq P(x)\leq 1,\quad\forall x\in\mc{X}.
    \end{array}
\end{equation}
Using the Schur complement technique, we show in Lemma \ref{lmm:schur_property} (Appendix \ref{sec:analysis_opt}) the following equivalence:
$y^\top \E_{X\sim\nu}\Mp{P(X)XX^\top}^{-1}y\leq c_\ell^2\Longleftrightarrow \E_{X\sim\nu}\Mp{P(X)XX^\top}\succeq\frac{1}{c_\ell^2}yy^\top$.
This transforms a constraint involving matrix inversion into one with ordering between PSD matrices. 
Then, we remove the bound constraints $0\leq P(x)\leq 1$, $\forall x\in\mc{X}$ by introducing the barrier function $-\log(1-x)-\log(x)$. That is, instead of working with the objective $\E_{X\sim\nu}\Mp{P(X)}$ directly, we consider the following problem.
\begin{equation}
    \label{equ:opt_barrier}
    \begin{array}{rl}
		\min_{P} & \E_{X\sim\nu}[P(X)-\mu_b(\log(1-P(X))+\log(P(X)))] \\
		\text{subject to} &  \E_{X\sim\nu}\Mp{P(X)XX^\top}\succeq\frac{1}{c_\ell^2}yy^\top, \quad\forall y\in\mathcal{Y}_\ell.
	\end{array}
\end{equation}
Here, $\mu_b\in(0, 1)$ is some small constant that controls how strong the barrier is. Intuitively, a smaller $\mu_b$ will make problem \eqref{equ:opt_barrier} closer to the original problem. 
We now show that unlike the primal, the dual problem is indeed finite-dimensional. 
For each constraint of $y\in\mc{Y}_\ell$, let the matrix $\Lambda_y\succeq\bm{0}$ be its dual variable. Further, let $\Lambda=\sum_{y\in\mc{Y}_\ell}\Lambda_y$ and $\bm{\Lambda}=\Sp{\Lambda_y}_{y\in\mc{Y}_\ell}$. The corresponding Lagrangian is
\begin{align*}
    \mc{L}&\Sp{\bm{\Lambda},P}=\E_{X\sim\nu}\Mp{P(X)\!-\!\mu_b\Sp{\log(1\!-\!P(X))\!+\!\log(P(X))}\!-\!P(X)X^\top\Lambda X}+\frac{1}{c_\ell^2}\sum_{y\in\mathcal{Y}_\ell}y^\top\Lambda_y y.
\end{align*}
The dual problem is $\max_{\Lambda_{y}\succeq\bm{0},\forall y\in \mc{Y}_\ell}\min_{P}\mc{L}\Sp{\bm{\Lambda}, P}$. Notice that minimization over $P:\mc{X}\mapsto[0, 1]$ can be done via minimizing $P(x)$ point-wise for each  $x\in\mc{X}$. To do this, we take the gradient with respect to each $P(x)$ and set it to zero to get
\begin{equation}
    \label{equ:P_equ}
    1+\frac{\mu_b}{1-P(x)}-\frac{\mu_b}{P(x)}-x^\top\Lambda x=0.
\end{equation}
Solving this equation and defining $q_{\Lambda}(x)=x^\top\Lambda x-1$, we get 
\begin{equation}
    \label{equ:P_sol}
    P_{\Lambda}(x)=\frac{1}{2}-\frac{\mu_b}{q_{\Lambda}(x)}+\frac{\sqrt{\Sp{2\mu_b-q_{\Lambda}(x)}^2+4\mu_b q_{\Lambda}(x)}}{2q_{\Lambda}(x)}.
\end{equation}


Note that if $\mu_b=0$ (no barrier), the above reduces to the ``threshold'' decision rule $P_{\Lambda}(x)=\frac{1}{2}+\frac{|q_{\Lambda}(x)|}{2q_{\Lambda}(x)}$, which gives $0$ when $q_\Lambda (x)<0$ and $1$ when $q_\Lambda (x)>0$.\footnote{When $q_{\Lambda}(x)=0$, $P_\Lambda(x)$ is undetermined from the dual.} This is exactly the hard elliptical threshold rule mentioned before, in which whether to query the label for $x$ depends on whether it falls inside ($x^\top\Lambda x<1$) or outside ($x^\top\Lambda x>1$) of the ellipsoid defined by the positive semidefinite matrix $\Lambda$. A visualization of the decision rule $P_\Lambda$ is given in Figure \ref{fig:visual_P} in the Appendix.

Now, by plugging in $P_{\Lambda}(x)$, our dual problem becomes $\max_{\Lambda_y\succeq\bm{0},\forall y}D(\bm{\Lambda}):=\mathcal{L}\Sp{\bm{\Lambda}, P_\Lambda}$.
%
This is a finite-dimensional optimization problem, and can be solved by projected gradient ascent (or projected stochastic gradient ascent when we have only samples from $\nu$). The gradient of $D(\bm{\Lambda})$ is
\begin{align*}
    \nabla_{\Lambda_y}D(\bm{\Lambda})&=\E_{X\sim\nu}\Mp{\Sp{1\!+\!\frac{\mu_b}{1-P_{\Lambda}(x)}\!-\!\frac{\mu_b}{P_{\Lambda}(X)}\!-\!X^\top\Lambda X}\nabla_{\Lambda_y}P_{\Lambda}(X)\!-\!P_{\Lambda}(X)XX^\top}\!+\!\frac{yy^\top}{c_\ell^2}\\
    &=\frac{yy^\top}{c_\ell^2}-\E_{X\sim\nu}\Mp{P_{\Lambda}(X)XX^\top}.\tag{Since $P_{\Lambda}(X)$ solves Eq. \eqref{equ:P_equ}}
\end{align*}
The algorithm to solve the problem has been summarized in Algorithm \ref{algo:solve_sgd}, in which the gradient during $k$th iteration is replaced by its unbiased estimator $\frac{yy^\top}{c_\ell^2}-P_{\hat{\Lambda}^{(k)}}(x_k)x_kx_k^\top$. The adaptive learning rate is chosen by following the discussion in chapter 4 of \cite{orabona2019modern}. 
Optimizing the assignment of $\hat{\Lambda}_y$ to each y in line \ref{line:assign_Lambda_y} ensures that the re-scaling step in line \ref{line:re-scaling} increases the function value in an optimized way. Finally, the re-scaling step is used to ensure that the output primal objective value $\E_{X\sim\nu}\Mp{P_{\title{\Lambda}}(X)}$ is bounded well, which will be explained in more details in Appendix \ref{sec:analysis_opt}.

\begin{algorithm}[ht]
\caption{Projected Stochastic Gradient Ascent to Solve \textsc{OptimizeDesign}}
\label{algo:solve_sgd}
\begin{algorithmic}[1]
\STATE{\textbf{Input:} Number of iterations $K$; number of samples $u$; barrier weight $\mu_b \in(0, 1)$}
\STATE{Initialize $\hat{\Lambda}^{(0)}_y=\mathbf{0}$ for each $y\in\mc{Y}_\ell$}
\FOR{$k=0, 1, 2, \dots, K-1$}
\STATE{Sample $x_k\sim\nu$}
\STATE{Set $g_{k, y}=\frac{yy^\top}{c_\ell^2}-P_{\hat{\Lambda}^{(k)}}(x_k)x_kx_k^\top$, where $P_{\Lambda}$ is defined in Eq. \eqref{equ:P_sol}}
\STATE{Set $\hat{\Lambda}^{(k+1)}_y\leftarrow \hat{\Lambda}^{(k)}_y+\eta_kg_{k, y}$ for each $y\in\mc{Y}_\ell$, where $\eta_k=\frac{1}{\sqrt{2\sum_{s=1}^k\sum_{y\in\mc{Y}_\ell}\Norm{g_{s, y}}_2^2}}$}
\STATE{Update $\hat{\Lambda}^{(k+1)}_y\leftarrow\Pi_{\mathbb{S}^d_{+}}(\hat{\Lambda}^{(k+1)}_y)$ for each $y\in\mc{Y}_\ell$, a projection to the set of $d\times d$ PSD matrices}
\ENDFOR
\STATE{Let $\hat{\Lambda}_y=\frac{1}{K}\sum_{k=1}^{K}\hat{\Lambda}^{(k)}_y$ for each $y\in\mc{Y}_\ell$ and $\hat{\Lambda}=\sum_{y\in \mc{Y}_{\ell}} \hat{\Lambda}_y$}
\STATE{Update $(\hat{\Lambda}_y)_{y\in\mc{Y}_\ell}\leftarrow\argmax_{\bm{\Lambda}} \sum_{y\in\mc{Y}_\ell}y^\top\Lambda_y y$, $\text{subject to }\sum_{y\in\mc{Y}_\ell}\Lambda_y=\hat{\Lambda}, \Lambda_y\succeq\bm{0}, \forall y\in\mc{Y}_\ell.$}\label{line:assign_Lambda_y}
\STATE{Find $s^*\leftarrow\argmax_{s\in[0, 1]}D_{E}(s\cdot\hat{\bm{\Lambda}})$, where $D_{E}$ empirically evaluates $D$ using $u$ i.i.d. samples}\label{line:re-scaling}
\STATE{\textbf{return} $\widetilde{\Lambda}=s^*\cdot\sum_{y\in\mc{Y}_\ell}\hat{\Lambda}_y$}
\end{algorithmic}
\end{algorithm}

Let $\bm{\Lambda}^*$ be an optimal solution for $D(\bm{\Lambda})$. Intuitively, as long as we run this algorithm with sufficiently large number of iterations $K$ and number of samples $u$, we can guarantee that $D(\widetilde{\bm{\Lambda}})$ and $D(\bm{\Lambda}^*)$ are close enough with high probability, which in turn guarantees that the primal constraints are violated by only a tiny amount and $\E_{X\sim\nu}\Mp{P_{\widetilde{\Lambda}}(X)}$ is close enough to the optimal value. 
Specifically, we can prove the following theorem.
\begin{theorem}
\label{theo:opt}
Suppose $\Norm{x}_2\leq M$ for any $x\in\mathrm{supp}(\nu)$ and $\Sigma=\E_{X\sim\nu}\Mp{XX^\top}$ is invertible. Let $\bm{\Lambda}^*\in\argmax_{\Lambda_y\succeq\bm{0}, \forall y\in\mc{Y}_\ell}D(\bm{\Lambda})$ and $\kappa(\Sigma)=\frac{\lambda_{\max}(\Sigma)}{\lambda_{\min}(\Sigma)}$ be its condition number. Assume $\Norm{\Lambda^*}_F>0$ and define $\omega=\min_{\Gamma\in\mathbb{S}^d:\Norm{\Gamma}_F=1}\E_{X\sim\nu}\Mp{\Sp{X^\top\Gamma X}^2}$, where $\mathbb{S}^d$ is the set of $d\times d$ symmetric matrices.

Then, $\Lambda^*=\sum_{y\in\mc{Y}_\ell}\Lambda^*_y$ is unique. Further, for any $\epsilon>0$ and $\delta>0$, if it holds that $\mu_b\leq O\Sp{ \sqrt{\Norm{\Lambda^*}_F\kappa(\Sigma)}M}\cdot\sqrt{(1+\epsilon)/\epsilon}$ and
$$K\!\geq\! O\!\!\Sp{\!\frac{\!\abs{\mc{Y}_\ell}^3\!\!\kappa(\Sigma)^2\!\Norm{\Lambda^*}_F^8\!M^{16}\!\log(1/\delta)\!}{\omega^2\mu_b^6}\!}\!\cdot\!\Sp{\!\frac{1+\epsilon}{\epsilon}\!}^2\!\!,u\!\geq\! O\!\!\Sp{\!\!\frac{\kappa(\Sigma)^2\!\Norm{\Lambda^*}_F^6\!M^{16}\!\log(1/\delta)}{\omega^2\mu_b^6}\!}\!\cdot\!\Sp{\!\frac{1+\epsilon}{\epsilon}\!}^2,$$
then, with probability at least $1-\delta$, Algorithm \ref{algo:solve_sgd} will output $\widetilde{\Lambda}$ that satisfies
\begin{itemize}
    \item $y^\top \E_{X\sim\nu}\Mp{P_{\widetilde{\Lambda}}(X)XX^\top}^{-1}y\leq (1+\epsilon)c_\ell^2,\quad\forall y\in\mc{Y}_\ell$.
    
    \item $\E_{X\sim\nu}\Mp{P_{\widetilde{\Lambda}}(X)}\leq \E_{X\sim\nu}\Mp{\widetilde{P}(X)}+4\sqrt{\mu_b}$, where $\widetilde{P}$ is the optimal solution to problem \eqref{equ:opt_original} with barrier constraint repaced by $0\leq P(x)\leq 1-\mu_b, \forall x\in\mc{X}$.
\end{itemize}
\end{theorem}

The proof is in Appendix \ref{sec:analysis_opt}. Although $\widetilde{P}$ is not exactly the same as the optimal solution of the original problem \eqref{equ:opt_original}, when $\mu_b$ is sufficiently small, they will be very close. Meanwhile, it should be noted that Theorem \ref{theo:opt} mainly reveals that with sufficiently large number of iterations and number of samples, Algorithm \ref{algo:solve_sgd} can output sufficiently good solution. 
In future work, we plan to examine how much this bound can be improved via a tighter analysis.

Finally, notice that Algorithm \ref{algo:solve_sgd} needs to maintain $\abs{\mc{Y}_\ell}d^2=O(\abs{\mc{Z}_\ell}^2d^2)$ variables, which can be large when we have a large set $\mc{Z}_\ell$. Therefore, as an alternative, we also propose Algorithm \ref{algo:solve_sgd2} that only needs to maintain $d^2$ variables but requires more computational power in each iteration. The details are given in Appendix \ref{sec:analysis_opt}.


\section{Empirical results}

In this section we present a benchmark experiment validating the fundamental trade-offs that are theoretically characterized in Theorem~\ref{thm:lower_bound_tau_input} and Theorem~\ref{thm:upper_bound_tau_input}. We take inspiration from \cite{soare2014best} to define our experimental protocol:
\begin{itemize}[leftmargin=5pt]
\setlength\itemsep{.1em}
    \item $d=2$, a two-dimensional problem.
    
    \item $\mathcal{Z} = [\mathbf{e}_1, \mathbf{e}_2, (\cos(\omega), \sin(\omega))]$ for $\omega = 0.3$, where $\mathbf{e}_1, \mathbf{e}_2$ are canonical vectors.
    
    \item $\theta_* = 2\mathbf{e}_1$ and $y = x^\top \theta_* + \eta$, where $\eta\sim\mathcal{N}(0, 1)$.
    
    \item The distribution $\nu$ for streaming measurements $x_t\overset{i.i.d.}{\sim}\nu$ is such that $x_t = (\cos(2 I_t \pi/N), \sin(2 I_t \pi/N))$ where $I_t \in \{0,\dots,N-1\}$, $\mathbb{P}(I_t = i)\propto \cos(2i\pi/N)^2$, and $N=30$.
\end{itemize}
In this problem, the angle $\omega$ is small enough that the item $(\cos(\omega), \sin(\omega))$ is hard to discriminate from the best item $\mathbf{e}_1$. As argued in \cite{soare2014best}, an efficient sampling strategy for this problem instance would be to pull arms in the direction of $\pm \mathbf{e}_2$ in order
to reduce the uncertainty in the direction of interest, $\mathbf{e}_1 - (\cos(\omega), \sin(\omega))$.
However, the distribution $\nu$ is defined such that it is more likely to receive a vector $x_t$ in the direction of $\pm \mathbf{e}_1$ rather than $\pm \mathbf{e}_2$.
Thus, if one seeks a small label complexity, then $P$ should be taken to reject measurements in the direction of $\pm e_1$.

In the benchmark experiment, we compare the following three algorithms which all use Algorithm~\ref{algo:best_arm_meta} as a meta-algorithm and just swap out the definition of $\widehat{P}_\ell$.
\texttt{Naive Algorithm} uses no selective sampling so that $\widehat{P}_\ell(x)=1$ for all $x$; the \texttt{Oracle Algorithm} uses $\widehat{P}_\ell = P_*$ where $P_*$ is the ideal solution to \eqref{eq:optimal_design}, and \texttt{Our Algorithm} uses the solution to \eqref{equ:opt_barrier} for $\widehat{P}_\ell$, where we take $\mu_b=2\times 10^{-5}$.
We swept over the values of $\tau$ and plotted on the y-axis the amount of labeled data needed before termination, as shown in Figure \ref{fig:L_vs_tau}.

\begin{figure}[ht]
    \centering
    \includegraphics[width=\textwidth, height=120pt]{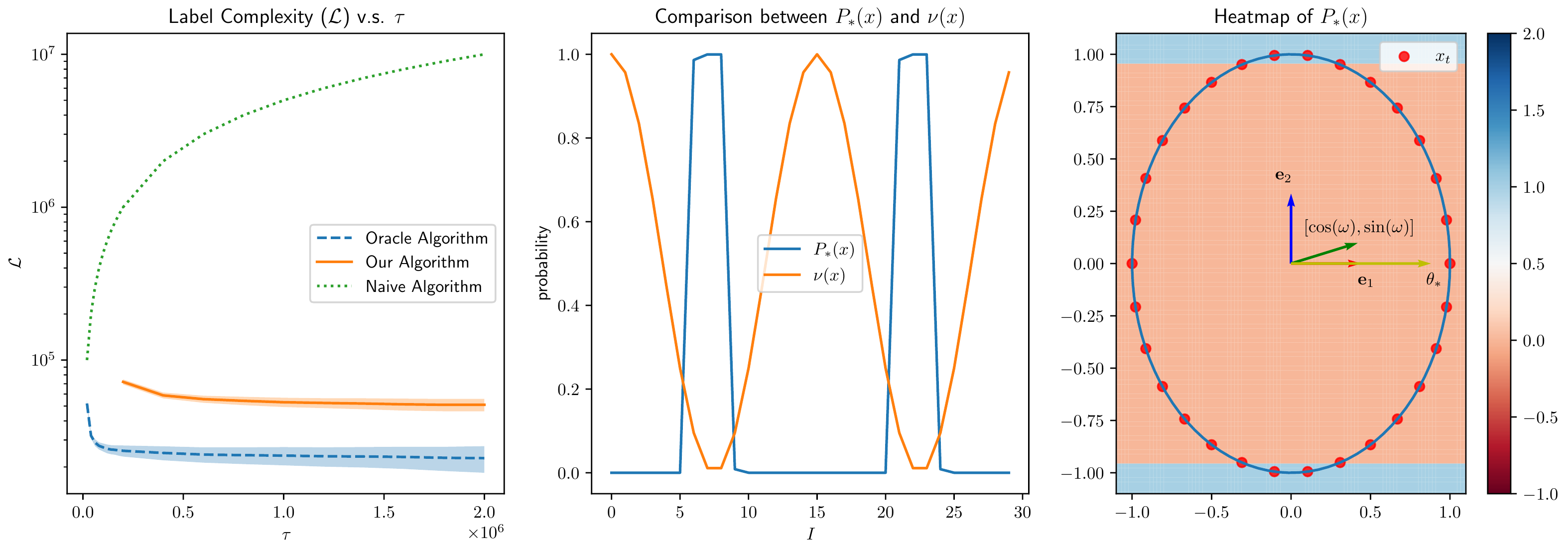}
    \caption{(left) For each value of $\tau$, we plot the average label complexity over 50 repeated trials. (middle) Visualization of $P_*(x)$ and $\nu(x)$ v.s. $x$, where $x$ is indexed by $I$ such that $x_I=(\cos(2 I \pi/N), \sin(2 I \pi/N))$. Here, $P_*$ is solved with $\tau=4\times 10^5$ and distribution $\nu$ is not normalized. (right) A heat map of $P_*(x)$ along with the setting of experimental protocol.}
    \label{fig:L_vs_tau}
\end{figure}

We observe in Figure~\ref{fig:L_vs_tau} that the algorithms using non-naive selection rules require far less label complexity than the naive algorithm for all $\tau$. This reflects the intuition that selection strategies that focus on requesting the more informative streaming measurements are much more efficient than naively observing every streaming measurement. 
Meanwhile, the trade-off between label complexity $\mathcal{L}$ and sample complexity $\mathcal{U}$ characterized in Theorem~\ref{thm:lower_bound_tau_input} and Theorem~\ref{thm:upper_bound_tau_input} is precisely illustrated in Figure~\ref{fig:L_vs_tau}. Indeed, we see the number of labels queried by the two selective sampling algorithms decrease as the number of unlabeled data seen in each round increases. 

\section{Conclusion}

In this paper, we proposed a new approach for the important problem of \textit{selective sampling for best arm identification}. We provide a lower bound that quantifies the trade-off between labeled samples and stopping time and also presented an algorithm that nearly achieves the minimal label complexity given a desired stopping time.

One of the main limitations of this work is that our approach depends on a well-specified model following stationary stochastic assumptions. In practice, dependencies over time and model mismatch are common. 
Utilizing the proposed algorithm outside of our assumptions may lead to poor performance and  unexpected behavior with adverse consequences.
While negative results justify some of the most critical assumptions we make (e.g., allowing the stream $x_t$ to be arbitrary, rather than iid, can lead to trivial algorithms, see Theorem 7 of \cite{chen2021active}), exploring what theoretical guarantees are possible under relaxed assumptions is an important topic of future work. 
\section*{Acknowledgements}
We sincerely thank Chunlin Sun for the insightful discussion on the alternative approach to the optimal design. This work was supported in part by the NSF TRIPODS II grant DMS 2023166, NSF TRIPODS CCF 1740551, NSF CCF 2007036 and NSF TRIPODS+X DMS 1839371. 

\bibliography{refs}
\bibliographystyle{plain}

 \newpage
\appendix
\tableofcontents
\newpage

\section{Selective Sampling Lower Bound}
First, we review the standard argument for best-arm identification lower bounds applied to linear bandits.
Fix $\theta_* \in \R^d$ and let $z_* = \arg\max_{z \in \mc{Z}} \langle z, \theta_* \rangle$.
Define the set $\mc{C} = \{ \theta \in \R^d : \exists z \in \mc{Z} \text{ s.t. } \langle \theta, z - z_* \rangle \geq 0 \}$ as those $\theta$ in which $z_*$ is note the best arm under $\theta$.
We now recall the transportation lemma of \cite{kaufmann2016complexity}. Under a $\delta$-PAC strategy for finding the best arm for the bandit instance $(\mc{X} , \mc{Z}, \theta_*)$, let $T_x$ denote the random variable which is the number of times arm $x$ is pulled. In addition let $\mc{N}_{\theta, x}$ denote the reward distribution of the arm $x$ of $\mc{X}$, i.e. $\mc{N}_{\theta, x} = \mc{N}(x^\top\theta, 1)$. Then for any $\delta$-PAC algorithm
\begin{align*}
    \log(1/2.4\delta) &\leq \min_{\theta \in \mc{C}} \sum_{x\in \mc{X}} \E[T_x]\text{KL}(\mc{N}_{\theta_*, x}, \mc{N}_{\theta, x}) \\
    &= \min_{\theta \in \mc{C}} \sum_{x\in \mc{X}} \E[T_x] \, \tfrac{1}{2}\| \theta_* - \theta \|_{x x^\top}^2 \\
    &= \min_{\theta \in \mc{C}} \tfrac{1}{2} \| \theta_* - \theta \|_{( \sum_{x\in \mc{X}} \E[T_x] \, x x^\top)}^2 \\
    &\leq \min_{z \in \mc{Z} \setminus z_*}  \tfrac{1}{2} \| \theta_* - \theta_z(\epsilon) \|_{( \sum_{x\in \mc{X}} \E[T_x] \, x x^\top)}^2 
\end{align*}
where 
\begin{align*}
    \theta_z(\varepsilon) = \theta_* - \frac{((z_*-z)^\top \theta_*+\varepsilon) ( \sum_{x\in \mc{X}} \E[T_x] \, x x^\top)^{-1}(z_*-z)^\top}{(z_*-z)^\top ( \sum_{x\in \mc{X}} \E[T_x] \, x x^\top)^{-1}(z_*-z)}
\end{align*}
for some small $\epsilon$. This is a valid choice since for all $z \in \mc{Z} \setminus z_*$ we have  $(z_*-z)^\top \theta_z(\varepsilon) = -\varepsilon < 0$ and thus $\theta_z(\varepsilon) \in \mc{C}$. 
A straightforward calculation shows that
\begin{align*}
    \| \theta_* - \theta_z(\epsilon) \|_{( \sum_{x\in \mc{X}} \E[T_x] \, x x^\top)}^2 = \frac{(\langle z_*-z, \theta_* \rangle+\varepsilon)^2}{\|z_*-z\|^2_{(\sum_{x\in \mc{X}} \E[T_x] \, x x^\top)^{-1}}}
\end{align*}
so that after rearranging and lettering $\epsilon \rightarrow 0$ we have that any $\delta$-PAC algorithm satisfies
\begin{align}
     \max_{z \in \mc{Z} \setminus z_*}  \frac{2 {\|z_*-z\|^2_{(\sum_{x\in \mc{X}} \E[T_x] \, x x^\top)^{-1}}}}{\langle z_*-z, \theta_* \rangle^2} \log(1/2.4\delta) \leq 1. \label{eqn:rho_lowerbound_identity}
\end{align}
This series of steps will be applied for each bullet point of the theorem.

\subsection{Proof of Theorem~\ref{thm:lower_bound_tau_input}, part I}
We use the consequence of Lemma~19 of \cite{kaufmann2016complexity}. 
Consider a $\delta$-PAC algorithm that sets $P(x) =1$ for all $x \in \mc{X}$ for all time until it exits at time $\mc{U}$ after this many unlabelled examples have been observed. 
If $T_x$ denotes the number of times $x \in \mc{X}$ was observed before stopping time $\mc{U}$, then by Wald's identity we have that
\begin{align*}
    \E[T_x] = \E\left[ \sum_{t=1}^{\mc{U}} \1\{ x_t = x \} \right] = \nu(x) \E[ \mc{U} ].
\end{align*}
Plugging this into Equation~\ref{eqn:rho_lowerbound_identity} and rearranging we conclude that
\begin{align*}
    \E[\mc{U} ] \geq \max_{z \in \mc{Z} \setminus z_*}  \frac{2 {\|z_*-z\|^2_{(\sum_{x\in \mc{X}} \nu(x) \, x x^\top)^{-1}}}}{\langle z_*-z, \theta_* \rangle^2} \log(1/2.4\delta) =: \rho(\nu) \log(1/2.4\delta)
\end{align*}
which concludes the proof of the first bullet. 

\subsection{Proof of Theorem~\ref{thm:lower_bound_tau_input}, part II}
By definition, the (random) number of times we measure $x$ is
\begin{align*}
    \mc{L}_x = \sum_{s=1}^{\mc{U}}\textbf{1}\{x_s = x, Q_s(x) = 1\}
\end{align*} 
and we want to show that $\E[\mc{L}_x] = \nu(x)\E\left[\sum_{\ell=1}^{\mc{U}}P_\ell(x)\right]$. To do so, we define 
\begin{align*}
    M_t = \sum_{s=1}^{t}\left(\textbf{1}\{x_s = x, Q_s(x) = 1\} - \nu(x)P_s(x)\right)
\end{align*}
It is easy to check that $P_{t+1} \in \mc{F}_t :=\{(x_s, y_s, Q_s)\}_{s=1}^t$ and that 
\begin{align*}
    \E[M_{t+1} | \mc{F}_t ] = M_t + \E[ \textbf{1}\{x_s = x, Q_s(x) = 1\} - \nu(x)P_s(x)| \mc{F}_t ] = M_t
\end{align*}
Applying Doob's equality $\E[M_{\mc{U}}] = \E[M_0] = 0$. Consequence: 
\begin{align*}
    \E[\mc{L}_x] = \E\left[\sum_{s=1}^{\mc{U}}\textbf{1}\{x_s = x, Q_s(x) = 1\}\right] = \nu(v)\E\left[\sum_{s=1}^{\mc{U}}P_s(x)\right]
\end{align*}
Define $\alpha(x) := \frac{\E\left[\sum_{s=1}^{\mc{U}}P_s(x)\right]}{\E[\mc{U}]}$ and note that each $\alpha_x \in [0,1]$. Then $\E[ \mc{L}_x ]= \E[\mc{U}] \alpha(x) \nu(x)$ so
applying equation (18) of \cite{kaufmann2016complexity} again, we have
\begin{align*}
    \log(1/2.4\delta) \leq& \min_{\theta \in \mc{C}} \sum_{x\in\mc{X}}\E[\mc{L}_x]\text{KL}(\mc{N}_{\theta_*, x}, \mc{N}_{\theta, x})  \\
    =& \min_{\theta \in \mc{C}} \sum_{x\in\mc{X}}\E[\mc{L}_x] \, \|\theta - \theta_*\|_{x x^\top}^2 / 2 \\
    =& \min_{z \in \mc{Z} \setminus z_*}  \frac{\langle \theta_*, z_* - z\rangle^2}{2 \|z-z_*\|_{ ( \sum_{x \in \mc{X}} \E[\mc{L}_x] x x^\top )^{-1} }^2 }  \\
    =& \min_{z \in \mc{Z} \setminus z_*}  \frac{\langle \theta_*, z_* - z\rangle^2}{2 \|z-z_*\|_{ ( \sum_{x \in \mc{X}} \nu(x) \alpha(x) x x^\top )^{-1} }^2 }  \E[ \mc{U} ] .
\end{align*}
Rearranging, and applying the identity $\E_{X \sim \nu}[ \alpha(X) X X^\top ] = \sum_{x \in \mc{X}} \nu(x) \alpha(x) x x^\top$, the above implies that
\begin{align*}
    \E[ \mc{U} ] \geq \max_{z \in \mc{Z} \setminus z_*}  \frac{2 \|z-z_*\|_{ \E_{X \sim \nu}[ \alpha(X) X X^\top ]^{-1} }^2 }{\langle \theta_*, z_* - z\rangle^2} \log(1/2.4\delta) .
\end{align*}
Noting that the total expected number of labels is equal to 
\begin{align*}
    \E[\mc{L}] = \sum_{x \in \mc{X}} \E[ \mc{L}_x ] = \sum_{x \in \mc{X}} \E[\mc{U}] \alpha(x) \nu(x) = \E[\mc{U}] \, \E_{X \sim \nu}[ \alpha(X) ]
\end{align*}
we conclude that
\begin{align*}
    \E[\mc{L}] \geq \min_{\alpha: \mc{X} \rightarrow [0,1]} \quad & \E[\mc{U}] \ \E_{X\sim\nu}[\alpha(X)] \\ \text{ subject to }&\quad \E[\mc{U}] \geq  \max_{z \in \mc{Z} \setminus \{z_*\}} \frac{2 \|z-z_*\|_{\E_{X \sim \nu}[\alpha(X) X X^\top]^{-1}}^2}{\langle \theta_*, z_* - z \rangle^2}\log(1/2.4\delta).
\end{align*}
The second bullet point result follows by denoting $\alpha$ as $P$ and applying  Proposition~\ref{prop:reparameterization}.

\section{Selective Sampling Algorithm for Known Distribution $\nu$}

\subsection{Proof of Theorem~\ref{thm:upper_bound_tau_input}, upper bound}\label{sec:proof_upperbound}

At each round $\ell$ we assume an implementation such that $\widehat{P}_\ell, \widehat{\Sigma}_{\widehat{P}_\ell} \leftarrow $\textsc{OptimizeDesign}$(\mc{\mc{Z}_\ell},2^{-\ell},\tau)$ returns the solution of Equation~\ref{eqn:ideal_opt_problem} with $\epsilon = 2^{-\ell}$, essentially.
More explicitly, let $\epsilon_\ell := 2^{-\ell}$, $B< \infty$ such that $\max_{x \in \mc{X}} |\langle x,\theta_* \rangle| \leq B$, and $\sigma < \infty$ such that $\E[(y_s - \langle \theta_*, x_s \rangle)^2|x_s] \leq \sigma^2$.
If
$$\beta_{\delta,\ell} := 16(B^2 + \sigma^2) \log(2 \ell^2 |\mc{Z}|^2/\delta)$$ 
then $\widehat{P}_\ell = P_\ell$ where
\begin{align*}
    \displaystyle 
    P_{\ell} := \argmin_{P: \mc{X} \rightarrow [0,1]} \E_{X \sim \nu}[P(X)]\text{ subject to }
    \max_{z,z' \in \mc{\mc{Z}_\ell}} \frac{\|z-z'\|^2_{\E_{X\sim\nu}[ \tau P(X) X X^\top]^{-1}}}{\epsilon_\ell^2} \beta_{\delta,\ell} \leq  1 
\end{align*} 
and $\widehat{\Sigma}_{\widehat{P}_\ell} := \E_{X \sim \nu}[ P_\ell(X) X X^\top ]$

We first provide an intermediate lemma on the correctness of Algorithm~\ref{algo:best_arm_meta} that relies on the feasibility of $P_\ell$ which we will show shortly.
\begin{lemma}
With probability at least $1-\delta$ we have for all stages $\ell \in \mathbb{N}$ such that $P_\ell$ is feasible, that $z_* \in \mc{Z}_\ell$ and $\max_{z \in \mc{Z}_\ell} \langle z_* - z, \theta_* \rangle \leq 4\epsilon_\ell$. 
\end{lemma}
\begin{proof}
Define the event $\mathcal{E}$ as
\begin{align*}
    \mc{E} := \bigcap_{\ell=1}^\infty \bigcap_{z,z' \in \mc{Z}_\ell} \left\{  |\langle z-z', \widehat{\theta}_\ell - \theta_* \rangle| \leq  \epsilon_\ell \right\}
\end{align*}

By Lemma \ref{lmm:intersection_good_event}, we know that $\P(\mc{E})\geq 1-\delta$. Then, the rest of the proof is the same as the one in \cite{fiez2019sequential}, but we include it here for completeness. Assume that $\mc{E}$ holds. Then for any $z'\in \mc{Z}_\ell$
\begin{align*}
    \langle z' - z^*, \widehat{\theta}_\ell \rangle 
    &= \langle z' - z^*, \widehat{\theta}_\ell - \theta^* \rangle + \langle z' - z^*, \theta^* \rangle\\
    &= \langle z' - z^*, \widehat{\theta}_\ell - \theta^* \rangle \\
    &\leq \epsilon_\ell
\end{align*}
so that $z^*$ would survive to round $\mc{Z}_{\ell+1}$. And for any $z\in\mc{Z}_\ell$ such that $\langle z^*-z, \theta^* \rangle > 2\epsilon_\ell$, we have 
\begin{align*}
    \max_{z'\in\mc{Z}_\ell}\langle z' - z, \widehat{\theta}_\ell \rangle
    &\geq \langle z^* - z, \widehat{\theta}_\ell \rangle\\
    &= \langle z^* - z, \widehat{\theta}_\ell - \theta^* \rangle + \langle z^* - z, \theta^*\rangle\\
    &> -\epsilon_\ell + 2\epsilon_\ell\\
    &= \epsilon_\ell
\end{align*}
which implies this $z$ would be kicked out. Note that this implies that $\max_{z\in\mc{Z}_{\ell+1}}\langle z^*-z, \theta^* \rangle \leq 2\epsilon_\ell = 4\epsilon_{\ell+1}$.
\end{proof}
We can now prove Theorem~\ref{thm:upper_bound_tau_input}.
After $L := \lceil \log_2(\frac{4}{\Delta}) \rceil$ rounds $\mc{Z}_\ell = \{z_*\}$ by the above lemma.
Thus, the total number of labels requested after $L$ rounds is equal to $\mc{L} := \sum_{\ell=1}^L \sum_{t=(\ell-1)\tau+1}^{\ell \tau} Q_\ell(x_t)$. 
By Freedman's inequality (c.f., Theorem 1 of \cite{beygelzimer2011contextual}) we have that
\begin{align*}
    \sum_{\ell=1}^L \sum_{t=(\ell-1)\tau+1}^{\ell \tau} Q_\ell(x_t) \leq 2 \sum_{\ell=1}^L \tau \E_{X \sim \nu}[ P_\ell(X) | \mc{Z}_{\ell} ] + \log(1/\delta)
\end{align*}

We can now bound the expected sample complexity of this algorithm.
\begin{align*}
    &\sum_{\ell=1}^{L} \tau \E_{X \sim \nu}[ P_\ell(X) | \mc{Z}_{\ell} ]  \\
    &= \sum_{\ell=1}^{L} \left[  \min_{P: \mc{X} \rightarrow [0,1]} \tau \E_{X \sim \nu}[P(X)]\quad\text{ subject to }\quad\max_{z,z' \in \mc{\mc{Z}_\ell}} \frac{\|z-z'\|_{\E_{X\sim\nu}[ \tau P(X) X X^\top]^{-1}}^2  }{\epsilon_\ell^2}\beta_{\delta,\ell} \leq  1 \right] .
\end{align*}
Using Lemma~\ref{lmm:signal_to_noise_bound}, we have
\begin{align*}
    \max_{z,z' \in \mc{\mc{Z}_\ell}} \frac{\|z-z'\|_{\E_{X\sim\nu}[ \tau P(X) X X^\top]^{-1}}^2  }{\epsilon_\ell^2}\beta_{\delta,\ell} 
    &\leq \beta_{\delta,L} \max_{z,z' \in \mc{\mc{Z}_\ell}} \frac{\|z-z'\|_{\E_{X\sim\nu}[ \tau P(X) X X^\top]^{-1}}^2  }{\epsilon_\ell^2} \\
    &\leq 64\beta_{\delta,L}\max_{z \in \mc{Z} \setminus z_*} \frac{\|z-z_*\|_{\E_{X\sim\nu}[ \tau P(X) X X^\top]^{-1}}^2  }{\langle z - z_*,\theta_* \rangle^2 } \\
    &=: \max_{z \in \mc{Z} \setminus z_*} \frac{\|z-z_*\|_{\E_{X\sim\nu}[ \tau P(X) X X^\top]^{-1}}^2  }{\langle z - z_*,\theta_* \rangle^2 } \beta_\delta
\end{align*}
Note that the last line also describes a condition for which a $P_\ell$ is feasible.  
Indeed, at round $\ell$, a sufficient condition for a feasible $P_\ell$ (i.e., the RHS $\leq 1$) is if $\tau$ exceeds $\rho(\nu) \beta_\delta$ with $\beta_\delta := 1024(B^2 + \sigma^2) \log(2 L^2 |\mc{Z}|^2/\delta)$ and $\rho(\nu) = \max_{z \in \mc{Z} \setminus z_*} \frac{\|z-z_*\|_{\E_{X\sim\nu}[  X X^\top]^{-1}}^2  }{\langle z - z_*,\theta_* \rangle^2 }$, which holds by assumption in the theorem.

Plugging this constraint back into above we have
\begin{align*}
    &\sum_{\ell=1}^{L} \tau \E_{X \sim \nu}[ P_\ell(X) | \mc{Z}_{\ell} ]  \\
    &\leq \sum_{\ell=1}^{L} \left[  \min_{P: \mc{X} \rightarrow [0,1]} \tau \E_{X \sim \nu}[P(X)]\quad\text{ subject to }\quad\max_{z \in \mc{Z} \setminus z_*} \frac{\|z-z_*\|_{\E_{X\sim\nu}[ \tau P(X) X X^\top]^{-1}}^2  }{\langle z - z_*,\theta_* \rangle^2 } \beta_{\delta} \leq  1 \right] \\
    &\leq L \min_{\lambda \in \triangle_{\mc{X}}} \rho(\lambda) \beta_\delta \quad \text{ subject to }\quad \|\lambda/\nu\|_{\infty} \rho(\lambda) \beta_\delta \leq \tau
\end{align*}
where the last line follows by applying the reparameterization of Proposition~\ref{prop:reparameterization}.

\subsubsection{High-probability Events}
\label{sec:appendix_correctness}

\begin{lemma}\label{lmm:intersection_good_event}
We have $\P( \mc{E} ) \geq 1-\delta$.
\end{lemma}
\begin{proof}
For any $\mc{V} \subseteq \mc{Z}$ and $z,z' \in \mc{V}$ define
\begin{align*}
\mc{E}_{z,z',\ell}( \mc{V} ) = \{ |\langle z-z', \widehat{\theta}_\ell(\mc{V}) - \theta_* \rangle|  \leq \epsilon_\ell \}
\end{align*}
where $\widehat{\theta}_\ell( \mc{V} )$ is the estimator that would be constructed by the algorithm at stage $\ell$ with $\mc{Z}_\ell = \mc{V}$.
For fixed $\mc{V} \subset \mc{Z}$ and $\ell \in \mathbb{N}$ we apply Proposition~\ref{prop:rips_bound} so that with probability at least $1-\frac{\delta}{ \ell^2 |\mc{Z}|^2}$ we have that for any $z, z' \in \mc{V}$
\begin{align*}
    |\langle z-z', \widehat{\theta}_\ell(\mc{V}) - \theta_* \rangle| &\leq \|z-z'\|_{\E_{X\sim\nu}[ \tau P_\ell(X) X X^\top]^{-1}} \sqrt{16(B^2 + \sigma^2) \log(2 \ell^2 |\mc{Z}|^2/\delta)}  \\
    &\leq \epsilon_\ell
\end{align*}
Noting that $ \mc{E} := \bigcap_{\ell=1}^\infty \bigcap_{z,z' \in \mc{Z}_\ell} \mc{E}_{z,z',\ell}( \mc{Z}_\ell )  $ we have
\begin{align*}
\P\left( \bigcup_{\ell=1}^\infty \bigcup_{z,z' \in \mc{Z}_\ell} \{ \mc{E}^c_{z,z',\ell}( \mc{Z}_\ell ) \} \right) &\leq \sum_{\ell=1}^\infty \P\left( \bigcup_{z,z' \in \mc{Z}_\ell} \{ \mc{E}^c_{z,z',\ell}( \mc{Z}_\ell) \} \right) \\
&= \sum_{\ell=1}^\infty \sum_{\mc{V} \subseteq \mc{Z}} \P\left( \bigcup_{z,z' \in \mc{V}} \{ \mc{E}^c_{z,z',\ell}( \mc{V} ) \} , {\mc{Z}}_\ell = \mc{V}\right) \\
&= \sum_{\ell=1}^\infty \sum_{\mc{V} \subseteq \mc{Z}} \P\left(\bigcup_{z,z' \in \mc{V}} \{ \mc{E}^c_{z,z',\ell}( \mc{V} ) \} \right) \P( {\mc{Z}}_\ell = \mc{V}) \\
&\leq \sum_{\ell=1}^\infty  \sum_{\mc{V} \subseteq \mc{Z}} \tfrac{ \delta}{\ell^2 |\mc{Z}|^2} \binom{|\mc{V}|}{2} \P( {\mc{Z}}_\ell = \mc{V} ) \\
&\leq \sum_{\ell=1}^\infty  \sum_{\mc{V} \subseteq \mc{Z}} \tfrac{ \delta}{2 \ell^2}  \P( \mc{Z}_\ell= \mc{V})  \leq \delta
\end{align*}
\end{proof}

\subsection{Technical Lemmas}

The following definition characterizes the RIPS estimator we used in Algorithm \ref{algo:best_arm_meta}.

\begin{definition}\label{def:robust_estimator}
 Let $X_1,\ldots, X_n$ be i.i.d. random variables with mean $\bar{x}$ and variance $\nu^2$. Let $\delta \in (0,1)$. We say that $\widehat{\mu}(X_1,\ldots,X_n)$ is a \emph{$\delta$-robust estimator} if there exist universal constants $c_1,\HighProbConst > 0$ such that if $n \geq c_1 \log(1/\delta)$, then with probability at least $1-\delta$
\begin{align*}
|\widehat{\mu}( \{X_t\}_{t=1}^n )- \bar{x}| \leq \HighProbConst \sqrt{\frac{\nu^2 \log(1/\delta)}{n}}.
\end{align*}
\end{definition}
\noindent Examples of $\delta$-robust estimators include the median-of-means estimator and Catoni's estimator \cite{lugosi2019mean}. 
This work employs the use of the Catoni estimator which satisfies $|\widehat{\mu}( \{X_t\}_{t=1}^n )- \bar{x}| \leq \sqrt{ \frac{2 \nu^2 \log(1/\delta)}{n - 2 \log(1/\delta)}}$ for $n > 2 \log(1/\delta)$ which leads to an optimal leading constant as $n \rightarrow \infty$. See \cite{camilleri2021highdimensional} or \cite{lugosi2019mean} for more details.
\begin{proposition}\label{prop:rips_bound}

Let $x_1,\dots,x_n$ be drawn IID from a distribution $\nu$. 
Assume that $|\langle \theta,x_s \rangle| \leq B$ and $\E[ |\langle \theta,x_s \rangle-y_s|^2 ] \leq \sigma^2$. Let $P : \mc{X} \rightarrow [0,1]$ be arbitrary. 
Let $Q(x_s) \sim \text{Bernoulli}(P(x_s))$ independently for all $s \in [n]$.
For a given finite set $\mc{V} \subset \R^d$ define for any $v \in \mc{V}$
$$w_v = \mathrm{Catoni}( \{ \langle v, \E_{X \sim \nu}[P(X) X X^\top]^{-1} Q(x_s) x_s y_s \rangle \}_{s=1}^n ).$$
If $\widehat{\theta} = \arg\min_\theta \max_v \frac{|w_v - \langle \theta, v \rangle|}{\| v \|_{\E_{X \sim \nu}[P(X) X X^\top]^{-1}}}$ and $n\geq 4\log(2|\mc{V}|/\delta)$, then with probability at least $1-\delta$, it holds that
\begin{align*}
    |\langle v, \widehat{\theta}-\theta \rangle| \leq \| v \|_{\E_{X \sim \nu}[n P(X) X X^\top]^{-1}} \sqrt{16 (B^2+\sigma^2)\log(2|\mc{V}|/\delta)}
\end{align*}
\end{proposition}
\begin{proof}
Inspired by \cite{camilleri2021highdimensional}, we note that 
\begin{align*}
     \max_{v \in \mc{V}} \frac{| \langle \widehat{\theta}, v \rangle - \langle \theta, v \rangle|}{\| v \|_{\E_{X \sim \nu}[n P(X) X X^\top]^{-1}}} &= \max_{v \in \mc{V}} \frac{| \langle \widehat{\theta}, v \rangle - w_v + w_v - \langle \theta, v \rangle|}{\| v \|_{\E_{X \sim \nu}[n P(X) X X^\top]^{-1}}} \\
     &\leq \max_{v \in \mc{V}} \frac{| \langle \widehat{\theta}, v \rangle - w_v|}{\| v \|_{\E_{X \sim \nu}[n P(X) X X^\top]^{-1}}} + \max_{v \in \mc{V}} \frac{| w_v - \langle \theta, v \rangle|}{\| v \|_{\E_{X \sim \nu}[n P(X) X X^\top]^{-1}}} \\
     &= \min_{\theta} \max_{v \in \mc{V}} \frac{ | \langle \theta, v\rangle - w_v| }{\| v \|_{\E_{X \sim \nu}[n P(X) X X^\top]^{-1}}} + \max_{v \in \mc{V}} \frac{| w_v - \langle \theta, v \rangle|}{\| v \|_{\E_{X \sim \nu}[n P(X) X X^\top]^{-1}}} \\
     &\leq 2 \max_{v \in \mc{V}} \frac{ | \langle \theta, v \rangle - w_v| }{\| v \|_{\E_{X \sim \nu}[n P(X) X X^\top]^{-1}}}
\end{align*}
So it suffices to show that each $| \langle \theta, v \rangle - w_v|$ is small. We begin by fixing some $v\in\mc{V}$ and bounding the variance of $v^\top \E_{X \sim \nu}[P(X) X X^\top]^{-1} Q(x_s)x_s y_s$ for any $s \leq n$ which is necessary to use the robust estimator. For readability purposes, we shorten $\E_{x_s \sim \nu, Q(x_s)\sim P(x_s)}$ as $\E_{x_s, Q}$ in the rest of this proof. Note that 
\begin{align*}
&\mathbb{V}\text{ar}_{x_s \sim \nu, Q(x_s)\sim P(x_s)}( v^\top \E_{X \sim \nu}[P(X) X X^\top]^{-1} Q(x_s)x_s y_s ) \\
= &\E_{x_s, Q}[ (v^\top \E_{X \sim \nu}[P(X) X X^\top]^{-1} Q(x_s)x_s y_s)^2 ]\\
&\qquad-\E_{x_s, Q}[ v^\top \E_{X \sim \nu}[P(X) X X^\top]^{-1} Q(x_s)x_s y_s ]^2 
\end{align*}
which means we can drop the second term to bound the variance by 
\begin{align*}
    &\E_{x_s, Q}[ \left( (v^\top \E_{X \sim \nu}[P(X) X X^\top]^{-1} Q(x_s)x_s y_s \right)^2 ]\\
    &= \E_{x_s, Q}[ \left( v^\top \E_{X \sim \nu}[P(X) X X^\top]^{-1} Q(x_s)x_s (x_s^\top \theta + \xi_s ) \right)^2 ] \\
    &= \E_{x_s, Q}[ \left( v^\top \E_{X \sim \nu}[P(X) X X^\top]^{-1} Q(x_s)x_s (  x_s^\top \theta ) \right)^2 ] \\
    &\qquad+ \E_{x_s, Q}[ \left( v^\top \E_{X \sim \nu}[P(X) X X^\top]^{-1} Q(x_s)x_s \right)^2 \xi_t^2 ] \\
    &\leq B^2 \E_{x_s, Q}[ \left( v^\top \E_{X \sim \nu}[P(X) X X^\top]^{-1} Q(x_s)x_s \right)^2 ] \\
    &\qquad+ \sigma^2 \E_{x_s, Q}[ \left( v^\top \E_{X \sim \nu}[P(X) X X^\top]^{-1} Q(x_s)x_s \right)^2 ] \\
    &= \E_{x_s \sim \nu}\left[(B^2 + \sigma^2) \E_{Q(x_s)\sim P(x_s)}[ v^\top \E_{X \sim \nu}[P(X) X X^\top]^{-1} Q(x_s)x_s x_s^\top Q(x_s)\E_{X \sim \nu}[P(X) X X^\top]^{-1} v]\right] \\
    &\overset{\text{(i)}}{=} \E_{x_s \sim \nu}\left[(B^2 + \sigma^2) \E_{Q(x_s)\sim P(x_s)}[ v^\top \E_{X \sim \nu}[P(X) X X^\top]^{-1} Q(x_s)x_s x_s^\top \E_{X \sim \nu}[P(X) X X^\top]^{-1} v]\right] \\
    &\leq \E_{x_s \sim \nu}\left[(B^2 + \sigma^2) v^\top \E_{X \sim \nu}[P(X) X X^\top]^{-1} P(x_s) x_s x_s^\top \E_{X \sim \nu}[P(X) X X^\top]^{-1} v]\right], 
\end{align*}
where we used that $Q(x_s)^2 = Q(x_s)$ in equality (i) above. Thus, we have
\begin{align*}
    &\mathbb{V}\text{ar}( v^\top \E_{X \sim \nu}[P(X) X X^\top]^{-1} Q(x_s)x_s y_s )\\
    \leq& (B^2 + \sigma^2) v^\top (\E_{X \sim \nu}[P(X) X X^\top]^{-1} \E_{x_s \sim \nu} [P(x_s) x_s x_s^\top] (\E_{X \sim \nu}[P(X) X X^\top]^{-1}) v \\ 
    =& (B^2 + \sigma^2) \|v\|^2_{(\E_{X \sim \nu}[P(X) X X^\top]^{-1}}
\end{align*}
By using the property of Catoni estimator stated in Definition~\ref{def:robust_estimator}, we have $c_0=\sqrt{2}$ and 
\begin{align*}
&\abs{\inner{\theta_*, v}-w_v} \\
=&|\mathrm{Catoni}( \{ \langle v, \E_{X \sim \nu}[P(X) X X^\top]^{-1} Q(x_s) x_s y_s \rangle \}_{s=1}^n ) - \E[ \langle v, \E_{X \sim \nu}[P(X) X X^\top]^{-1} Q(x_s) x_s y_s \rangle ]|\\
\leq & \sqrt{2} \sqrt{\Sp{\mathbb{V}\text{ar}(\langle v, \E_{X \sim \nu}[P(X) X X^\top]^{-1} Q(x_s) x_s y_s \rangle)}\frac{\log(\tfrac{2}{\delta})}{n/2}}\tag{with probability at least $1-\delta$ if $n\geq 4\log(2/\delta)$}\\
\leq&\|v\|_{(\E_{X \sim \nu}[P(X) X X^\top]^{-1}}\sqrt{\frac{4 (B^2+\sigma^2)\log(\tfrac{2}{\delta})}{n}}\\
=&\| v \|_{\E_{X \sim \nu}[n P(X) X X^\top]^{-1}}\sqrt{4 (B^2+\sigma^2)\log(2/\delta)}.
\end{align*}
Finally, the proof is complete by taking union bounding over all $v\in \mc{V}$.
\end{proof}

\begin{lemma}\label{lmm:signal_to_noise_bound}
Holds
\begin{align*}
    \max_{z,z' \in \mc{\mc{Z}_\ell}} &\frac{\|z-z'\|_{\E_{X\sim\nu}[ \tau P(X) X X^\top]^{-1}}^2  }{\epsilon_\ell^2}\leq 64\max_{z \in \mc{Z} \setminus z_*} \frac{\|z-z_*\|_{\E_{X\sim\nu}[ \tau P(X) X X^\top]^{-1}}^2  }{\langle z - z_*,\theta_* \rangle^2 }
\end{align*}
\end{lemma}

\begin{proof}
Let $\mc{S}_\ell = \{ z \in \mc{Z} : \langle z_* -z ,\theta_* \rangle \leq 4 \epsilon_\ell \}$. We have 
\begin{align*}
    \max_{z,z' \in \mc{\mc{Z}_\ell}} \frac{\|z-z'\|_{\E_{X\sim\nu}[ \tau P(X) X X^\top]^{-1}}^2  }{\epsilon_\ell^2}
    &\leq \max_{z,z' \in \mc{S_\ell}} \frac{\|z-z'\|_{\E_{X\sim\nu}[ \tau P(X) X X^\top]^{-1}}^2  }{\epsilon_\ell^2} \\
    &= 16\max_{z,z' \in \mc{S_\ell}} \frac{\|z-z'\|_{\E_{X\sim\nu}[ \tau P(X) X X^\top]^{-1}}^2  }{(4\epsilon_\ell)^2} \\
    &\leq 64\max_{z \in \mc{S_\ell}} \frac{\|z-z_*\|_{\E_{X\sim\nu}[ \tau P(X) X X^\top]^{-1}}^2  }{(4\epsilon_\ell)^2}\\
    &= 64\max_{z \in \mc{S_\ell} \setminus z_*} \frac{\|z-z_*\|_{\E_{X\sim\nu}[ \tau P(X) X X^\top]^{-1}}^2  }{\max\{ (4\epsilon_\ell)^2, \langle z - z_*,\theta_* \rangle^2 \}} \\
    &\leq 64\max_{z \in \mc{Z} \setminus z_*} \frac{\|z-z_*\|_{\E_{X\sim\nu}[ \tau P(X) X X^\top]^{-1}}^2  }{\langle z - z_*,\theta_* \rangle^2 } .
\end{align*}
\end{proof}

\subsubsection{Reparameterization}

\begin{proposition}\label{prop:reparameterization}
Fix $\nu \in \triangle_{\mc{X}}$  and any $\lambda \in \triangle_{\mc{X}}$. Define $\| \lambda/\nu \|_\infty = \sup_{x \in \mc{X}} \lambda(x)/\nu(x)$ and 
$\rho(\lambda) = \max_{z \neq z_*} \frac{\|z-z_*\|_{\E_{X\sim\lambda}[X X^\top]^{-1}}^2}{\langle z_* - z, \theta_* \rangle^2 }$. For any $t,\beta \in \R_+$ the following optimization problems achieve the same value
\begin{itemize}
    \item $\displaystyle\min_{P: \mc{X} \rightarrow [0,1]} t\, \E_{X \sim \nu}[ P(X) ]$ subject to $\max_{z \neq z_*} \frac{\|z-z_*\|_{\E_{X\sim\nu}[ P(X) X X^\top]^{-1}}^2}{\langle z_* - z, \theta_* \rangle^2 } \beta \leq t$
    \item $\displaystyle\min_{\lambda \in \triangle_{\mc{X}}} \rho(\lambda) \beta \quad \text{ subject to }\quad \|\lambda/\nu\|_{\infty} \rho(\lambda) \beta \leq t$
\end{itemize}
\end{proposition}

Let us first prove a simple lemma.
\begin{lemma}
Let $\mc{P}$ denote the set of all functions $P: \mc{X} \rightarrow [0,1]$.
And for any $\nu \in \triangle_{\mc{X}}$ with support $\mc{X}$ let $\mc{P}' = \{ \kappa \lambda_x/\nu_x  : \lambda \in \triangle_{\mc{X}}, \kappa \geq 0 : \kappa \lambda_x /\nu_x \in [0,1] \}$.
Then $\mc{P} = \mc{P}'$.
\end{lemma}
\begin{proof}
Fix any $P \in \mc{P}$. If $\lambda_x = P_x \nu_x / \|P \circ \nu\|_1$ and $\kappa = \|P \circ \nu\|_1$ then $\kappa \lambda / \nu \in \mc{P}'$ and is equal to $P$.
This implies $\mc{P} \subseteq \mc{P}'$.

For the other direction, fix any $\lambda \in \triangle_{\mc{X}}$ and $\kappa \geq 0$ such that $\kappa \lambda_x / \nu_x \in [0,1]$ for all $x$. If $P =  \kappa \lambda / \nu$ then $P \in \mc{P}$ which implies $\mc{P}' \subseteq \mc{P}$ and concludes the proof.
\end{proof}

\begin{proof}[Proof of Proposition~\ref{prop:reparameterization}]
Using the above lemma we have that
\begin{align*}
\min_{P: \mc{X} \rightarrow [0,1]} t\, \E_{X \sim \nu}[ P(X) ]\quad\text{ subject to }\quad\max_{z \neq z_*} \frac{\|z-z_*\|_{\E_{X\sim\nu}[ P(X) X X^\top]^{-1}}^2}{\langle z_* - z, \theta_* \rangle^2 } \beta \leq t
\end{align*}
is equivalent to
\begin{align*}
    \min_{\kappa \geq 0, \lambda \in \triangle_{\mc{X}}} t\, \E_{X \sim \nu}[ \kappa \lambda(X) / \nu(X) ]\quad\text{ subject to }\quad&\max_{z \neq z_*} \frac{\|z-z_*\|_{\E_{X\sim\nu}[ \kappa \lambda(X)/\nu(X) X X^\top]^{-1}}^2}{\langle z_* - z, \theta_* \rangle^2 } \beta \leq t\\
    &\kappa \lambda(x) /\nu(x) \leq 1 \quad \forall x \in \mc{X}
\end{align*}
which is equal to, after simplifying,
\begin{align*}
    \min_{\kappa \geq 0, \lambda \in \triangle_{\mc{X}}} t\, \kappa \quad\text{ subject to }\quad&\max_{z \neq z_*} \frac{\|z-z_*\|_{\E_{X\sim\lambda}[  X X^\top]^{-1}}^2}{\langle z_* - z, \theta_* \rangle^2 } \beta \leq t \kappa\\
    &\kappa \lambda(x) /\nu(x) \leq 1 \quad \forall x \in \mc{X}
\end{align*}
which is equal to 
\begin{align*}
    \min_{u \geq 0, \lambda \in \triangle_{\mc{X}}} u \quad\text{ subject to }\quad& \rho(\lambda) \beta \leq u\\
    &\| \lambda/\nu\|_\infty \leq \frac{t}{u}.
\end{align*}
Note, there exists a feasible $(\lambda,u)$ precisely when there exists a $\lambda \in \triangle_{\mc{X}}$ such that $\| \lambda/\nu\|_\infty \rho(\lambda) \leq t$, in which case the optimization problem is equal to
\begin{align*}
    \min_{\lambda \in \triangle_{\mc{X}}} \rho(\lambda) \beta \quad\text{ subject to }\quad&  \| \lambda/\nu\|_\infty \rho(\lambda) \beta \leq t
\end{align*}
\end{proof}

\section{Analysis of the Optimization Problem}
\label{sec:analysis_opt}
\subsection{Proof of Theorem \ref{theo:opt}}
\label{sec:opt_proof}
For simplicity, we will use $\mu$ instead of $\mu_b$ to denote the number that controls the intensity of barrier function.

The proof relies on analyzing another function $\overline{D}:\R^{d\times d}_{\succeq\bm{0}}\mapsto\R$. For simplicity, first, we define
\begin{equation}
    \label{equ:h_func}
    h_{\Lambda}(x)=P_{\Lambda}(x)-\mu\Sp{\log(1-P_{\Lambda}(x))+\log(P_{\Lambda}(x))}-P_{\Lambda}(x)x^\top\Lambda x.
\end{equation}
Recall that our dual objective is $D(\bm{\Lambda})=\E_{X\sim\nu}\Mp{h_\Lambda(X)}+\frac{1}{c_\ell^2}\sum_{y\in\mathcal{Y}_\ell}y^\top\Lambda_y y$. Since the first term in $\E_{X\sim\nu}\Mp{h_\Lambda(X)}$ only depends on $\Lambda=\sum_{y\in\mathcal{Y}_\ell}\Lambda_y$, we can consider the following optimization problem.
\begin{equation}
    \label{equ:opt_semi}
\begin{array}{rl}
    f(\Lambda)=\max_{\Lambda_y} & \sum_{y\in\mc{Y}_\ell}y^\top\Lambda_{y}y \\
    \text{subject to} & \sum_{y\in\mc{Y}_\ell}\Lambda_{y}=\Lambda\\
     & \Lambda_{y}\succeq\mathbf{0}, \quad\forall y\in\mc{Y}_\ell.
\end{array}
\end{equation}
Then, the alternative dual objective $\overline{D}(\Lambda)$ is defined as $\overline{D}(\Lambda)=\E_{X\sim\nu}\Mp{h_\Lambda(X)}+\frac{1}{c_\ell^2}f(\Lambda)$. We can immediately see that maximizing $\overline{D}(\cdot)$ is equivalent to maximizing $D(\cdot)$. In particular, let $\Lambda^*\in \argmax_{\Lambda\succeq\bm{0}}\overline{D}(\Lambda)$ and $\Sp{\Lambda^*_y}_{y\in\mc{Y}_\ell}$ be the set of PSD matrices that solve problem \eqref{equ:opt_semi} and evaluate $f(\Lambda^*)$. We can see that $\Sp{\Lambda^*_y}_{y\in\mc{Y}_\ell}$ also maximizes $D(\cdot)$. Conversely, for $\bm{\Lambda}^*=\Sp{\Lambda^*_y}_{y\in\mc{Y}_\ell}\in\argmax_{\Lambda_y\succeq\bm{0}, \forall y} D(\bm{\Lambda})$, we also have $\sum_{y\in\mc{Y}_\ell}\Lambda^*_y\in\argmax_{\Lambda\succeq\bm{0}}\overline{D}(\Lambda)$.


Further, we also define their empirical version $D_E$ and $\overline{D}_E$ with extra i.i.d. samples $x_1, \dots, x_u$ as
\begin{equation}
    \label{equ:D_empirical}
    D_E(\bm{\Lambda})=\frac{1}{u}\sum_{i=1}^{u}h_{\Lambda}(x_i)+\frac{1}{c_\ell^2}\sum_{y\in\mathcal{Y}_\ell}y^\top\Lambda_y y\quad\text{and}\quad\overline{D}_E(\Lambda)=\frac{1}{u}\sum_{i=1}^{u}h_{\Lambda}(x_i)+\frac{1}{c_\ell^2}f(\Lambda).
\end{equation}
Recall that the problem Algorithm \ref{algo:solve_sgd} tries to solve is
\begin{equation}
    \label{equ:opt_barrier_2}
    \begin{array}{rl}
		\min_{P} & \E_{X\sim\nu}[P(X)-\mu(\log(1-P(X))+\log(P(X)))] \\
		\text{subject to} &  \E_{X\sim\nu}\Mp{P(X)XX^\top}\succeq\frac{1}{c_\ell^2}yy^\top, \quad\forall y\in\mathcal{Y}_\ell.
	\end{array}
\end{equation}

We will restate a more precise version of Theorem \ref{theo:opt} and then prove it.
\begin{theorem}
\label{theo:opt_detail}
Suppose $\Norm{x}_2\leq M$ for any $x\in\mathrm{supp}(\nu)$ and $\Sigma=\E_{X\sim\nu}\Mp{XX^\top}$ is invertible. Let $\bm{\Lambda}^*\in\argmax_{\Lambda_y\succeq\bm{0}, \forall y}D(\bm{\Lambda})$ and $\kappa(\Sigma)=\frac{\lambda_{\max}(\Sigma)}{\lambda_{\min}(\Sigma)}$ be condition number. Assume $\Norm{\Lambda^*}_F>0$ and define $\omega=\min_{\Gamma\in\mathbb{S}^d:\Norm{\Gamma}_F=1}\E_{X\sim\nu}\Mp{\Sp{X^\top\Gamma X}^2}$, where $\mathbb{S}^d$ is the set of $d\times d$ symmetric matrices. Let $\abs{\mc{Y}_\ell}C_\ell^2=\frac{1}{c_\ell^2}\sum_{y\in\mc{Y}_\ell}\Norm{y}_2^4$.

Then, $\Lambda^*=\sum_{y\in\mc{Y}_\ell}\Lambda^*_y$ is unique. Further, for any $\epsilon>0$ and $\delta>0$, suppose it holds that
\begin{align*}
    \mu&\leq \min\Bp{ \sqrt{\frac{3\kappa(\Sigma)\Norm{\Lambda^*}_F M^2}{8}\cdot\frac{1+\epsilon}{\epsilon}}, \frac{4}{9}\Norm{\Lambda^*}_F^2M^4, \frac{1}{2\sqrt{3}}}\\
    K&\geq\frac{288\kappa(\Sigma)^2\abs{\mc{Y}_\ell}^3\Norm{\Lambda^*}_F^4M^4(M^4+C_\ell^2)\cdot\Sp{2\Norm{\Lambda^*}_FM^2+1}^4\log(6/\delta)}{\omega^2\mu^6}\cdot\Sp{\frac{1+\epsilon}{\epsilon}}^2\\
    u&\geq \frac{576\kappa(\Sigma)^2\Norm{\Lambda^*}_F^2M^8\cdot\Sp{2\Norm{\Lambda^*}_FM^2+1}^4\log(6/\delta)}{\omega^2\mu^6}\cdot\Sp{\frac{1+\epsilon}{\epsilon}}^2.
\end{align*}

Then, with probability at least $1-\delta$, Algorithm \ref{algo:solve_sgd} will output  $\widetilde{\Lambda}$ that satisfies
\begin{itemize}
    \item $y^\top \E_{X\sim\nu}\Mp{P_{\widetilde{\Lambda}}(X)XX^\top}^{-1}y\leq (1+\epsilon)c_\ell^2,\quad\forall y\in\mc{Y}_\ell$.
    
    \item $\E_{X\sim\nu}\Mp{P_{\widetilde{\Lambda}}(X)}\leq \E_{X\sim\nu}\Mp{\widetilde{P}(X)}+4\sqrt{\mu}$, where $\widetilde{P}$ is the optimal solution to problem \eqref{equ:opt_mu_bound}.
\end{itemize}
\end{theorem}
\begin{proof}
\textbf{First Bullet Point.} Fix some $\epsilon>0$. Let $\hat{\bm{\Lambda}}$ and corresponding $\hat{\Lambda}=\sum_{y\in\mc{Y}_\ell}\hat{\Lambda}_y$ be the parameters obtained by Algorithm \ref{algo:solve_sgd} just before the re-scaling step, which means that at line \ref{line:assign_Lambda_y} of Algorithm \ref{algo:solve_sgd}, the assignment of $\hat{\Lambda}_y$ to each $y\in\mc{Y}_\ell$ has been optimized by solving problem \eqref{equ:opt_semi}. That is, we have $D(\hat{\bm{\Lambda}})=\overline{D}(\hat{\Lambda})$ and $D_E(\hat{\bm{\Lambda}})=\overline{D}_E(\hat{\Lambda})$. Let $\widetilde{\bm{\Lambda}}$ and $\widetilde{\Lambda}$ be the ones after the re-scaling step. Then, by Theorem 3.13 of \cite{orabona2019modern}, with probability at least $1-\frac{\delta}{3}$, it holds that
$$\overline{D}(\Lambda^*)-\overline{D}(\hat{\Lambda})= D(\bm{\Lambda}^*)-D(\hat{\bm{\Lambda}})\leq\frac{\mathrm{Reg}(K)+2\sqrt{2K\log(6/\delta)}}{K},$$
where $\mathrm{Reg}(K)$ is the regret of running projected stochastic gradient ascent for $K$ steps with $\eta_k$ specified in Algorithm \ref{algo:solve_sgd}. Meanwhile, by Theorem 4.14 of \cite{orabona2019modern} also, we have $\mathrm{Reg}(K)=\sqrt{2}B^2\sqrt{\sum_{k=1}^{K}\sum_{y\in\mc{Y}_\ell}\Norm{g_{k, y}}_2^2}$, where $B=\sqrt{\abs{\mc{Y}_\ell}}\Norm{\Lambda^*}_F$ bound the norm of $\bm{\Lambda}^*=\Sp{\Lambda^*_y}_{y\in\mc{Y}_\ell}$. Since $g_{k, y}=\frac{yy^\top}{c_\ell^2}-P_{\hat{\Lambda}^{(k)}}(x_k)x_kx_k^\top$, we can easily get $\sum_{y\in\mc{Y}_\ell}\Norm{g_{k, y}}_2^2\leq 2\abs{\mc{Y}_\ell}M^4+\frac{2}{c_\ell^2}\sum_{y\in\mc{Y}_\ell}\Norm{y}_2^4=2\abs{\mc{Y}_\ell}M^4+2\abs{\mc{Y}_\ell}C_\ell^2$. Thus, we have

\begin{equation}
    \label{equ:regret}
    \mathrm{Reg}(K)\leq 2\abs{\mc{Y}_\ell}\Norm{\Lambda^*}_F^2\sqrt{\abs{\mc{Y}_\ell}M^4+\abs{\mc{Y}_\ell}C_\ell^2}\cdot\sqrt{K}:=C_{\mathrm{Reg}}\sqrt{K}
\end{equation}
\begin{equation}
    \label{equ:D_diff}
    \implies \overline{D}(\Lambda^*)-\overline{D}(\hat{\Lambda})\leq\frac{C_\mathrm{Reg}+2\sqrt{2\log(6/\delta)}}{\sqrt{K}},
\end{equation}

We now consider the effect of using $u$ i.i.d. samples in the re-scaling step. First, since re-scaling always increases the function value, we must have $D_E(\hat{\bm{\Lambda}})\leq D_E(\widetilde{\bm{\Lambda}})$. Meanwhile, since $D_E(\hat{\bm{\Lambda}})=\overline{D}_E(\hat{\Lambda})$, by Lemma \ref{lmm:rescaling}, we have $D_E(\hat{\bm{\Lambda}})=\overline{D}_E(\hat{\Lambda})$, which together implies $\overline{D}_E(\hat{\Lambda})\leq \overline{D}_E(\widetilde{\Lambda})$.

By Lemma \ref{lmm:D_strong_concave}, we know that $\Lambda^*$ is unique and as long as $\mu\leq\frac{1}{2\sqrt{3}}$, $\overline{D}(\Lambda)$ is $G$-strongly concave with respect to $\ell_2$ norm over $\mc{S}=\Bp{\Lambda\succeq\bm{0}:\Norm{\Lambda}_F\leq 2\Norm{\Lambda^*}_F}$, where $G$ is defined in Eq. \eqref{equ:G_coefficient}. Thus, by Lemma \ref{lmm:region_concave}, if $K$ is large enough such that
$$\overline{D}(\Lambda^*)-\overline{D}(\hat{\Lambda})\leq\frac{C_\mathrm{Reg}+2\sqrt{2\log(6/\delta)}}{\sqrt{K}}\leq\frac{G\Norm{\Lambda^*}_F}{2},$$
then $\Norm{\hat{\Lambda}-\Lambda^*}_F\leq\Norm{\Lambda^*}_F$, which implies $\Norm{\hat{\Lambda}}_F\leq 2\Norm{\Lambda^*}_F$. That is, $\hat{\Lambda}\in\mc{S}$. Then, under this condition, by using Lemma \ref{lmm:opt_concentration}, when $\mu\leq\frac{4}{9}\Norm{\Lambda^*}_FM^4$ and
\begin{equation}
    \label{equ:u_cond_1}
    u\geq\Sp{\frac{6\kappa(\Sigma)\Norm{\Lambda^*}_FM^4\Sp{2+\sqrt{2\log(6/\delta)}}}{G\mu^2}\cdot\frac{1+\epsilon}{\epsilon}}^2,
\end{equation}
for $\widetilde{\Lambda}$ after re-scaling, with probability at least $1-\frac{\delta}{3}$, it holds simultaneously that
\begin{equation}
    \label{equ:D_concen}
    \abs{\overline{D}_E(\hat{\Lambda})-\overline{D}(\hat{\Lambda})}\leq \frac{G\mu^2}{3M^2\kappa(\Sigma)}\cdot\frac{\epsilon}{1+\epsilon}\quad\text{and}\quad \abs{\overline{D}_E(\widetilde{\Lambda})-\overline{D}(\widetilde{\Lambda})}\leq \frac{G\mu^2}{3M^2\kappa(\Sigma)}\cdot\frac{\epsilon}{1+\epsilon}
\end{equation}
\begin{align*}
    \implies\overline{D}(\Lambda^*)-\overline{D}(\widetilde{\Lambda})&\leq \overline{D}(\Lambda^*)-\overline{D}(\hat{\Lambda})+\overline{D}(\hat{\Lambda})-\overline{D}(\widetilde{\Lambda})\\
    &\leq \overline{D}(\Lambda^*)-\overline{D}(\hat{\Lambda})+\overline{D}(\hat{\Lambda})-\overline{D}_E(\hat{\Lambda})+\overline{D}_E(\widetilde{\Lambda})-\overline{D}(\widetilde{\Lambda})\tag{Since $\overline{D}_E(\hat{\Lambda})\leq \overline{D}_E(\widetilde{\Lambda})$}\\
    &\leq \frac{C_\mathrm{Reg}+2\sqrt{2\log(6/\delta)}}{\sqrt{K}}+\frac{2G\mu^2}{3M^2\kappa(\Sigma)}\cdot\frac{\epsilon}{1+\epsilon}.\tag{By Eq. \eqref{equ:D_diff} and \eqref{equ:D_concen}}
\end{align*}

Since $\widetilde{\Lambda}$ is a smaller re-scaling of $\hat{\Lambda}$, we have $\widetilde{\Lambda}\in\mc{S}$, which implies $\frac{G}{2}\Norm{\Lambda^*-\widetilde{\Lambda}}_F\leq \overline{D}(\Lambda^*)-\overline{D}(\widetilde{\Lambda})$ by property of strongly concave function \cite{bertsekas2009convex}. Therefore, by Lemma \ref{lmm:opt_guarantees}, to guarantee an at most $\epsilon$ multiplicative constraint violation, it is sufficient to choose $K$ such that
\begin{align*}
    \frac{G}{2}\Norm{\Lambda^*-\widetilde{\Lambda}}_F&\leq\overline{D}(\Lambda^*)-\overline{D}(\widetilde{\Lambda})\\
    &\leq\frac{C_\mathrm{Reg}+2\sqrt{2\log(6/\delta)}}{\sqrt{K}}+\frac{2G\mu^2}{3M^2\kappa(\Sigma)}\cdot\frac{\epsilon}{1+\epsilon}\\
    &\leq \min\Bp{\frac{4G\mu^2}{3M^2\kappa(\Sigma)}\cdot\frac{\epsilon}{1+\epsilon}, \frac{G\Norm{\Lambda^*}_F}{2}}\\
    &= \frac{4G\mu^2}{3M^2\kappa(\Sigma)}\cdot\frac{\epsilon}{1+\epsilon}.\tag{If $\mu\leq \sqrt{\frac{3\kappa(\Sigma)\Norm{\Lambda^*}_FM^2}{8}\cdot\frac{1+\epsilon}{\epsilon}}$}
\end{align*}
An algebraic rearrangement gives us
\begin{equation}
    \label{equ:K_bound}
    K\geq\Sp{\frac{3\kappa(\Sigma)M^2\Sp{C_\mathrm{Reg}+2\sqrt{2\log(6/\delta)}}}{2G\mu^2}\cdot\frac{1+\epsilon}{\epsilon}}^2.
\end{equation}

\textbf{Second Bullet Point.} We then prove the upper bound for primal objective value $\E_{X\sim\nu}\Mp{P_{\widetilde{\Lambda}}(X)}$, which explains the reason why an extra re-scaling step is needed. Define $g(s)=D_E(s\cdot\widetilde{\bm{\Lambda}})$. By construction, we know that $g(s)$ is maximized at $s=1$ because $\widetilde{\bm{\Lambda}}=s^*\cdot\hat{\bm{\Lambda}}$, where $s^*=\argmax_{s\in[0, 1]}D_E(s\cdot\hat{\bm{\Lambda}})$. Therefore, we have $g'(1)\geq 0$, which in turn gives us
$$g'(1)=\frac{1}{c_\ell^2}\sum_{y\in\mc{Y}_\ell}y^\top\widetilde{\Lambda}_yy-\frac{1}{u}\sum_{i=1}^{u}P_{\widetilde{\Lambda}}(x_i)x_i^\top\widetilde{\Lambda} x_i\geq 0.$$

By the concentration inequality in Lemma \ref{lmm:opt_concentration}, we know that when
\begin{equation}
    \label{equ:u_cond_2}
    u\geq \Sp{\frac{2\Norm{\Lambda^*}_FM^2\Sp{\Norm{\Lambda^*}_FM^2+\mu\sqrt{2\log(6/\delta)}}}{\mu^{3/2}}}^2,
\end{equation}
with probability at least $1-\frac{\delta}{3}$, it holds that
\begin{align}
    \abs{\E_{X\sim\nu}\Mp{P_{\Lambda}(X)X^\top\Lambda X}-\frac{1}{u}\sum_{i=1}^{u}P_{\Lambda}(x_i)x_i^\top\Lambda x_i}&\leq\sqrt{\mu}\nonumber\\
    \implies\frac{1}{c_\ell^2}\sum_{y\in\mc{Y}_\ell}y^\top\widetilde{\Lambda}_yy-\E_{X\sim\nu}\Mp{P_{\widetilde{\Lambda}}(X)X^\top\widetilde{\Lambda} X}+\sqrt{\mu}&\geq 0.\label{equ:g_prime}
\end{align}

Now, let $\widetilde{P}$ be the optimal solution of problem \eqref{equ:opt_mu_bound} and $\hat{P}$ be the optimal solution of the same problem with bound constraint $\mu\leq P(x)\leq 1-\mu$. 
\begin{equation}
    \label{equ:opt_mu_bound}
    \begin{array}{rl}
        \min_{P} & \E_{X\sim\nu}\Mp{P(X)} \\
        \text{subject to} & y^\top \E_{X\sim\nu}\Mp{P(X)XX^\top}^{-1}y\leq c_\ell^2,\quad\forall y\in\mc{Y}_\ell,\\
        & 0\leq P(x)\leq 1-\mu,\quad\forall x\in\mc{X}.
    \end{array}
\end{equation}

Then, we can notice that
\begin{align*}
    &\E_{X\sim\nu}\Mp{P_{\widetilde{\Lambda}}(X)}\\
    \leq & \E_{X\sim\nu}\Mp{P_{\widetilde{\Lambda}}(X)-\mu(\log(1-P_{\widetilde{\Lambda}}(X))+\log(P_{\widetilde{\Lambda}}(X)))}\\
    \leq& \E_{X\sim\nu}\Mp{P_{\widetilde{\Lambda}}(X)-\mu(\log(1-P_{\widetilde{\Lambda}}(X))+\log(P_{\widetilde{\Lambda}}(X)))}\\
    &\qquad\qquad+\frac{1}{c_\ell^2}\sum_{y\in\mc{Y}_\ell}y^\top\widetilde{\Lambda}_yy-\E_{X\sim\nu}\Mp{P_{\widetilde{\Lambda}}(X)X^\top\widetilde{\Lambda} X}+\sqrt{\mu}\tag{By Eq. \eqref{equ:g_prime}}\\
    =&\inf_{P}\mathcal{L}(P, \widetilde{\bm{\Lambda}})+\sqrt{\mu}\tag{By definition of Lagrangian function and how we solve for $P_\Lambda$}\\
    \leq& \max_{\Lambda_y\succeq\bm{0}, \forall y\in\mc{Y}_\ell}\inf_P\mc{L}\Sp{P, \bm{\Lambda}}+\sqrt{\mu}\\
    =&\E_{X\sim\nu}\Mp{P_{\Lambda^*}(X)-\mu(\log(1-P_{\Lambda^*}(X))+\log(P_{\Lambda^*}(X)))}+\sqrt{\mu}\\
    \leq& \E_{X\sim\nu}\Mp{\hat{P}(X)-\mu\log(1-\hat{P}(X))}-\mu\log\Sp{\hat{P}(X)}+\sqrt{\mu}\tag{Since $\hat{P}$ is feasible to problem \eqref{equ:opt_barrier_2}}\\
    \leq &\E_{X\sim\nu}\Mp{\hat{P}(X)}+3\sqrt{\mu},\tag{Since $-a\log(a)\leq\sqrt{a}$ for $a\in(0, 1)$}\\
    \leq &\E_{X\sim\nu}\Mp{\widetilde{P}(X)}+4\sqrt{\mu}.\tag{Since $\hat{P}(x)$ can have at most $\mu$ more contribution than $\widetilde{P}$}
\end{align*}

Therefore, in summary, Suppose $K$ and $u$ satisfy conditions specified in Eq. \eqref{equ:K_bound}, \eqref{equ:u_cond_1} and \eqref{equ:u_cond_2} and $\mu\leq \min\Bp{ \sqrt{\frac{3\kappa(\Sigma)\Norm{\Lambda^*}_FM^2}{8}\cdot\frac{1+\epsilon}{\epsilon}}, \frac{4}{9}\Norm{\Lambda^*}_F^2M^4, \frac{1}{2\sqrt{3}}}$, where $C_\mathrm{Reg}$ and $G$ are defined in Eq. \eqref{equ:regret} and \eqref{equ:G_coefficient}, respectively. Then. by applying a simple union bound, with probability at least $1-\delta$, the output of Algorithm \ref{algo:solve_sgd} $\widetilde{\Lambda}$ satisfies $y^\top \E_{X\sim\nu}\Mp{P(X)XX^\top}^{-1}y\leq (1+\epsilon)c_\ell^2,\forall y\in\mc{Y}_\ell$ and $\E_{X\sim\nu}\Mp{P_{\widetilde{\Lambda}}(X)}\leq \E_{X\sim\nu}\Mp{\widetilde{P}(X)}+4\sqrt{\mu}$.
\end{proof}

\subsection{Relevant Lemmas}
\subsubsection{Strong Concavity of $\overline{D}(\Lambda)$}
\begin{lemma}
\label{lmm:D_strong_concave}
As long as $\mu\leq\frac{1}{2\sqrt{3}}$, $\overline{D}(\Lambda)$ is $G$-strongly concave with respect to $\ell_2$-norm on the bounded region $\mathcal{S}=\Bp{\Lambda\succeq\bm{0}:\Norm{\Lambda}_F\leq 2\Norm{\Lambda^*}_F}$ with coefficient
\begin{equation}
    \label{equ:G_coefficient}
    G=\frac{\mu}{2\Sp{2\Norm{\Lambda^*}_FM^2+1}^2}\cdot\min_{\Gamma\in\mathbb{S}^d:\Norm{\Gamma}_F=1}\E_{X\sim\nu}\Mp{\Sp{X^\top\Gamma X}^2}.
\end{equation}
Because of this, as a corollary, $\Lambda^*$ will be unique.
\end{lemma}
\begin{proof}
By Lemma \ref{lmm:f_lambda}, since $f(\Lambda)$ is concave in $\Lambda$, it is sufficient to prove that $\E_{X\sim\nu}\Mp{h_\Lambda(X)}$ is $G$-strongly concave on $\mc{S}$, where $h_\Lambda(x)$ is defined in Eq. \eqref{equ:h_func}. Then, we have
$$-\nabla^2_{\Lambda}\E_{X\sim\nu}\Mp{h_\Lambda(X)}=\E_{X\sim\nu}\Mp{\frac{\mathrm{d}P_\Lambda}{\mathrm{d}q_\Lambda}(X)\mathrm{vec}\Sp{XX^\top}\mathrm{vec}\Sp{XX^\top}^\top}.$$

Since $\Norm{x}_2\leq M$, for any $\Lambda\in\mc{S}$, we have $q_{\Lambda}(x)=x^\top\Lambda x-1\leq 2\Norm{\Lambda^*}_FM^2+1$. By Lemma, \ref{lmm:P_property}, we know that if $12\mu^2\leq \Sp{2\Norm{\Lambda^*}_FM^2+1}^2$, which can be done by choosing $\mu\leq\frac{1}{2\sqrt{3}}$, we have $\frac{\mathrm{d}P_\Lambda}{\mathrm{d}q_\Lambda}(x)\geq \frac{\mu}{2\Sp{2\Norm{\Lambda^*}_FM^2+1}^2}$ for any $x\in\mc{X}$ and $\Lambda\in\mc{S}$. Therefore, we have
$$-\nabla^2_{\Lambda}\E_{X\sim\nu}\Mp{h_\Lambda(X)}\succeq \gamma\cdot\E_{X\sim\nu}\Mp{\mathrm{vec}\Sp{XX^\top}\mathrm{vec}\Sp{XX^\top}^\top}$$

Now, let $\mathbb{S}$ be the set of all $d\times d$ symmetric matrices. It is obvious that $\mathbb{S}$ is a subspace of the vector space of all $d\times d$ matrices and $\mc{S}\subseteq\mathbb{S}$. Thus, by applying Lemma \ref{lmm:strong_convex}, we can conclude that $\E_{X\sim\nu}\Mp{h_\Lambda(X)}$ is $G$-strongly concave on $\mc{S}$ with respect to $\ell_2$ norm and
\begin{align*}
    G&=\frac{\mu}{2\Sp{2\Norm{\Lambda^*}_FM^2+1}^2}\cdot\min_{\Gamma\in\mathbb{S}^d:\Norm{\Gamma}_F=1}\mathrm{vec}(\Gamma)^\top\E_{X\sim\nu}\Mp{\mathrm{vec}\Sp{XX^\top}\mathrm{vec}\Sp{XX^\top}^\top}\mathrm{vec}(\Gamma)\\
    &=\frac{\mu}{2\Sp{2\Norm{\Lambda^*}_FM^2+1}^2}\cdot\min_{\Gamma\in\mathbb{S}^d:\Norm{\Gamma}_F=1}\E_{X\sim\nu}\Mp{\Sp{X^\top\Gamma X}^2}.
\end{align*}
Thus the proof is complete.
\end{proof}

\begin{lemma}
\label{lmm:f_lambda}
$f(\Lambda)$ defined in Eq. \eqref{equ:opt_semi} is concave in $\Lambda$.
\end{lemma}
\begin{proof}
To show its concavity, consider $\Lambda^{(1)}\succeq\bm{0}$, $\Lambda^{(2)}\succeq\bm{0}$ and some $\gamma\in (0, 1)$. Let $(\Lambda_y^{(i)})_{y\in\mc{Y}_\ell}$ be the optimal solution obtained by evaluating $f(\Lambda^{(i)})$ for $i\in\Bp{1, 2}$. Then, we can notice that
\begin{align*}
    \gamma f(\Lambda^{(1)})+(1-\gamma)f(\Lambda^{(2)})&=\gamma\sum_{y\in\mc{Y}_\ell}y^\top\Lambda_y^{(1)}y+(1-\gamma)\sum_{y\in\mc{Y}_\ell}y^\top\Lambda_y^{(2)}y\\
    &=\sum_{y\in\mc{Y}_\ell}y^\top(\gamma\Lambda^{(1)}_y+(1-\gamma)\Lambda^{(2)}_y)y\\
    &\leq f(\gamma\Lambda^{(1)}+(1-\gamma)\Lambda^{(2)}).
\end{align*}
The last inequality above holds because $\sum_{y\in\mc{Y}_\ell}\Lambda_y^{(i)}=\Lambda^{(i)}$ for $i\in\Bp{1, 2}$ and thus $\sum_{y\in\mc{Y}_\ell}\Sp{\gamma\Lambda^{(1)}_y+(1-\gamma)\Lambda^{(2)}_y}=\gamma\Lambda^{(1)}+(1-\gamma)\Lambda^{(2)}$, which means that $(\gamma\Lambda^{(1)}_y+(1-\gamma)\Lambda^{(2)}_y)_{y\in\mc{Y}_\ell}$ is a feasible solution for problem \eqref{equ:opt_semi} with parameter $\gamma\Lambda^{(1)}+(1-\gamma)\Lambda^{(2)}$. Therefore, we can conclude that $f(\Lambda)$ is concave in $\Lambda$.
\end{proof}

\begin{lemma}
\label{lmm:strong_convex}
Let $f:\R^d\mapsto\R$ be a convex and twice differentiable function in $\R^d$. If for some subspace $S\subseteq\R^d$, we have $\min_{w\in S:\Norm{w}_2=1}w^\top\nabla^2f(x)w\geq\sigma>0$, $\forall x\in S$, then $f$ is $\sigma$-strongly convex with respect to $\ell_2$-norm on $S$.
\end{lemma}
\begin{proof}
Suppose $S$ has dimension $m$ and let $v_1, \dots, v_m$ be a set of orthonormal basis that span $S$. Then, for each $x\in S$, there exists unique $z\in\R^m$ such that $x=Vz$, where $V=\matenv{v_1 & \dots & v_m}$. That is, there is one-to-one correspondence between $S$ and $\R^m$.

Now, we define $g:\R^m\mapsto\R$ as $g(z)=f(Vz)$. It is easy to compute $\nabla^2 g(z)=V^\top\nabla^2f(Vz)V$. Then, notice that for any $w'\in\R^m$ such that $\Norm{w'}_2=1$, we have $Vw'\in S$ and $\Norm{Vw'}_2=\sqrt{w'^\top V^\top Vw'}=\sqrt{w'^\top w'}=1$. Thus, we have
\begin{align*}
    \min_{w'\in\R^m:\Norm{w'}_2=1}w'^\top\nabla^2g(z)w'&=\min_{w'\in\R^m:\Norm{w'}_2=1}w'^\top V^\top\nabla^2 f(Vz)Vw'\\
    &=\min_{w\in S:\Norm{w}_2=1}w^\top\nabla^2 f(Vz)w\geq\sigma.
\end{align*}
Therefore, $g$ is $\sigma$-strongly convex with respect to $\ell_2$ norm. Then, for any $x_1, x_2\in S$, there exists unique $z_1, z_2\in\R^m$ such that $x_1=Vz_1$ and $x_2=Vz_2$. Notice that $\Norm{z_1-z_2}_2=\Norm{x_1-x_2}_2$ since $V$ preserves the norm. Further, by definition of strong convexity, for any $\alpha\in[0, 1]$, we have
\begin{align*}
    g(\alpha z_1+(1-\alpha)z_2)+\frac{\sigma}{2}\alpha(1-\alpha)\Norm{z_1-z_2}_2^2&\leq\alpha g(z_1)+(1-\alpha)g(z_2)\\
    \implies f(\alpha Vz_1+(1-\alpha)Vz_2)+\frac{\sigma}{2}\alpha(1-\alpha)\Norm{x_1-x_2}_2^2&\leq\alpha f(Vz_1)+(1-\alpha)f(Vz_2)\\
    \implies f(\alpha x_1+(1-\alpha)x_2)+\frac{\sigma}{2}\alpha(1-\alpha)\Norm{x_1-x_2}_2^2&\leq\alpha f(x_1)+(1-\alpha)f(x_2).
\end{align*}
Thus, $f$ is also $\sigma$-strongly convex with respect to $\ell_2$ norm on $S$.
\end{proof}

\subsubsection{Concentration Inequalities}
\begin{lemma}
\label{lmm:opt_concentration}
Let $x_1, \dots, x_u\sim\nu$ be i.i.d. samples. If $\Norm{\hat{\Lambda}}_F\leq 2\Norm{\Lambda^*}_F$, $\Norm{x}_2\leq M$ for any $x\in\mc{X}$ and $\mu\leq\frac{4}{9}\Norm{\Lambda^*}_F^2M^4$, then with probability at least $1-\frac{2\delta}{3}$, it holds for any $\Lambda\in\Theta=\Bp{s\cdot\hat{\Lambda}: s\in[0, 1]}$ simultaneously that
\begin{align*}
    \abs{\E_{X\sim\nu}\Mp{h_{\Lambda}(X)}-\frac{1}{u}\sum_{i=1}^{u}h_{\Lambda}(x_i)}&\leq\frac{2\Norm{\Lambda^*}_FM^2\Sp{2+\sqrt{2\log(6/\delta)}}}{\sqrt{u}}\\
    \abs{\E_{X\sim\nu}\Mp{P_{\Lambda}(X)X^\top\Lambda X}-\frac{1}{u}\sum_{i=1}^{u}P_{\Lambda}(x_i)x_i^\top\Lambda x_i}&\leq\frac{2\Norm{\Lambda^*}_FM^2\Sp{\Norm{\Lambda^*}_FM^2+\mu\sqrt{2\log(6/\delta)}}}{\mu\sqrt{u}}.
\end{align*}
\end{lemma}
\begin{proof}
    To prove the first inequality, first, notice that we have $h_{\Lambda}(x)=-P_{\Lambda}(x)q_{\Lambda}(x)-\mu\Sp{\log(1-P_{\Lambda}(x))+\log(P_{\Lambda}(x))}$, where $q_{\Lambda}(x)=x^\top\Lambda x-1$. Since $P_{\Lambda}(x)$, defined in Eq. \eqref{equ:P_sol}, explicitly only depends on $q_{\Lambda}(x)$ instead of $x$ directly, we can treat $h_\Lambda$ as a function of $q_\Lambda$ and define a function class $\mathcal{F}=\Bp{x\mapsto x^\top(s\cdot\hat{\Lambda})x: s\in[0, 1]}$. It is well-known that if $h_{\Lambda}$ is $L_1$-Lipschitz in $q_{\Lambda}$ and $\abs{h_{\Lambda}(x)}\leq R_1$ for any $\Lambda\in\Theta$ and $x\sim\nu$, then, with probability at least $1-\frac{\delta}{3}$, it holds simultaneously for all $\Lambda\in\Theta$ that \cite{bartlett2002rademacher, mohri2018foundations}
    \begin{equation}
        \label{equ:general_concen}
        \abs{\E_{X\sim\nu}\Mp{h_{\Lambda}(X)}-\frac{1}{u}\sum_{i=1}^{u}h_{\Lambda}(x_i)}\leq 2L_1\cdot\mathcal{R}_u(\mc{F})+R_1\sqrt{\frac{2\log(6/\delta)}{u}},
    \end{equation}
    where $\mathcal{R}_u(\mc{F})$ is the Rademacher complexity of $\mc{F}$.
    
    To find $L_1$, we can compute
    \begin{align*}
        \frac{\mathrm{d}h_\Lambda}{\mathrm{d}q_\Lambda}&=-\frac{\mathrm{d}P_\Lambda}{\mathrm{d}q_\Lambda}q_\Lambda-P_\Lambda+\frac{\mathrm{d}P_\Lambda}{\mathrm{d}q_\Lambda}\Sp{\frac{\mu}{1-P_\Lambda}-\frac{\mu}{P_\Lambda}}\\
        &=-\frac{\mathrm{d}P_\Lambda}{\mathrm{d}\cdot q_\Lambda}q_\Lambda-P_\Lambda+\frac{\mathrm{d}P_\Lambda}{\mathrm{d}q_\Lambda}\cdot q_\Lambda\tag{Since $P_\Lambda$ satisfies Eq. \eqref{equ:P_equ}}\\
        &=-P_\Lambda
    \end{align*}
    Therefore, we have $\frac{\mathrm{d}h_\Lambda}{\mathrm{d}q_\Lambda}\in\Mp{-1, -\frac{\mu}{3}}$ by Lemma \ref{lmm:P_property}. Therefore, we can set $L_1=1$.
    
    Let $h_0$ be the value of $h_\Lambda$ when $q_\Lambda=-1$, which means $x^\top\Lambda x=0$. To find $R_1$, notice that since $\frac{\mathrm{d}h_\Lambda}{\mathrm{d}q_\Lambda}\in\Mp{-1, -\frac{\mu}{3}}$, we must have $-q_\Lambda+h_0\leq h_\Lambda\leq-\frac{\mu}{3}q_\Lambda+h_0$. By Lemma \ref{lmm:P_property}, we know that $h_0\in\Mp{0, 2\sqrt{\mu}}$. Therefore, we have $-x^\top\Lambda x\leq h_\Lambda(x)\leq -\frac{\mu}{3}x^\top\Lambda x+3\sqrt{\mu}$ for any $x\in\mc{X}$ and $\Lambda\in\Theta$. Since $\Norm{\Lambda}_F\leq\Norm{\hat{\Lambda}}_F\leq 2\Norm{\Lambda^*}_F$, we have $\abs{h_\Lambda(x)}\leq 2\Norm{\Lambda^*}_FM^2:=R_1$, which holds when $\mu\leq\frac{4}{9}\Norm{\Lambda^*}_F^2M^4$. Then, by Lemma \ref{lmm:concen_quantity}, we know that $\mathcal{R}_u(\mc{F})\leq\frac{2\Norm{\Lambda^*}_FM^2}{\sqrt{u}}$. Thus, plugging in values of $L_1$, $R_1$ and $\mc{R}_u(\mc{F})$ into Eq. \eqref{equ:general_concen} gives our first concentration inequality.
    
    We can basically follow exactly the same strategy to prove the second concentration inequality. In particular, define $\tilde{h}_\Lambda(x)=P_\Lambda(x)x^\top\Lambda x=P_\Lambda(x)q_\Lambda(x)+P_\Lambda(x)$. Then, with probability at least $1-\frac{\delta}{3}$, it holds simultaneously for any $\Lambda\in\Theta$ that
    \begin{equation}
        \label{equ:general_concen_2}
        \abs{\E_{X\sim\nu}\Mp{\tilde{h}_{\Lambda}(X)}-\frac{1}{u}\sum_{i=1}^{u}\tilde{h}_{\Lambda}(x_i)}\leq 2L_2\cdot\mathcal{R}_u(\mc{F})+R_2\sqrt{\frac{2\log(6/\delta)}{u}},
    \end{equation}
    where $\abs{\tilde{h}_\Lambda(x)}\leq R_2$ for any $x\in\mc{X}$, $\Lambda\in\Theta$ and $\tilde{h}_\Lambda$ is $L_2$-Lipschitz in $q_\Lambda$.
    
    To find $L_2$, we can compute
    $$\frac{\mathrm{d}\tilde{h}_\Lambda}{\mathrm{d}q_\Lambda}=P_\Lambda+\frac{\mathrm{d}P_\Lambda}{\mathrm{d}q_\Lambda}\cdot x^\top\Lambda x.$$
    By Lemma \ref{lmm:P_property}, we know that $\frac{\mathrm{d}P_\Lambda}{\mathrm{d}q_\Lambda}\in\Mp{0, \frac{1}{8\mu}}$. Thus, we have $\abs{\frac{\mathrm{d}\tilde{h}_\Lambda}{\mathrm{d}q_\Lambda}}\leq 1+\frac{\Norm{\Lambda^*}_FM^2}{4\mu}:=L_2$. It is obvious that $\tilde{h}_\Lambda(x)\leq 2\Norm{\Lambda^*}_FM^2:=R_2$. Thus, by plugging the values of $L_2$, $R_2$ and $\mc{R}_u(\mc{F})$ into Eq. \eqref{equ:general_concen_2}, we can obtain the second concentration inequality.
    
    Finally, both concentration inequalities hold simultaneously with probability at least $1-\frac{2\delta}{3}$ by a simple union bound.
\end{proof}

\begin{lemma}
\label{lmm:concen_quantity}
If $\Norm{\hat{\Lambda}}_F\leq 2\Norm{\Lambda^*}_F$, then, we have $\mathcal{R}_u(\mc{F})\leq\sqrt{\frac{\E_{X\sim\nu}\Mp{(X^\top\hat{\Lambda}X)^2}}{u}}\leq\frac{2\Norm{\Lambda^*}_FM^2}{\sqrt{u}}$, where $\mathcal{F}=\Bp{x\mapsto x^\top(s\cdot\hat{\Lambda})x: s\in[0, 1]}$.
\end{lemma}
\begin{proof}
Let $\sigma_1, \dots, \sigma_u$ be i.i.d. Rademacher random variables, which are uniform over $\Bp{-1, +1}$. Let $x_1, \dots, x_u\sim\nu$ be i.i.d. samples. Then, by definition of Rademacher complexity, we have
\begin{align*}
    \mathcal{R}_u(\mc{F})&=\E\Mp{\sup_{q\in\mc{F}}\frac{1}{u}\sum_{i=1}^{u}\sigma_iq(x_i)}\\
    &=\E\Mp{\sup_{s\in[0, 1]}\frac{1}{u}\sum_{i=1}^{u}\sigma_ix_i^\top(s\hat{\Lambda})x_i}\tag{By definition of $\mc{F}$}\\
    &\overset{\text{(i)}}{=}\frac{1}{u}\E\Mp{\mathds{1}\Bp{\sum_{i=1}^{n}\sigma_ix_i^\top\hat{\Lambda}x_i\geq 0}\sum_{i=1}^{n}\sigma_ix_i^\top\hat{\Lambda}x_i}.\\
    &\leq \frac{1}{u}\E\Mp{\abs{\sum_{i=1}^{u}\sigma_ix_i^\top\hat{\Lambda}x_i}}\\
    &\leq\frac{1}{u}\sqrt{\E\Mp{\Sp{\sum_{i=1}^{u}\sigma_ix_i^\top\hat{\Lambda}x_i}^2}}\tag{By Jensen's inequality}\\
    &=\frac{1}{u}\sqrt{\E\Mp{\sum_{i=1}^{u}\Sp{x_i^\top\hat{\Lambda}x_i}^2}}\tag{Since $\sigma_i$'s are i.i.d. and $\E\Mp{\sigma_i}=0$}\\
    &=\sqrt{\frac{\E_{X\sim\nu}\Mp{\Sp{X^\top\hat{\Lambda}X}^2}}{u}}\leq \frac{2\Norm{\Lambda^*}_FM^2}{\sqrt{u}}.
\end{align*}
Here, the equality (i) holds because when $\sum_{i=1}^{n}\sigma_ix_i^\top\hat{\Lambda}x_i<0$, the supremum over $s\in[0, 1]$ will be obtained by taking $s=0$; otherwise, it will be obtained by taking $s=1$.
\end{proof}

\subsubsection{Other Lemmas}
The following lemma basically shows that $f(\Lambda)$ is linear in scalar multiplication.
\begin{lemma}
	\label{lmm:rescaling}
	If $D_E(\hat{\bm{\Lambda}})=\overline{D}_E(\hat{\Lambda})$, with $\hat{\Lambda}=\sum_{y\in\mc{Y}_\ell}\hat{\Lambda}_y$, then, for any $s\geq 0$, it holds that $D_E(s\cdot\hat{\bm{\Lambda}})=\overline{D}_E(s\cdot\hat{\Lambda})$, where $D_E$ and $\overline{D}_E$ are defined in Eq. \eqref{equ:D_empirical}.
\end{lemma}
\begin{proof}
	It suffices to show that if $\sum_{y\in\mc{Y}_\ell}y^\top\hat{\Lambda}_y y=f(\hat{\Lambda})$, then $\sum_{y\in\mc{Y}_\ell}y^\top(s\cdot\hat{\Lambda}_y) y=f(s\cdot\hat{\Lambda})$ for any $s>0$. By definition, we have
	$$\begin{array}{rl}
		f(s\cdot\hat{\Lambda})=\max_{\Lambda_y} & \sum_{y\in\mc{Y}_\ell}y^\top\Lambda_{y}y \\
		\text{subject to} & \sum_{y\in\mc{Y}_\ell}\Lambda_{y}=s\cdot\hat{\Lambda}\\
		& \Lambda_{y}\succeq\mathbf{0}, \quad\forall y\in\mc{Y}_\ell.
	\end{array}$$
	For the above optimization problem, we can do a change of variable by setting $\Lambda_{y}'=\frac{1}{s}\cdot\Lambda_{y}\implies\Lambda_{y}=s\cdot\Lambda_{y}'$. Then, we have
	$$\begin{array}{rl}
		f(s\cdot\hat{\Lambda})=\max_{\Lambda_y} & \sum_{y\in\mc{Y}_\ell}y^\top(s\cdot\Lambda_{y}')y \\
		\text{subject to} & \sum_{y\in\mc{Y}_\ell}s\cdot\Lambda'_{y}=s\cdot\hat{\Lambda}\\
		& s\cdot\Lambda_{y}'\succeq\mathbf{0}, \quad\forall y\in\mc{Y}_\ell.
	\end{array}$$
	$$\begin{array}{rl}
		\implies f(s\cdot\hat{\Lambda})=\max_{\Lambda_y} & s\sum_{y\in\mc{Y}_\ell}y^\top\Lambda_{y}'y \\
		\text{subject to} & \sum_{y\in\mc{Y}_\ell}\Lambda'_{y}=\hat{\Lambda}\\
		& \Lambda_{y}'\succeq\mathbf{0}, \quad\forall y\in\mc{Y}_\ell.
	\end{array}$$
	$$\implies f(s\cdot\hat{\Lambda})=s\cdot f(\hat{\Lambda})=s\cdot\sum_{y\in\mc{Y}_\ell}y^\top\Lambda_y y=\sum_{y\in\mc{Y}_\ell}y^\top(s\cdot\hat{\Lambda}_y) y.$$
	Thus, the proof is complete.
\end{proof}

\begin{lemma}
	\label{lmm:region_concave}
	Let $f:\R^d\mapsto\R$ be a concave function with maximizer $x^*$ over the convex set $\mc{C}$. Further, assume that $f$ is $G$-strongly concave with respect to $\ell_2$ norm in region $\mc{S}\cap\mc{C}$, where $\mathcal{S}=\Bp{x:\Norm{x-x^*}_2\leq A}$. If $f(x^*)-f(x)\leq\frac{AG}{2}$ and $c\in\mc{C}$, then $x\in\mathcal{S}$.
\end{lemma}
\begin{proof}
	By property of strong concavity, we know that, $f(x^*)-f(x)\geq\frac{G}{2}\Norm{x-x^*}_2$ for any $x\in\mathcal{S}\cap\mc{C}$. Now, suppose $x'$ satisfies $f(x^*)-f(x')\leq\frac{AG}{2}$, $x'\in\mc{C}$ and $x'\notin\mathcal{S}$. Then, we must have $\Norm{x'-x^*}_2>A$.
	
	Let $\gamma\in(0, 1)$ be some number such that $z=\gamma x'+(1-\gamma)x^*$ lies on the boundary of $\mathcal{S}$. By convexity, we also have $z\in\mc{C}$. Then, since $f$ is concave, we have $f(z)\geq\gamma f(x')+(1-\gamma)f(x^*)> f(x')$, where the second inequality is strict because $f$ is strongly concave in a region around $x^*$. Since $f(x^*)-f(x')\leq\frac{AG}{2}$, $f$ is $G$-strongly concave on $\mathcal{S}$ and $z$ lies on the boundary of $\mathcal{S}$, we have
	$$\frac{AG}{2}=\frac{G}{2}\Norm{z-x^*}_2\leq f(x^*)-f(z)< f(x^*)-f(x')\leq\frac{AG}{2}.$$
	This is a contradiction and thus we must have $x'\in\mathcal{S}$.
\end{proof}

The following lemma quantitatively describes how close $\widetilde{\Lambda}$ and $\Lambda^*$ needs to be to ensure an at most $\epsilon$ multiplicative constraint violation.
\begin{lemma}
\label{lmm:opt_guarantees}
Assume $\Norm{x}_2\leq M$ for any $x\in\mathcal{X}$. Let $\Sigma=\E_{X\sim\nu}\Mp{XX^\top}\succ\mathbf{0}$ and $\Lambda^*=\argmax_{\Lambda\succeq\bm{0}}\overline{D}(\Lambda)$. Then, for any $\epsilon>0$, if we have
$$\Norm{\widetilde{\Lambda}-\Lambda^*}_F\leq\frac{8\mu^2\lambda_{\min}(\Sigma)}{3M^2\lambda_{\max}(\Sigma)}\cdot\frac{\epsilon}{1+\epsilon},$$
then it holds that $y^\top\E_{X\sim\nu}\Mp{P_{\widetilde{\Lambda}}(X)XX^\top}^{-1}y\leq (1+\epsilon)c_\ell^2$ for any $y\in\mathcal{Y}_\ell$.
\end{lemma}
\begin{proof}
    Fix some $\epsilon>0$. First, notice that if we regard $P_\Lambda$ as a function of $q_\Lambda(x)=x^\top\Lambda x-1$, it then holds that
    $$\Norm{\nabla_{\Lambda}P_{\Lambda}(x)}_2=\Norm{\frac{\mathrm{d}P_{\Lambda}}{\mathrm{d}q_{\Lambda}}\nabla_{\Lambda}q_{\Lambda}(x)}_2\leq\abs{\frac{\mathrm{d}P_{\Lambda}}{\mathrm{d}q_{\Lambda}}}\Norm{xx^\top}_2\leq \abs{\frac{\mathrm{d}P_{\Lambda}}{\mathrm{d}q_{\Lambda}}}M^2\leq\frac{M^2}{8\mu},$$
    where we obtain the last inequality by using Lemma \ref{lmm:P_property}. Therefore, for any $x\in\mc{X}$ and $\widetilde{\Lambda}\succeq\bm{0}$, we have $\abs{P_{\widetilde{\Lambda}}(x)-P_{\Lambda^*}(x)}\leq\frac{M^2}{8\mu}\cdot\Norm{\widetilde{\Lambda}-\Lambda^*}_F$ by mean value theorem and Cauchy-Schwartz. inequality.
    
    Therefore, if we have $\Norm{\widetilde{\Lambda}-\Lambda^*}_F\leq\delta$, then 
    $$\abs{P_{\widetilde{\Lambda}}(x)-P_{\Lambda^*}(x)}\leq\frac{M^2\delta}{8\mu}\implies P_{\widetilde{\Lambda}}(x)\geq P_{\Lambda^*}(x)-\frac{M^2\delta}{8\mu}$$
    $$\implies\E_{X\sim\nu}\Mp{P_{\widetilde{\Lambda}}(X)XX^\top}\succeq\E_{X\sim\nu}\Mp{P_{\Lambda^*}(X)XX^\top}-\frac{M^2\delta}{8\mu}\E_{X\sim\nu}\Mp{XX^\top}.$$
    By Lemma \ref{lmm:schur_property}, we know that
    \begin{equation}
        \label{equ:ep_constraint}
        y^\top\E_{X\sim\nu}\Mp{P_{\widetilde{\Lambda}}(X)XX^\top}^{-1}y\leq c_\ell^2(1+\epsilon)\Longleftrightarrow\E_{X\sim\nu}\Mp{P_{\widetilde{\Lambda}}(X)XX^\top}\succeq\frac{yy^\top}{(1+\epsilon)c_\ell^2}.
    \end{equation}
    Let $\Sigma^*=\E_{X\sim\nu}\Mp{P_{\Lambda^*}(X)XX^\top}$. Therefore, to guarantee the condition in Eq. \eqref{equ:ep_constraint}, it is sufficient to guarantee that $\Sigma^*-\frac{M^2\delta}{8\mu}\Sigma\succeq\frac{yy^\top}{(1+\epsilon)c_\ell^2}$, which is equivalent to
    $$w^\top\Sigma^* w-\frac{M^2\delta}{8\mu}w^\top\Sigma w\geq\frac{(w^\top y)^2}{c_\ell^2(1+\epsilon)},\quad\forall\text{unit vector }w\in\R^d$$
    $$\Longleftrightarrow \frac{1}{w^\top\Sigma w}\cdot w^\top\Sp{\Sigma^*-\frac{yy^\top}{(1+\epsilon)c_\ell^2}}w\geq\frac{M^2\delta}{8\mu},\quad\forall\text{unit vector }w\in\R^d.$$
    
    Therefore, it is sufficient to choose $\delta$ such that
    $$\frac{M^2\delta}{8\mu}\leq \frac{1}{\lambda_{\max}(\Sigma)}\cdot \lambda_{\min}\Sp{\Sigma^*-\frac{yy^\top}{c_\ell^2(1+\epsilon)}}\leq\min_{w:\Norm{w}_2=1}\frac{1}{w^\top\Sigma w}\cdot w^\top\Sp{\Sigma^*-\frac{yy^\top}{(1+\epsilon)c_\ell^2}}w.$$
    Since $P_{\Lambda^*}$ satisfies the constraint defined in problem \eqref{equ:opt_barrier_2}, we have $\Sigma^*\succeq\frac{yy^\top}{c_\ell^2}$. Meanwhile, by Lemma \ref{lmm:P_property}, we know that $P_{\Lambda^*}(x)\geq\frac{\mu}{3}$ for any $x\in\mc{X}$, which means that $\Sigma^*\succeq\frac{\mu}{3}\cdot\Sigma$. That is, for any unit vector $w\in\R^d$, we have
    $$w^\top\Sigma^*w\geq\frac{\Sp{w^\top y}^2}{c_\ell^2}\quad\text{and}\quad w^\top\Sigma^*w\geq\frac{\mu}{3}\lambda_{\min}\Sp{\Sigma},$$
    which together implies $w^\top\Sigma^*w\geq\max\Bp{\frac{\mu}{3}\cdot\lambda_{\min}(\Sigma), \frac{\Sp{w^\top y}^2}{c_\ell^2}}$. Therefore, it holds that
    \begin{align*}
        w^\top\Sigma w-\frac{\Sp{w^\top y}^2}{(1+\epsilon)c_\ell^2}&\geq \max\Bp{\frac{\mu}{3}\cdot\lambda_{\min}(\Sigma), \frac{\Sp{w^\top y}^2}{c_\ell^2}}-\frac{\Sp{w^\top y}^2}{(1+\epsilon)c_\ell^2}\\
        &=\max\Bp{\frac{\mu}{3}\cdot\lambda_{\min}(\Sigma)-\frac{\Sp{w^\top y}^2}{(1+\epsilon)c_\ell^2}, \frac{\epsilon\Sp{w^\top y}^2}{(1+\epsilon)c_\ell^2}}\\
        &\geq\frac{\epsilon\mu}{3(1+\epsilon)}\cdot\lambda_{\min}\Sp{\Sigma}\\
        \implies\lambda_{\min}\Sp{\Sigma^*-\frac{yy^\top}{c_\ell^2(1+\epsilon)}}&\geq\frac{\epsilon\mu}{3(1+\epsilon)}\cdot\lambda_{\min}\Sp{\Sigma}.
    \end{align*}
    
    Therefore, to guarantee the condition in Eq. \eqref{equ:ep_constraint}, it is sufficient to have
    $$\frac{M^2\delta}{8\mu}=\frac{\epsilon\mu\lambda_{\min}(\Sigma)}{3(1+\epsilon)\lambda_{\max}(\Sigma)}\implies\mu=\frac{8\mu^2\lambda_{\min}(\Sigma)}{3M^2\lambda_{\max}(\Sigma)}\cdot\frac{\epsilon}{1+\epsilon},$$
    Thus, the proof is complete.
\end{proof}

The following lemma is a result of standard Schur complement technique.
\begin{lemma}
\label{lmm:schur_property}
If $\E_{X\sim\nu}\Mp{P(X)XX^\top}$ is invertible and $c_\ell>0$, then 
$$y^\top \E_{X\sim\nu}\Mp{P(X)XX^\top}^{-1}y\leq c_\ell^2\Longleftrightarrow \E_{X\sim\nu}\Mp{P(X)XX^\top}\succeq\frac{yy^\top}{c_\ell^2}.$$
\end{lemma}
\begin{proof}
For simplicity, let $A=E_{X\sim\nu}\Mp{P(X)XX^\top}\succ\bm{0}$. Then, we consider the block matrix $\matenv{A & y\\ y^\top & c_\ell^2}\in\R^{(d+1)\times(d+1)}$. Let $\matenv{u & a}^\top\in\R^{d+1}$ with $u\in\R^d$ be some vector.

Now, for one direction, suppose $y^\top A^{-1}y\leq c_\ell^2$ holds. Consider
$$\matenv{u & a}\matenv{A & y\\ y^\top & c_\ell^2}\matenv{u \\ a}=u^\top Au+2au^\top y+2c_\ell^2a^2:=r(u, a).$$
If we minimize $r(u, a)$ over $u$, which means to treat $a$ as fixed, we can get (by taking gradient and setting it to zero)
$$u^*=-aA^{-1}y\implies r(u^*, a)=a^2(c_\ell^2-y^\top A^{-1}y).$$
Since $y^\top A^{-1}y\leq c_\ell^2$, we know that $r(u^*, a)\geq 0$, which means $r(u, a)\geq 0$ for any $\matenv{u & a}^\top\in\R^{d+1}$.

Then, if we minimize $r(u, a)$ over $a$, we can get
$$a^*=-\frac{u^\top y}{c_\ell^2}\implies r(u, a^*)=u^\top Au-\frac{\Sp{u^\top y}^2}{c_\ell^2}.$$
Since $r(u, a)\geq 0$ for any $\matenv{u & a}^\top\in\R^{d+1}$, we know that $u^\top Au-\frac{\Sp{u^\top y}^2}{c_\ell^2}\geq 0$ for any $u\in\R^d$. That is, we have $A\succeq\frac{yy^\top}{c_\ell^2}$.

The other direction simply takes the above calculation in a reversed way and thus the proof is complete.
\end{proof}

\subsubsection{Properties of $P_\Lambda$}
A visualization of $P_\Lambda$ is given in Figure \ref{fig:visual_P}.
\begin{figure}[ht]
    \centering
    \includegraphics[width=\linewidth]{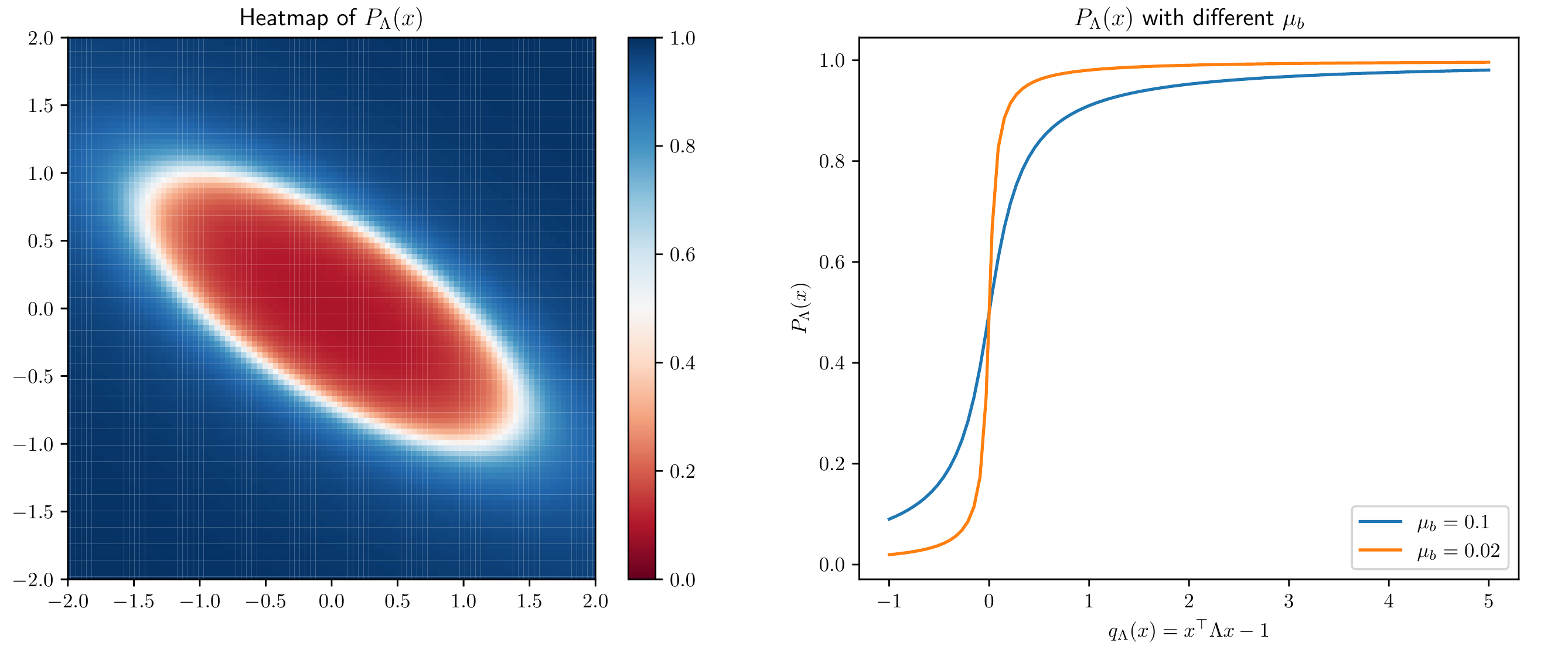}
    \caption{(left) A heatmap of some $P_{\Lambda}$ when problem dimension is $d=2$, which shows that $P_{\Lambda}$ is approximately an 0-1 threshold rule characterized by an ellipsoid. (right) A plot of $P_\Lambda$ as a function of $q_\Lambda(x)=x^\top\Lambda x-1$, which shows that the change of $P_\Lambda$ near the boundary of ellipsoid is sharper when the barrier weight $\mu$ is smaller.}
    \label{fig:visual_P}
\end{figure}

\begin{lemma}
\label{lmm:P_property}
The function $P_{\Lambda}(x)$ defined in \eqref{equ:P_sol}, if regarding as a function of $q_{\Lambda}(x)=x^\top\Lambda x-1\geq -1$, satisfies
\begin{itemize}
    \item $\lim_{q_\Lambda\rightarrow 0}P_\Lambda=\frac{1}{2}$ for any $\mu\in(0, 1)$
    
    \item When $q_\Lambda=-1$, $P_\Lambda=\frac{1}{2}+\mu-\frac{\sqrt{1+4\mu^2}}{2}\geq\frac{\mu}{3}$ and $P_\Lambda-\mu(\log(1-P_\Lambda)+\log(P_\Lambda))\leq 2\sqrt{\mu}$ for any $\mu\in(0, 1)$.

    \item $\frac{\mathrm{d}P_\Lambda}{\mathrm{d}q_\Lambda}=\frac{\mu\sqrt{q_\Lambda^2+4\mu^2}-2\mu^2}{q_\Lambda^2\sqrt{q_\Lambda^2+4\mu^2}}$ decreases as $q_\Lambda^2$ increases. Further, $\frac{\mathrm{d}P_\Lambda}{\mathrm{d}q_\Lambda}\in[0, \frac{1}{8\mu}]$. Thus, $P_\Lambda$ increases monotonically as $q_\Lambda$ increases and $P_\Lambda(x)\geq \frac{\mu}{3}$ for any $x\in\mc{X}$ and $\Lambda\succeq\bm{0}$.
    
    \item $\frac{\mathrm{d}P_\Lambda}{\mathrm{d}q_\Lambda}\vert_{q_\Lambda=\pm 1}\geq\frac{\mu}{10}$ and $\frac{\mathrm{d}P_\Lambda}{\mathrm{d}q_\Lambda}\geq\frac{\mu}{2q_\Lambda^2}$ when $q_\Lambda^2\geq 12\mu^2$.
\end{itemize}
\end{lemma}
\begin{proof}
For simplicity, we will drop the subscript $\Lambda$ and just treat $P$ as a function of $q$. That is, we have
$$P(q)=\frac{1}{2}-\frac{\mu}{q}+\frac{\sqrt{\Sp{2\mu-q}^2+4\mu q}}{2q}.$$
We prove each bullet point separately.
\begin{itemize}
    \item Since $P(q)$ also satisfies Eq. \eqref{equ:P_equ}, which in simpler form is $\frac{\mu}{1-P(q)}-\frac{\mu}{P(q)}=q$, we can easily see that $P(q)=\frac{1}{2}$ satisfies this equation when $q=0$.
    
    \item By direction computation, we can get $P(-1)=\frac{1}{2}+\mu-\frac{\sqrt{1+4\mu^2}}{2}$. To show this is greater than $\frac{\mu}{3}$ for any $\mu\in[0, 1]$, consider $\ell(\mu)=P(-1)-\frac{\mu}{3}$. It is easy to check that $\ell(0)=0$ and $\ell(1)>0$. Then, since $\ell'(\mu)=\frac{2}{3}-\frac{2\mu}{\sqrt{1+4\mu^4}}$ is initially greater than 0 and then smaller than 0, we know $\ell(\mu)$ first increases and then decreases on $[0, 1]$. Thus, $\ell(\mu)\geq 0$ on $[0, 1]$ and thus $P(-1)\geq\frac{\mu}{3}$ for any $\mu\in[0, 1]$.
    
    For the second part, define $\tilde{\ell}(\mu)=2\sqrt{\mu}-P(-1)+\mu\Sp{\log(1-P(-1))+\log(P(-1))}$. Then, by utilizing the fact that $P$ satisfies Eq. \eqref{equ:P_equ}, we can compute its derivative and get $\frac{\mathrm{d}\tilde{\ell}}{\mathrm{d}\mu}=\frac{1}{\sqrt{\mu}}+\log(1-P(-1))+\log(P(-1))$. We can check that on the domain $(0, 1)$, we have $\frac{\mathrm{d}^2\tilde{\ell}}{\mathrm{d}\mu^2}=-\frac{1}{2\mu^{3/2}}+\frac{1}{\mu}-\frac{2}{\sqrt{1+4\mu^2}}\cdot\frac{2\sqrt{\mu\Sp{1+4\mu^2}}-4\mu^{3/2}-\sqrt{1+4\mu^2}}{2\mu^{3/2}\sqrt{1+4\mu^2}}\leq 0$ on $(0, 1)$, which means that $\frac{\mathrm{d}\tilde{\ell}}{\mathrm{d}\mu}$ is monotonically decreasing. To see why the second derivative is smaller than 0, we can compute
    $$\Sp{4\mu^{3/2}+\sqrt{1+4\mu^2}}-4\mu\Sp{1+4\mu^2}=\Sp{1-2\mu}^2+8\mu^{3/2}\sqrt{1+4\mu^2}\geq 0.$$
    Thus, $\frac{\mathrm{d}\tilde{\ell}}{\mathrm{d}\mu}$ is initially greater than 0 and then smaller than 0 on $(0, 1)$. It is easy to verify that $\lim_{\mu\rightarrow 0}\tilde{\ell}=0$ and $\tilde{\ell}(1)>0$. Therefore, we have $\tilde{\ell}(\mu)\geq 0$ for any $\mu\in(0, 1)$.
    
    \item We can get $\frac{\mathrm{d}P}{\mathrm{d}q}=\frac{\mu\sqrt{q^2+4\mu^2}-2\mu^2}{q^2\sqrt{q^2+4\mu^2}}$ by direct computation. To show it is decreasing as $q^2$ increasing, we consider $\tilde{f}(z)=\frac{\mu\sqrt{z+4\mu^2}-2\mu^2}{z\sqrt{z+4\mu^2}}$ and it is sufficient to show that $\frac{\mathrm{d}\tilde{f}}{\mathrm{d}z}<0$ for any $z>0$. Again by direct computation, we have
    $$\frac{\mathrm{d}\tilde{f}}{\mathrm{d}z}=\frac{\mu\Sp{8\mu^3+3\mu z-\Sp{z+4\mu^2}^{3/2}}}{z^2\Sp{z+4\mu^2}^{3/2}}.$$
    By direct computation, We can show that $\Sp{z+4\mu^2}^3-\Sp{8\mu^3+3\mu z}^2=z^3+3z^2\mu^2>0$ for any $z>0$ and $\mu\in[0, 1]$. Thus, $\frac{\mathrm{d}\tilde{f}}{\mathrm{d}z}<0$ and thus $\frac{\mathrm{d}P}{\mathrm{d}q}$ is decreasing as $q^2$ increases.
    
    It is obvious that $\frac{\mathrm{d}P}{\mathrm{d}q}\geq 0$ for any $q^2\geq 0$ and $\mu\in[0, 1]$ since we always have $\mu\sqrt{q^2+4\mu^2}\geq 2\mu^2$. Thus, the maximum value could potentially happen is when $q^2\rightarrow 0$, which can be evaluated by using L'Hospital's rule. A direct computation gives us $\lim_{q^2\rightarrow 0}\frac{\mathrm{d}P}{\mathrm{d}q}=\frac{1}{8\mu}$. Thus, we can conclude that $\frac{\mathrm{d}P}{\mathrm{d}q}\in\Mp{0, \frac{1}{8\mu}}$. Therefore, $P$ increases monotonically as $q$ increases, which implies that $P_\Lambda(x)\geq\frac{\mu}{3}$ for any $x\in\mc{X}$ and $\Lambda$.
    
    \item By direct computation, we have $\frac{\mathrm{d}P_\Lambda}{\mathrm{d}q_\Lambda}\vert_{q_\Lambda=\pm 1}=\mu\Sp{1-\frac{2\mu}{\sqrt{1+4\mu^2}}}\geq\mu\Sp{1-\frac{2}{\sqrt{5}}}\geq\frac{\mu}{10}$ for any $\mu\in[0, 1]$. The reason is that we can easily see $\frac{2\mu}{\sqrt{1+4\mu^2}}$ is increasing in $\mu$.
    
    Finally, notice that when $2\mu\leq\frac{1}{2}\sqrt{q^2+4\mu^2}$, which is equivalent to $q^2\geq 12\mu^2$, we have
    $$\frac{\mathrm{d}P}{\mathrm{d}q}=\frac{\mu\sqrt{q^2+4\mu^2}-2\mu^2}{q^2\sqrt{q^2+4\mu^2}}\geq\frac{\mu\sqrt{q^2+4\mu^2}-\frac{\mu}{2}\sqrt{q^2+4\mu^2}}{q^2\sqrt{q^2+4\mu^2}}=\frac{\mu}{2q^2}.$$
\end{itemize}
Thus, the proof is complete.
\end{proof}

\subsection{An Alternative Approach to \textsc{OptimizeDesign}}
Based on the analysis in Section \ref{sec:opt_proof}, we know that maximizing $\overline{D}(\cdot)$ is equivalent to maximizing $D(\cdot)$. Therefore, as an alternative to Algorithm \ref{algo:solve_sgd}, which maximizes $D(\cdot)$ through stochastic gradient ascent, it is natural to have an algorithm that directly maximizes $\overline{D}(\cdot)$. Here, we will consider subgradient ascent.

Recall that $\overline{D}:\mathbb{S}^{d}_+\mapsto\R$ is defined as
$$\overline{D}(\Lambda)=\E_{X\sim\nu}\Mp{P_{\Lambda}(X)-\mu\Sp{\log(1-P_{\Lambda}(X))+\log(P_{\Lambda}(X))}-P_{\Lambda}(X)X^\top\Lambda X}+\frac{1}{c_\ell^2}\cdot f(\Lambda),$$
where $f(\Lambda)$ is defined in problem \eqref{equ:opt_semi}. The subgradient of $\overline{D}(\Lambda)$ is
\begin{align*}
    \partial\overline{D}(\Lambda)&=\E_{X\sim\nu}\Mp{\Sp{1\!+\!\frac{\mu}{1-P_{\Lambda}(x)}\!-\!\frac{\mu}{P_{\Lambda}(X)}\!-\!X^\top\Lambda X}\nabla P_{\Lambda}(X)\!-\!P_{\Lambda}(X)XX^\top}\!+\!\frac{\partial f(\Lambda)}{c_\ell^2}\tag{The first term is differentiable}\\
    &=\frac{\partial f(\Lambda)}{c_\ell^2}-\E_{X\sim\nu}\Mp{P_{\Lambda}(X)XX^\top}.\tag{Since $P_{\Lambda}(X)$ solves Eq. \eqref{equ:P_equ}}
\end{align*}

Therefore, to run subgradient ascent, we only need to find an element in $\partial f(\Lambda)$, which can be obtained by solving the following optimization problem as claimed by Lemma \ref{lemma:partial_f}.
\begin{equation}
    \label{equ:opt_semi2}
    \begin{array}{rl}
        \min_{\Gamma} & \inner{\Gamma, \Lambda} \\
        \text{subject to} & \Gamma\succeq yy^\top,\quad\forall y\in\mc{Y}_\ell,\\
        & \Gamma\preceq 2\sum_{y\in\mc{Y}_\ell}yy^\top.
    \end{array}
\end{equation}

As a result, we have Algorithm \ref{algo:solve_sgd2} as an alternative to solve \textsc{OptimizeDesign}. Compared to Algorithm \ref{algo:solve_sgd}, which needs to maintain $\abs{\mc{Y}_\ell}d^2$ number of objective variables, Algorithm \ref{algo:solve_sgd2} only has $d^2$ variables. However, each iteration of Algorithm \ref{algo:solve_sgd2} is computationally more intensive since finding a subgradient needs to solve the problem \eqref{equ:opt_semi2}.

\begin{algorithm}[ht]
\caption{Projected Stochastic Subgradient Ascent to Solve \textsc{OptimizeDesign}}
\label{algo:solve_sgd2}
\begin{algorithmic}[1]
\STATE{\textbf{Input:} Number of iterations $K$; number of samples $u$; barrier weight $\mu_b \in(0, 1)$}
\STATE{Initialize $\hat{\Lambda}^{(0)}=\mathbf{0}$}
\FOR{$k=0, 1, 2, \dots, K-1$}
\STATE{Sample $x_k\sim\nu$}
\STATE{Solve problem \eqref{equ:opt_semi2} with current $\hat{\Lambda}^{(k)}$ to get $\Gamma^{(k)}$}
\STATE{Set $g_{k}=\frac{\Gamma^{(k)}}{c_\ell^2}-P_{\hat{\Lambda}^{(k)}}(x_k)x_kx_k^\top$}
\STATE{Set $\hat{\Lambda}^{(k+1)}\leftarrow \hat{\Lambda}^{(k)}+\eta_kg_{k}$, where $\eta_k=\frac{1}{\sqrt{2\sum_{s=1}^k\Norm{g_{s}}_2^2}}$}
\STATE{Update $\hat{\Lambda}^{(k+1)}\leftarrow\Pi_{\mathbb{S}^d_{+}}(\hat{\Lambda}^{(k+1)})$, a projection to the set of $d\times d$ PSD matrices}
\ENDFOR
\STATE{Let $\hat{\Lambda}=\frac{1}{K}\sum_{k=1}^{K}\hat{\Lambda}^{(k)}$}
\STATE{Find $s^*\leftarrow\argmax_{s\in[0, 1]}\overline{D}_{E}(s\cdot\hat{\Lambda})$, where $\overline{D}_{E}$ is the empirical version of $\overline{D}$, evaluated using $u$ i.i.d. samples}\label{line:re-scaling_alt}
\STATE{\textbf{return} $\widetilde{\Lambda}=s^*\cdot\hat{\Lambda}$}
\end{algorithmic}
\end{algorithm}

A result similar to Theorem \ref{theo:opt_detail} can also be obtained for Algorithm \ref{algo:solve_sgd2}, which is given in Theorem \ref{theo:opt_alt}. The bounds are almost identical except that the old lower bound for $K$ depends on $\abs{\mc{Y}_\ell}^3$ while the new one depends on $\abs{\mc{Y}_\ell}$. Steps identical to the proof of Theorem \ref{theo:opt_detail} will be skipped in the proof of Theorem \ref{theo:opt_alt}.

\begin{theorem}
\label{theo:opt_alt}
Let $\Lambda^*\in\argmax_{\Lambda\succeq\bm{0}}\overline{D}(\Lambda)$ and take other settings the same as that in Theorem \ref{theo:opt_detail}.

Then, $\Lambda^*$ is unique. Further, for any $\epsilon>0$ and $\delta>0$, suppose it holds that
\begin{align*}
    \mu&\leq \min\Bp{ \sqrt{\frac{3\kappa(\Sigma)\Norm{\Lambda^*}_F M^2}{8}\cdot\frac{1+\epsilon}{\epsilon}}, \frac{4}{9}\Norm{\Lambda^*}_F^2M^4, \frac{1}{2\sqrt{3}}}\\
    K&\geq\frac{288\kappa(\Sigma)^2\Norm{\Lambda^*}_F^4M^4(M^4+4\abs{\mc{Y}_\ell}C_\ell^2)\cdot\Sp{2\Norm{\Lambda^*}_FM^2+1}^4\log(6/\delta)}{\omega^2\mu^6}\cdot\Sp{\frac{1+\epsilon}{\epsilon}}^2\\
    u&\geq \frac{576\kappa(\Sigma)^2\Norm{\Lambda^*}_F^2M^8\cdot\Sp{2\Norm{\Lambda^*}_FM^2+1}^4\log(6/\delta)}{\omega^2\mu^6}\cdot\Sp{\frac{1+\epsilon}{\epsilon}}^2.
\end{align*}

Then, with probability at least $1-\delta$, Algorithm \ref{algo:solve_sgd} will output  $\widetilde{\Lambda}$ that satisfies
\begin{itemize}
    \item $y^\top \E_{X\sim\nu}\Mp{P_{\widetilde{\Lambda}}(X)XX^\top}^{-1}y\leq (1+\epsilon)c_\ell^2,\quad\forall y\in\mc{Y}_\ell$.
    
    \item $\E_{X\sim\nu}\Mp{P_{\widetilde{\Lambda}}(X)}\leq \E_{X\sim\nu}\Mp{\widetilde{P}(X)}+4\sqrt{\mu}$, where $\widetilde{P}$ is the optimal solution to problem \eqref{equ:opt_mu_bound}.
\end{itemize}
\end{theorem}

\begin{proof}
\textbf{First Bullet Point.} Similar to the proof of Theorem \ref{theo:opt_detail}, let $\hat{\Lambda}$ be the parameter obtained by Algorithm \ref{algo:solve_sgd2} just before the re-scaling step (line \ref{line:re-scaling_alt}). Then, by Theorem 3.13 of \cite{orabona2019modern}, with probability at least $1-\frac{\delta}{3}$, it holds that
$$\overline{D}(\Lambda^*)-\overline{D}(\hat{\Lambda})\leq\frac{\mathrm{Reg}(K)+2\sqrt{2K\log(6/\delta)}}{K},$$
where $\mathrm{Reg}(K)$ is the regret of running projected stochastic subgradient ascent for $K$ steps with $\eta_k$ specified in Algorithm \ref{algo:solve_sgd2}. Meanwhile, by Theorem 4.14 of \cite{orabona2019modern} also, we have $\mathrm{Reg}(K)=\sqrt{2}B^2\sqrt{\sum_{k=1}^{K}\Norm{g_{k}}_2^2}$, where $B=\Norm{\Lambda^*}_F$. Since $g_{k}=\frac{\Gamma^{(k)}}{c_\ell^2}-P_{\hat{\Lambda}^{(k)}}(x_k)x_kx_k^\top$ and $\Norm{\Gamma^{(k)}}_F\leq 2\Norm{\sum_{y\in\mc{Y}_\ell}yy^\top}_F$, we can easily get $\Norm{g_{k}}_2^2\leq 2M^4+\frac{8}{c_\ell^2}\sum_{y\in\mc{Y}_\ell}\Norm{y}_2^4=2M^4+8\abs{\mc{Y}_\ell}C_\ell^2$. Thus, we have
\begin{equation}
    \label{equ:regret2}
    \mathrm{Reg}(K)\leq 2\Norm{\Lambda^*}_F^2\sqrt{M^4+4\abs{\mc{Y}_\ell}C_\ell^2}\cdot\sqrt{K}:=C_{\mathrm{Reg}}\sqrt{K}
\end{equation}
\begin{equation}
    \label{equ:D_diff2}
    \implies \overline{D}(\Lambda^*)-\overline{D}(\hat{\Lambda})\leq\frac{C_\mathrm{Reg}+2\sqrt{2\log(6/\delta)}}{\sqrt{K}},
\end{equation}

We now consider the effect of using $u$ i.i.d. samples in the re-scaling step. Since re-scaling always increases the function value, we must have $\overline{D}_E(\hat{\Lambda})\leq \overline{D}_E(\widetilde{\Lambda})$. 

Then, after \textbf{exactly the same} steps of analysis, we can get the following same lower bound for $K$,
\begin{equation}
    \label{equ:K_bound2}
    K\geq\Sp{\frac{3\kappa(\Sigma)M^2\Sp{C_\mathrm{Reg}+2\sqrt{2\log(6/\delta)}}}{2G\mu^2}\cdot\frac{1+\epsilon}{\epsilon}}^2,
\end{equation}
with a different value of $C_{\mathrm{Reg}}$.

\textbf{Second Bullet Point.} We then prove the upper bound for primal objective value $\E_{X\sim\nu}\Mp{P_{\widetilde{\Lambda}}(X)}$, which explains the reason why an extra re-scaling step is needed. Let $\hat{\bm{\Lambda}}=(\hat{\Lambda}_y)_{y\in\mc{Y}_\ell}$ be a set of PSD matrices that solves problem \eqref{equ:opt_semi} with parameter $\hat{\Lambda}$ and $\widetilde{\bm{\Lambda}}=s^*\cdot\hat{\bm{\Lambda}}$, where $s^*=\argmax_{s\in[0, 1]}\overline{D}_E(s\cdot\hat{\Lambda})$. Since the constraint in problem \eqref{equ:opt_semi} requires $\sum_{y\in\mc{Y}_\ell}\hat{\Lambda}_y=\hat{\Lambda}$, we have $\sum_{y\in\mc{Y}_\ell}\widetilde{\Lambda}_y=\widetilde{\Lambda}$, which is the output of Algorithm \ref{algo:solve_sgd2}.

Define $g(s)=D_E(s\cdot\widetilde{\bm{\Lambda}})$. By construction, we know that $g(s)$ is maximized at $s=1$ because $\overline{D}_E(s\cdot\hat{\Lambda})=D_E(s\cdot\hat{\bm{\Lambda}})$ for any $s\geq 0$ as shown in Lemma \ref{lmm:rescaling}, which means that $s^*=\argmax_{s\in[0, 1]}D_E(s\cdot\hat{\bm{\Lambda}})$. Therefore, we have $g'(1)\geq 0$, which in turn gives us
$$g'(1)=\frac{1}{c_\ell^2}\sum_{y\in\mc{Y}_\ell}y^\top\widetilde{\Lambda}_yy-\frac{1}{u}\sum_{i=1}^{u}P_{\widetilde{\Lambda}}(x_i)x_i^\top\widetilde{\Lambda} x_i\geq 0.$$
Then, after \textbf{exactly the same} steps of analysis, we can get $\E_{X\sim\nu}\Mp{P_{\widetilde{\Lambda}}(X)}\leq\E_{X\sim\nu}\Mp{\widetilde{P}(X)}+4\sqrt{\mu}$, where $\widetilde{P}$ is the optimal solution of the problem \eqref{equ:opt_mu_bound}.
\end{proof}

\subsubsection{Technical Lemmas}

\begin{lemma}
\label{lemma:partial_f}
The optimal value of the optimization problem \eqref{equ:opt_semi2} with parameter $\Lambda\succeq\bm{0}$ is equal to $f(\Lambda)$. Further, let $\Gamma^*(\Lambda)$ be an optimal solution to \eqref{equ:opt_semi2}. Then, it holds that $\Gamma^*(\Lambda)\in\partial f(\Lambda)$ and $\Norm{\Gamma^*(\Lambda)}\leq 2\Norm{\sum_{y\in\mc{Y}_\ell}yy^\top}_F$.
\end{lemma}
\begin{proof}
Alternatively, we first consider the following optimization problem.
\begin{equation}
    \label{equ:opt_semi3}
    \begin{array}{rl}
        \max_{\Lambda_y, \Sigma} & \sum_{y\in\mc{Y}_\ell}y^\top\Sp{\Lambda_y-2\Sigma}y \\
        \text{subject to} & \Lambda=\sum_{y\in\mc{Y}_\ell}\Lambda_y-\Sigma,\\
         & \Sigma\succeq\bm{0}, \Lambda_y\succeq\bm{0},\quad\forall y\in\mc{Y}_\ell.
    \end{array}
\end{equation}
Since $y^\top\Sigma y\geq 0$ for any $y\in\mc{Y}_\ell$ and $\Sigma\succeq\bm{0}$, it is clear that problem \eqref{equ:opt_semi3} has the same optimal value as problem \eqref{equ:opt_semi}. Then, let $\Gamma\in\R^{d\times d}$ be the dual variable for the equality constraint in problem \eqref{equ:opt_semi3}. We can have its dual problem to be
\begin{align*}
    \min_\Gamma\max_{\substack{\Lambda_y\succeq\bm{0},\forall y\in\mc{Y}_\ell,\\ \Sigma\succeq\bm{0}}}&\sum_{y\in\mc{Y}_\ell}\inner{yy^\top, \Lambda_y-2\Sigma}+\inner{\Gamma, \Lambda+\Sigma-\sum_{y\in\mc{Y}_\ell}\Lambda_y}\\
    \implies\min_\Gamma\max_{\substack{\Lambda_y\succeq\bm{0},\forall y\in\mc{Y}_\ell,\\ \Sigma\succeq\bm{0}}}&\inner{\Gamma, \Lambda}+\inner{\Sigma, \Gamma-2\sum_{y\in\mc{Y}_\ell}yy^\top}+\sum_{y\in\mc{Y}_\ell}\inner{\Lambda_y, yy^\top-\Gamma}.
\end{align*}
In order for the above optimization problem to have finite value, we must have $\Gamma\preceq 2\sum_{y\in\mc{Y}_\ell}yy^\top$ and $\Gamma\succeq yy^\top$ for any $y\in\mc{Y}_\ell$. Therefore, we obtain the following dual problem.
$$\begin{array}{rl}
        \min_{\Gamma} & \inner{\Gamma, \Lambda} \\
        \text{subject to} & \Gamma\succeq yy^\top,\quad\forall y\in\mc{Y}_\ell,\\
        & \Gamma\preceq 2\sum_{y\in\mc{Y}_\ell}yy^\top.
\end{array}.$$
This is exactly the problem \eqref{equ:opt_semi2}. Then, we can notice the Slater's condition is clearly satisfied by problem \eqref{equ:opt_semi2}, which means the strong duality holds. Therefore, problem \eqref{equ:opt_semi2} has the same optimal value as \eqref{equ:opt_semi3}, which is the same as \eqref{equ:opt_semi}.

Since $f(\Lambda)$ is concave in $\Lambda$ as shown in Lemma \ref{lmm:f_lambda}, to show that $\Gamma^*(\Lambda)\in\partial f(\Lambda)$, consider arbitrary $\Lambda, \Lambda'\succeq\bm{0}$. Then, we have
$$f(\Lambda)+\inner{\Gamma^*(\Lambda), \Lambda'-\Lambda}=\inner{\Gamma^*(\Lambda), \Lambda}+\inner{\Gamma^*(\Lambda), \Lambda'-\Lambda}=\inner{\Gamma^*(\Lambda), \Lambda'}\geq f(\Lambda').$$
The first equality holds because the optimal value of problem \eqref{equ:opt_semi2} is $f(\Lambda)$ as just shown above. The last inequality holds because $\Gamma^*(\Lambda)$ is a feasible solution to the problem \eqref{equ:opt_semi2} with parameter $\Lambda'$. Therefore, we have $\Gamma^*(\Lambda)\in\partial f(\Lambda)$. 

Finally, since the constraint of problem \eqref{equ:opt_semi2} requires $\Gamma^*(\Lambda)\preceq 2\sum_{y\in\mc{Y}_\ell }yy^\top$, we can obtain $\Norm{\Gamma^*(\Lambda)}_F\leq 2\Norm{\sum_{y\in\mc{Y}_\ell}yy^\top}_F$ as a direct consequence of Lemma \ref{lemma:frobenius_order}.
\end{proof}

\begin{lemma}
\label{lemma:frobenius_order}
For $A, B\in\mathbb{S}^{d\times d}$, if $A\succeq B\succeq \bm{0}$, then $\Norm{A}_F\geq\Norm{B}_F$.
\end{lemma}
\begin{proof}
Let $\lambda_1, \dots, \lambda_d$ and $\gamma_1, \dots, \gamma_d$ be eigenvalues of $A$ and $B$, respectively. Let $v_1, \dots, v_d$ be a set of orthogonal unit eigenvectors of matrix $A$. Then, we have 
$$\Norm{A}_F=\sqrt{\mathrm{tr}\Sp{AA}}=\sqrt{\mathrm{tr}\Sp{\Sp{\sum_{i=1}^{d}\lambda_i v_iv_i^\top}\Sp{\sum_{i=1}^{d}\lambda_i v_iv_i^\top}}}=\sqrt{\sum_{i=1}^{d}\lambda_i^2}.$$
Similarly, we have $\Norm{B}_F=\sqrt{\sum_{i=1}^d\gamma_i^2}$. By Corollary 7.7.4 in \cite{horn2012matrix}, since $A\succeq B\succeq\bm{0}$, we know that $\lambda_i\geq\gamma_i\geq 0$ for each $i$. Therefore, we have $\Norm{A}_F\geq\Norm{B}_F$.
\end{proof}

\section{Selective Sampling Algorithm for Unknown Distribution $\nu$}
\subsection{Statement and proof of Theorem~\ref{thm:upper_bound_tau_input_unknown_nu}}\label{sec:unknown_nu}

Consider now the case where we do not know $\nu$ exactly, and are returned $(\widehat{P}_\ell,\widehat{\Sigma}_{\widehat{P}_\ell})$ that only approximate their ideals. Algorithm~\ref{algo:best_arm_meta} can still be employed to solve this case where $\nu$ is unknown, but at the cost of sampling some historical data.
Note that compared to the case where $\nu$ is know, it assumes the knowledge of an upper bound on $\sup_{x\in\text{support}(\nu)}\|x\|$
.
It also relies on a multiplicative factor change in the constraint of the optimization problem, in order to account for the possible constraint violation of the output of the subroutine. The last difference is the use of an approximation of the covariance matrix to compute the estimator. The covariance matrix is empirically approximated by injecting additional unlabeled samples (historical data). With that, although we do not know $\nu$ but we can approximate the relevant quantities, such as the covariance matrix $\E_{X\sim\nu}[XX^\top]$.\\\\
Let us detail the properties of the implementation of  $\widehat{P}_\ell, \widehat{\Sigma}_{\widehat{P}_\ell} \leftarrow $\textsc{OptimizeDesign}$(\mc{\mc{Z}_\ell},2^{-\ell},\tau)$ we use at each round $\ell$. 

First, $\widehat{P}_\ell$ has the properties described in Theorem~\ref{theo:opt} (by using Algorithm~\ref{algo:solve_sgd}).
More explicitly, let $\epsilon_\ell := 2^{-\ell}$, $B< \infty$ such that $\max_{x \in \mc{X}} |\langle x,\theta_* \rangle| \leq B$, and $\sigma < \infty$ such that $\E[(y_s - \langle \theta_*, x_s \rangle)^2|x_s] \leq \sigma^2$.
If
$$\beta_{\delta,\ell} := 4(1+\varepsilon)^2 \left(4\sqrt{B^2 + \sigma^2}+1\right)^2\log(4 \ell^2 |\mc{Z}|^2/\delta)$$ 
then $\widehat{P}_\ell$ is such that
\begin{itemize}
    \item $\max_{z,z' \in \mc{\mc{Z}_\ell}} \frac{\|z-z'\|^2_{\E_{X\sim\nu}[ \tau \widehat{P}_\ell(X) X X^\top]^{-1}}}{\epsilon_\ell^2} \beta_{\delta,\ell} \leq  1 +\varepsilon $.
    
    \item $\E_{X\sim\nu}\Mp{\widehat{P}_\ell(X)}\leq \E_{X\sim\nu}\Mp{\widetilde{P}_\ell(X)}+4\sqrt{\mu_b}$, where $\widetilde{P}_\ell$ is the optimal solution to problem \eqref{equ:opt_modified_bis}.
\end{itemize}
\begin{equation}
	\label{equ:opt_modified_bis}
	\begin{array}{rl}
        \min_{P} & \E_{X\sim\nu}\Mp{P(X)}  \\
        \text{subject to} & \max_{z,z' \in \mc{\mc{Z}_\ell}} \frac{\|z-z'\|^2_{\E_{X\sim\nu}[ \tau P(X) X X^\top]^{-1}}}{\epsilon_\ell^2} \beta_{\delta,\ell} \leq  1 ,\\
        & 0\leq P(x)\leq 1-\mu_b,\quad\forall x\in\mc{X}.
    \end{array}
\end{equation}

    
where $\mu_b\geq0$.
The quantity $\E_{X\sim\nu}\Mp{\widetilde{P}_\ell(X)}$ that uses $\mu_b >0$ is easily related to the value when $\mu_b=0$ through a simple scaling factor of $\frac{1}{1-\mu_b}$ (see proof below).

$\widehat{\Sigma}_{\widehat{P}_\ell}$ is the empirical covariance matrix of $\Sigma_{\widehat{P}_\ell} := \E_{X \sim \nu}[ \widehat{P}_\ell(X) X X^\top ]$ using historical data and is such that
\begin{align*}
    (1-\gamma)\Sigma_{\widehat{P}_\ell} \preceq \widehat{\Sigma}_{\widehat{P}_\ell} \preceq (1+\gamma)\Sigma_{\widehat{P}_\ell}
\end{align*}
where $\gamma\geq 0$.

Again, while we think of historical data as independent data collected offline before the start of the game, in practice this historical data could just come from previous rounds (which is not technically correct since its use may introduce some dependencies).



\begin{theorem}[Upper bound]\label{thm:upper_bound_tau_input_unknown_nu}
Fix any $\delta \in (0,1)$. Let $\Delta = \min_{z \in \mc{Z} \setminus z_*} \langle z_* - z, \theta_* \rangle$ and set $$\beta_\delta = 256 (1+\varepsilon)^2 \left(4\sqrt{B^2 + \sigma^2}+1\right)^2\log(4 \log_2^2(\tfrac{4}{\Delta})|\mc{Z}|^2/\delta)
.$$
For any $\tau \geq \rho(\nu) \beta_\delta$ there exists a $\delta$-PAC selective sampling algorithm that collects $\mc{T}$ historical data before the start of the game, observes $\mc{U}$ unlabeled examples, and requests just $\mc{L}$ labels that satisfies
\begin{itemize}
    \item $\mc{U} \leq  \log_2(\tfrac{4}{\Delta}) \tau $,
    \item $\mc{L} \leq \frac{1}{1-\mu_b}\min_{\lambda \in \triangle_{\mc{X}}}\rho(\lambda)\beta_\delta + \frac{5\tau}{1-\mu_b} \sqrt{\mu_b}\quad\text{ subject to }\quad \tau \geq \|\lambda / \nu \|_\infty \rho(\lambda) \,  \beta_\delta$, and 
    \item $\mc{T} \leq  \log_2(\tfrac{4}{\Delta})(K+u+\kappa_\delta)$
\end{itemize}
with probability at least $1-\delta$.

Here, the sample complexity for estimating the covariance matrix is bounded by $\kappa_\delta = \lceil2K_{\psi_2}^2(\sqrt{d\ln 9/c_1}+\sqrt{\frac{\log(2/\delta)}{c_1}})\max\{1, 20\|\theta_*\|_{\E_{X \sim \nu}[X X^\top]}\}\rceil$ (where the sub-gaussian norm $K_{\psi_2} = \max_{s, P}\|\sqrt{P(\widetilde{x}_s)}\Sigma_P^{-1/2}\widetilde{x}_s\|_{\psi_2}$ 
), and the contributions from the optimization problem to compute $\{\widehat{P}_\ell\}_\ell$ are
$$K=\widetilde{O}\Sp{\frac{\abs{\mc{Z}}^6\kappa(\Sigma)^2\Norm{\Lambda^*}_2^8M^{16}}{\omega^2\mu_b^6}}\cdot\Sp{\frac{1+\epsilon}{\epsilon}}^2,\quad 
u=\widetilde{O}\Sp{\frac{\kappa(\Sigma)^2\Norm{\Lambda^*}_2^6M^{16}}{\omega^2\mu_b^6}}\cdot\Sp{\frac{1+\epsilon}{\epsilon}}^2,$$
\end{theorem}
Naturally, we have a trade-off on the subroutine tolerance $\mu_b$. In order to get a better solution of the optimization over the selection rule $P$ (and thus get a smaller $\sum_{t=(\ell-1)\tau+1}^{\ell \tau} P(x_t)$ term), the subroutine needs more unlabeled samples.
However, it suffices to take $\mu_b = \frac{1}{\tau^2}$ to make $\mc{U}$, and $\mc{L}$ roughly match those of the case when $\nu$ was known.

The proof of this theorem is established through several results, which we provide in Section~\ref{sec:correctness_unknown_nu}.

\subsection{Lemmas for the correctness}\label{sec:correctness_unknown_nu}
We first state here the correctness of Algorithm~\ref{algo:best_arm_meta} in the case where $\nu$ is unknown.
\begin{lemma}\label{lmm:correctness_unknow_nu}
With probability at least $1-\delta$ we have for all stages $\ell \in \mathbb{N}$, we have that $z_* \in \mc{Z}_\ell$ and $\max_{z \in \mc{Z}_\ell} \langle z_* - z, \theta_* \rangle \leq 4\epsilon_\ell$. 
\end{lemma}
The proof of the correctness lemma is established though several lemmas. First we provide Lemma~\ref{lmm:bound_empiricov} guaranteeing concentration of empirical covariance matrices, which is obtained by sampling $\kappa$ additional measurements. Then we show in Proposition~\ref{prop:rips_bound_approx_covariance} that the RIPS estimator does not suffer from using that empirical covariance matrix.

\begin{lemma}\label{lmm:bound_empiricov}
For any $P: \mc{X} \rightarrow [0,1]$, let $\Sigma_P = \E_{X\sim\nu}[P(X)XX^\top]$, $\widehat{\Sigma}_P = \frac{1}{\kappa}\sum_{s=1}^\kappa P(\widetilde{x}_s) \widetilde{x}_s \widetilde{x}_s^\top$. Define $K_{\psi_2} = \max_s\|\sqrt{P(\widetilde{x}_i)}\Sigma_P^{-1/2}\widetilde{x}_s\|_{\psi_2}$ 
. With probability at least $1-2\exp(-c_1 t^2/K_{\psi_2}^4)$ holds 
\begin{align*}
    (1-c) x^\top \Sigma_P x \leq x^\top \widehat{\Sigma}_P x \leq (1+c)x^\top \Sigma_P x 
\end{align*}
where $c = \max\left\{\frac{C\sqrt{d}+t}{\sqrt{\kappa}}, \left(\frac{C\sqrt{d}+t}{\sqrt{\kappa}}\right)^2\right\}$, $C = K_{\psi_2}^2\sqrt{\ln 9/c_1}$ and $c_1$ is an absolute constant.
\end{lemma}
Consequently for $\kappa\geq c_\delta := K_{\psi_2}^2(\sqrt{d\ln 9/c_1}+\sqrt{\frac{\log(2/\delta)}{c_1}})$, holds with probability at least $1-\delta$ 
\begin{align*}
    \left(1-\frac{c_\delta}{\sqrt{\kappa}}\right) x^\top \Sigma_P x \leq x^\top \widehat{\Sigma}_P x \leq \left(1+\frac{c_\delta}{\sqrt{\kappa}}\right)x^\top \Sigma_P x. 
\end{align*}
\begin{proof}
Let $A\in\R^{\kappa\times d}$ whose rows $A_i$ are independent sub-gaussian isotropic random vectors in $R^d$ and define $K_{\psi_2} = \max_i\|A_i\|_{\psi_2}$. We can apply Theorem 5.39 of \cite{vershynin2011introduction} on $A$ to have that with probability at least $1-2\exp(-c_1 t^2/K_{\psi_2}^4)$ holds 
\begin{align*}
    1 - \frac{C\sqrt{d}+t}{\sqrt{\kappa}}\leq \sigma_{\min}(A) \leq \sigma_{\max}(A) \leq 1 + \frac{C\sqrt{d}+t}{\sqrt{\kappa}},
\end{align*}
where $C = K_{\psi_2}^2\sqrt{\ln 9/c_1}$ and $c_1$ is an absolute constant.

With Lemma 5.36 of \cite{vershynin2011introduction}, this implies that with probability at least $1-2\exp(-c_0 t^2)$ holds 
\begin{align}\label{eq:matrix_concentration}
    \|A^\top A - I\| \leq \max\left\{\frac{C\sqrt{d}+t}{\sqrt{\kappa}}, \left(\frac{C\sqrt{d}+t}{\sqrt{\kappa}}\right)^2\right\} =: c
\end{align}
Recall $\Sigma_P = \E_{X\sim\nu}[P(X)XX^\top]$, so $Y = \sqrt{P(X)}\Sigma_P^{-1/2}X$ satisfies $\E[YY^\top] = \E[\Sigma_P^{-1/2}P(X)XX^\top\Sigma_P^{-1/2}] = \Sigma_P^{-1/2} \Sigma_P \Sigma_P^{-1/2} = I$. So we can apply \eqref{eq:matrix_concentration} to get $\|\Sigma_P^{-1/2} \widehat{\Sigma}_P \Sigma_P^{-1/2}-I\| \leq c$. Thus for any $y\in \R^d$, 
\begin{align*}
    1-c \leq \frac{y^\top}{\|y\|} \Sigma_P^{-1/2} \widehat{\Sigma}_P \Sigma_P^{-1/2}\frac{y}{\|y\|}\leq 1+c
\end{align*}
so setting $y = \Sigma_P^{1/2}x$
\begin{align*}
    (1-c) x^\top \Sigma_P x \leq x^\top \widehat{\Sigma}_P x \leq (1+c)x^\top \Sigma_P x.
\end{align*}
Also, the sub-gaussian bound becomes $K_{\psi_2} = \max_i\|\sqrt{P(\widetilde{x}_i)}\Sigma_P^{-1/2}\widetilde{x}_i\|_{\psi_2}$.
\end{proof}
\begin{proposition}[RIPS guarantees on empirical covariance matrix]\label{prop:rips_bound_approx_covariance}
Let $x_1,\dots,x_n$ and $\widetilde{x}_1,\dots,\widetilde{x}_\kappa$ be drawn IID from a distribution $\nu$. 
For $s=1,\dots,n$ , assume that $|\langle \theta,x_s \rangle| \leq B$ and $\E[ |\langle \theta,x_s \rangle-y_s|^2 ] \leq \sigma_\text{noise}^2$. For $s=1,\dots,\kappa$ , assume that $\E[ |\langle \theta,x_s \rangle-y_s|^2 ] \leq \sigma_\text{noise}^2$. Let $P \in [0,1]$ be arbitrary and let $Q_s(x_s) \sim \text{Bernoulli}(P)$ independently for all $s \in [n]$.
Let $\Sigma_P = \E_{X \sim \nu}[P(X) X X^\top]$ and $\widehat{\Sigma}_P = \frac{1}{\kappa}\sum_{s=1}^\kappa P(\widetilde{x}_s) \widetilde{x}_s \widetilde{x}_s^\top$. Assume that $\Sigma_P$ is invertible and that there exists $\gamma
\geq 0$ such that $(1-\gamma)\Sigma_P \preceq \widehat{\Sigma}_P \preceq (1+\gamma)\Sigma_P$.
For a given finite set $\mc{V} \subset \R^d$ define 
$$w_v = \mathrm{Catoni}( \{ \langle v, \widehat{\Sigma}_P^{-1} Q_s(x_s) x_s y_s \rangle \}_{s=1}^{n} ),$$
If $\widehat{\theta} = \arg\min_\theta \max_v \frac{|w_v - \langle \theta, v \rangle|}{\| v \|_{\widehat{\Sigma}_P^{-1}}}$ and $n\geq 4\log(2|\mc{V}|/\delta)$, then with probability at least $1-\delta$, it holds that
\begin{align*}
    |\langle v, \widehat{\theta}-\theta \rangle| \leq 4\left(\sqrt{\frac{B^2 + \sigma^2}{(1-\gamma)^2}}+\sqrt{n\gamma}\|\theta_*\|_{\E_{X \sim \nu}[X X^\top]}\right)\| v \|_{\E_{X \sim \nu}[n P(X) X X^\top]^{-1}}\sqrt{\log(2\abs{\mc{V}}/\delta)}
\end{align*}
\end{proposition}
We first state an intermediate matrix lemma before the proof of Proposition~\ref{prop:rips_bound_approx_covariance}.
\begin{lemma}\label{lmm:bound_matrix_norm_approx_covariance}
Assume that $\Sigma_P$ is invertible and that there exists $\gamma
\in [0, 1/2]$ such that $(1-\gamma)\Sigma_P \preceq \widehat{\Sigma}_P\preceq (1+\gamma)\Sigma_P$. Then for any $v\in\mc{V}$
\begin{align*}
    \|v\|^2_{\widehat{\Sigma}_P^{-1}\Sigma_P\widehat{\Sigma}_P^{-1}}\leq \frac{1}{(1-\gamma)^2}\|v\|^2_{\Sigma_P^{-1}}.
\end{align*}
and 
\begin{align*}
    \|v\|_{(I - \Sigma_P^{1/2}\widehat{\Sigma}_P^{-1}\Sigma_P^{1/2})^2} \leq \sqrt{1-\frac{2}{1+\gamma}+\frac{1}{(1-\gamma)^2}}\|v\|_2 \leq \sqrt{10\gamma}\|v\|_2.
\end{align*}
\end{lemma}
\begin{proof}
We know that taking the inverse of two ordered positive definite matrices will flip the order, so here
\begin{align*}
    \frac{1}{(1+\gamma)}\Sigma_P^{-1} \preceq \widehat{\Sigma}_P^{-1}\preceq  \frac{1}{(1-\gamma)}\Sigma_P^{-1}.
\end{align*}
$(1-\gamma)\Sigma_P \preceq \widehat{\Sigma}_P$ implies that for all $u\in\R^d$ holds $u^\top \Sigma_P u \leq 1/(1-\gamma)u^\top \widehat{\Sigma}_P u$. So taking $u = \widehat{\Sigma}_P^{-1}v$, we get $v^\top \widehat{\Sigma}_P^{-1}\Sigma_P\widehat{\Sigma}_P^{-1} v \leq 1/(1-\gamma)v^\top \widehat{\Sigma}_P^{-1} v$. Conclusion
\begin{align*}
    v^\top \widehat{\Sigma}_P^{-1}\Sigma_P\widehat{\Sigma}_P^{-1} v = \frac{1}{1-\gamma} v^\top \widehat{\Sigma}_P^{-1} v \leq \frac{1}{(1-\gamma)^2} v^\top \Sigma_P^{-1} v
\end{align*}
hence the first result of Lemma~\ref{lmm:bound_matrix_norm_approx_covariance}.\\
For the second one, we get
\begin{align*}
	\Norm{v}_{\Sp{I-\Sigma_P^{1/2}\hat{\Sigma}^{-1}_P\Sigma_P^{1/2}}^2}^2&=v^\top \Sp{I-\Sigma_P^{1/2}\hat{\Sigma}^{-1}_P\Sigma_P^{1/2}}^2v\\
	&=\Norm{v}_2^2-2v^\top \Sigma_P^{1/2}\hat{\Sigma}^{-1}_P\Sigma_P^{1/2}v+v^\top\Sigma_P^{1/2}\hat{\Sigma}^{-1}_P\Sigma_P\hat{\Sigma}_P^{-1}\Sigma_P^{1/2}v\\
	&\overset{\text{(i)}}{\leq} \Norm{v}_2^2-\frac{2}{1+\gamma}\Norm{v}_2^2+\frac{1}{1-\gamma}v^\top\Sigma_P^{1/2}\hat{\Sigma}_P^{-1}\Sigma_P^{1/2}v\\
	&\leq \Norm{v}_2^2-\frac{2}{1+\gamma}\Norm{v}_2^2+\frac{1}{\Sp{1-\gamma}^2}\Norm{v}_2^2\tag{Since $\hat{\Sigma}_P\preceq\frac{1}{1-\gamma}\Sigma_P$}\\
	&\leq \Sp{1-\frac{2}{1+\gamma}+\frac{1}{\Sp{1-\gamma}^2}}\Norm{v}_2^2\\
	&\overset{\text{(ii)}}{\leq}10\gamma\Norm{v}_2^2.
\end{align*}
The inequality (i) above holds because $\frac{1}{1+\gamma}\Sigma_P^{-1}\preceq\hat{\Sigma}_P^{-1}$ and $(1-\gamma)\Sigma_P\preceq\hat{\Sigma}_P\implies\Sigma_P\preceq\frac{1}{1-\gamma}\hat{\Sigma}_P$. The inequality (ii) above holds because for $\gamma\in\Mp{0, \frac{1}{2}}$, we have
\begin{align*}
    1-\frac{2}{1+\gamma}+\frac{1}{\Sp{1-\gamma}^2} \leq 1 - 2(1-\gamma) + (1+2\gamma)^2 \leq 10\gamma.
\end{align*}
Taking square root on both sides gives us the results.
\end{proof}
\begin{proof}[Proof of Proposition \ref{prop:rips_bound_approx_covariance}]
This proof is analogous to the proof of Proposition~\ref{prop:rips_bound}. We first note that 
\begin{align*}
     \max_{v \in \mc{V}} \frac{| \langle \widehat{\theta}, v \rangle - \langle \theta, v \rangle|}{\| v \|_{\widehat{\Sigma}_P^{-1}}} &= \max_{v \in \mc{V}} \frac{| \langle \widehat{\theta}, v \rangle - w_v + w_v - \langle \theta, v \rangle|}{\| v \|_{\widehat{\Sigma}_P^{-1}}} \\
     &\leq \max_{v \in \mc{V}} \frac{| \langle \widehat{\theta}, v \rangle - w_v|}{\| v \|_{\widehat{\Sigma}_P^{-1}}} + \max_{v \in \mc{V}} \frac{| w_v - \langle \theta, v \rangle|}{\| v \|_{\widehat{\Sigma}_P^{-1}}} \\
     &= \min_{\theta'} \max_{v \in \mc{V}} \frac{ | \langle \theta', v\rangle - w_v| }{\| v \|_{\widehat{\Sigma}_P^{-1}}} + \max_{v \in \mc{V}} \frac{| w_v - \langle \theta', v \rangle|}{\| v \|_{\widehat{\Sigma}_P^{-1}}} \\
     &\leq 2 \max_{v \in \mc{V}} \frac{ | \langle \theta, v \rangle - w_v| }{\| v \|_{\widehat{\Sigma}_P^{-1}}}
\end{align*}
So it suffices to show that each $| \langle \theta, v \rangle - w_v|$ is small. We begin by fixing some $v\in\mc{V}$ and bounding the variance of $v^\top \widehat{\Sigma}_P^{-1} Q_s(x_s)x_s y_s$ for any $s \leq n$ which is necessary to use the robust estimator. Note that
\begin{align*}
\mathbb{V}\text{ar}_{x_s \sim \nu, Q_s(x_s)\sim P(x_s)}( v^\top \widehat{\Sigma}_P^{-1} Q_s(x_s)x_s y_s ) = &\E_{x_s \sim \nu, Q_s(x_s)\sim P(x_s)}[ (v^\top \widehat{\Sigma}_P^{-1} Q_s(x_s)x_s y_s)^2 ]\\&\quad-\E_{x_s \sim \nu, Q_s(x_s)\sim P(x_s)}[ v^\top \widehat{\Sigma}_P^{-1} Q_s(x_s)x_s y_s ]^2 
\end{align*}
which means we can drop the second term to bound the variance by 
\begin{align*}
    &\E_{x_s \sim \nu, Q_s(x_s)\sim P(x_s)}[ \left( (v^\top \widehat{\Sigma}_P^{-1} Q_s(x_s)x_s y_s \right)^2 ]\\
    &= \E_{x_s \sim \nu, Q_s(x_s)\sim P(x_s)}[ \left( v^\top \widehat{\Sigma}_P^{-1} Q_s(x_s)x_s (x_s^\top \theta + \xi_s ) \right)^2 ] \\
    &= \E_{x_s \sim \nu}\left[\E_{Q_s(x_s)\sim P(s_s)}[ \left( v^\top \widehat{\Sigma}_P^{-1} Q_s(x_s)x_s (  x_s^\top \theta ) \right)^2 ] + \E_{Q_s(x_s)\sim P(s_s)}[ \left( v^\top \widehat{\Sigma}_P^{-1} Q_s(x_s)x_s \right)^2 \xi_t^2 ]\right] \\
    &\leq \E_{x_s \sim \nu}\left[ B^2 \E_{Q_s(x_s)\sim P(s_s)}[ \left( v^\top \widehat{\Sigma}_P^{-1} Q_s(x_s)x_s \right)^2 ] + \sigma^2 \E_{Q_s(x_s)\sim P(s_s)}[ \left( v^\top \widehat{\Sigma}_P^{-1} Q_s(x_s)x_s \right)^2 ]\right] \\
    &= \E_{x_s \sim \nu}\left[(B^2 + \sigma^2) \E_{Q_s(x_s)\sim P(s_s)}[ v^\top \widehat{\Sigma}_P^{-1} Q_s(x_s)x_s x_s^\top Q_s(x_s)\widehat{\Sigma}_P^{-1} v]\right] \\
    &= \E_{x_s \sim \nu}\left[(B^2 + \sigma^2) \E_{Q_s(x_s)\sim P(s_s)}[ v^\top \widehat{\Sigma}_P^{-1} Q_s(x_s)x_s x_s^\top \widehat{\Sigma}_P^{-1} v]\right] \\
    &\leq \E_{x_s \sim \nu}\left[(B^2 + \sigma^2) v^\top \widehat{\Sigma}_P^{-1} P(x_s) x_s x_s^\top \widehat{\Sigma}_P^{-1} v\right], 
\end{align*}
where we used that $Q_s^2(x_s) = Q_s(x_s)$. Thus, we have with Lemma~\ref{lmm:bound_matrix_norm_approx_covariance}
\begin{align*}
    \mathbb{V}\text{ar}( v^\top \widehat{\Sigma}_P^{-1} Q_s(x_s)x_s y_s )
    &\leq (B^2 + \sigma^2) v^\top \widehat{\Sigma}_P^{-1} \E_{x_s \sim \nu} [P(x_s) x_s x_s^\top] \widehat{\Sigma}_P^{-1} v \\ 
    &= (B^2 + \sigma^2)\|v\|^2_{\widehat{\Sigma}_P^{-1}\Sigma_P\widehat{\Sigma}_P^{-1}}\\
    &\leq \frac{B^2 + \sigma^2}{(1-\gamma)^2}\|v\|^2_{\Sigma_P^{-1}}.
\end{align*}
We have 
\begin{align*}
    | \langle \theta_*, v \rangle - w_v| 
    &= | \langle \theta_*, v \rangle - \E[v^\top \widehat{\Sigma}_P^{-1} P(x_1)x_1 y_1] + \E[v^\top \widehat{\Sigma}_P^{-1} P(x_1)x_1 y_1] - w_v| \\
    &\leq | \langle \theta_*, v \rangle - \E[v^\top \widehat{\Sigma}_P^{-1} P(x_1)x_1 y_1] | \\
    &\quad+ | \mathrm{Catoni}( \{ \langle v, \widehat{\Sigma}_P^{-1} Q_s(x_s) x_s y_s \rangle \}_{s=1}^{n} ) - \E_{X \sim \nu}[v^\top \widehat{\Sigma}_P^{-1} P(X)X Y]| .
\end{align*}
We now recall that we can write $y_t = x_t^\top \theta_* + \xi_t$ where $\xi_t$ is a mean-zero, independent random variable with variance at most $\sigma^2$. Thus, using Cauchy-Schwarz and applying Lemma~\ref{lmm:bound_matrix_norm_approx_covariance}, we get
\begin{align*}
    | \langle \theta_*, v \rangle - \E[v^\top \widehat{\Sigma}_P^{-1} P(x_1)x_1 y_1] | 
    &= |v^\top\theta_* - v^\top\widehat{\Sigma}_P^{-1}\Sigma_P\theta_*|\\ &= |v^\top (I - \widehat{\Sigma}_P^{-1}\Sigma_P)\theta_*|\\
    &=|v^\top\Sigma_P^{-1/2} (I - \Sigma_P^{1/2}\widehat{\Sigma}_P^{-1}\Sigma_P^{1/2})\Sigma_P^{1/2}\theta_*| \\
    &\leq \|\Sigma_P^{-1/2}v\|\;\|\Sigma_P^{1/2}\theta_*\|_{(I - \Sigma_P^{1/2}\widehat{\Sigma}_P^{-1}\Sigma_P^{1/2})^2}\\
    &\leq \sqrt{10\gamma}\|\Sigma_P^{-1/2}v\|\;\|\Sigma_P^{1/2}\theta_*\|\\
    &= \sqrt{10\gamma}\|v\|_{\Sigma_P^{-1}}\|\theta_*\|_{\Sigma_P}.
\end{align*}

By using the property of Catoni estimator stated in Definition~\ref{def:robust_estimator}, we have
\begin{align*}
&\abs{\inner{\theta_*, v}-w_v} \\
\leq&|\mathrm{Catoni}( \{ \langle v, \E_{X \sim \nu}[P(X) X X^\top]^{-1} Q_s(x_s) x_s y_s \rangle \}_{s=1}^n ) - \E[ \langle v, \E_{X \sim \nu}[P(X) X X^\top]^{-1} Q_s(x_s) x_s y_s \rangle ]| \\
&\quad\quad +\sqrt{10\gamma}\|\theta_*\|_{\E_{X \sim \nu}[X X^\top]}\|v\|_{(\E_{X \sim \nu}[P(X)X X^\top]^{-1}}\\
\leq & \sqrt{2} \sqrt{\Sp{\mathbb{V}\text{ar}(\langle v, \E_{X \sim \nu}[P(X) X X^\top]^{-1} Q_s(x_s) x_s y_s \rangle)}\frac{\log(\tfrac{2}{\delta})}{n/2}}\\
&\quad\quad + \sqrt{10\gamma}\|\theta_*\|_{\E_{X \sim \nu}[X X^\top]}\|v\|_{(\E_{X \sim \nu}[P(X)X X^\top]^{-1}}\tag{with probability at least $1-\delta$ if $n\geq 4\log(2/\delta)$}\\
\leq& \left(\sqrt{4}\sqrt{\frac{B^2 + \sigma^2}{(1-\gamma)^2}}+\sqrt{10n\gamma}\|\theta_*\|_{\E_{X \sim \nu}[X X^\top]}\right)\|v\|_{(\E_{X \sim \nu}[P(X) X X^\top]^{-1}}\sqrt{\frac{\log(\tfrac{2}{\delta})}{n}}\\
=&\left(\sqrt{4}\sqrt{\frac{B^2 + \sigma^2}{(1-\gamma)^2}}+\sqrt{10n\gamma}\|\theta_*\|_{\E_{X \sim \nu}[X X^\top]}\right)\| v \|_{\E_{X \sim \nu}[n P(X) X X^\top]^{-1}}\sqrt{\log(2/\delta)}.
\end{align*}
Finally, the proof is complete by taking union bounding over all $v\in \mc{V}$.
\end{proof}

\begin{proof}[Proof of Lemma~\ref{lmm:correctness_unknow_nu}]
Most of this proof is exactly the one of Section~\ref{sec:proof_upperbound} and Section~\ref{sec:appendix_correctness} so we only state the concentration bound. For any $\mc{V} \subseteq \mc{Z}$ and $z, z' \in \mc{V}$ define
\begin{align*}
\mc{E}_{z,z',\ell}( \mc{V} ) = \{ |\langle z-z', \widehat{\theta}_\ell(\mc{V}) - \theta_* \rangle|  \leq \epsilon_\ell \}
\end{align*}
where $\widehat{\theta}_\ell( \mc{V} )$ is the estimator that would be constructed by the algorithm at stage $\ell$ with $\mc{Z}_\ell = \mc{V}$. Naturally we want to apply Proposition~\ref{prop:rips_bound_approx_covariance} with $\tau$ labeled samples to obtain that $\mc{E}_{z,z',\ell}( \mc{V} )$ holds with probability at least $1-\frac{\delta}{ 2\ell^2 |\mc{Z}|^2}$. Note that as Lemma~\ref{lmm:P_property} gives $P(x)\geq \mu/3$ so $$\Sigma_P = \E_{X\sim\nu}[P(X)XX^\top] \geq \frac{\mu}{3}\E_{X\sim\nu}[XX^\top]
$$
$\Sigma_P$ is invertible.

Defining $\delta_0 := \frac{\delta}{ 4\ell^2 |\mc{Z}|^2}$ and setting $\kappa \geq 2c_{\delta_0}\max\{1, 20\|\theta_*\|^2_{\E_{X \sim \nu}[X X^\top]}\}$ where we recall that was defined $c_\delta = K_{\psi_2}^2(\sqrt{d\ln 9/c_1}+\sqrt{\frac{\log(2/\delta)}{c_1}})$, 
Lemma~\ref{lmm:bound_empiricov} leads to $$
\frac{c_{\delta_0}}{\kappa} \leq \frac{1}{2}\min\left\{1, \frac{1}{20\|\theta_*\|^2_{\E_{X \sim \nu}[X X^\top]}}\right\}$$
so that we can set $\gamma = c_{\delta_0}/(\tau\kappa)$ in the bound of Proposition~\ref{prop:rips_bound_approx_covariance} to get
$$
\sqrt{10\tau\gamma}\|\theta_*\|_{\E_{X \sim \nu}[X X^\top]} \leq \frac{1}{2}
$$
and 
$$
\sqrt{\frac{B^2 + \sigma^2}{(1-\gamma)^2}}\leq 2\sqrt{B^2 + \sigma^2}
$$
So for $\delta_0 = \frac{\delta}{ 4\ell^2 |\mc{Z}|^2}$ the event $\widetilde{\mc{E}}_{\text{cov}}$ defined as
\begin{align*}
    \widetilde{\mc{E}}_{\text{cov}} := \left\{\left(1-\frac{c_{\delta_0}}{\sqrt{\kappa}}\right) x^\top \Sigma_P x \leq x^\top \widehat{\Sigma}_P x \leq \left(1+\frac{c_{\delta_0}}{\sqrt{\kappa}}\right)x^\top \Sigma_P x \right\}. 
\end{align*} 
happen with probability at least $1-\delta_0$.

Now, let us for now condition on $\widetilde{\mc{E}}_{\text{cov}}$. For fixed $\mc{V} \subset \mc{Z}$ and $\ell \in \mathbb{N}$ we apply Proposition~\ref{prop:rips_bound_approx_covariance}, instantiating the arbitrary $P$ to $\widehat{P}_\ell$ (obtained with \textsc{OptimizeDesign}, recall Section~\ref{sec:unknown_nu}) so that with probability at least $1-\frac{\delta}{ 4\ell^2 |\mc{Z}|^2}$ we have that for any $z, z' \in \mc{V}$ holds that the event $\widetilde{\mc{E}}_{\text{RIPS}, z, z'}$ defined as
\begin{align*}
    \widetilde{\mc{E}}_{\text{RIPS}, z, z'} &:= \bigg\{|\langle z-z', \widehat{\theta}_\ell(\mc{V}) - \theta_* \rangle| \\
    &\qquad\leq 2\|z-z'\|_{\E_{X\sim\nu}[\tau \widehat{P}_\ell(X)XX^\top]^{-1}} \left(4\sqrt{B^2 + \sigma^2}+1\right)\sqrt{\log(4 \ell^2 |\mc{Z}|^2/\delta)} \bigg\}
\end{align*}
happen with probability at least $1-\delta_0$.


So with probability at least $1-\P(\widetilde{\mc{E}}_{\text{RIPS}, z, z'}^c) -\P(\widetilde{\mc{E}}_{\text{cov}}^c) \geq 1-\frac{\delta}{ 4\ell^2 |\mc{Z}|^2} -\frac{\delta}{ 4\ell^2 |\mc{Z}|^2} = 1-\frac{\delta}{ 2\ell^2 |\mc{Z}|^2}$, both events hold and we have that for any $z, z' \in \mc{V}$ holds
\begin{align*}
    |\langle z-z', \widehat{\theta}_\ell(\mc{V}) - \theta_* \rangle| 
    &\leq 2\|z-z'\|_{\E_{X\sim\nu}[\tau \widehat{P}_\ell(X)XX^\top]^{-1}} \left(4\sqrt{B^2 + \sigma^2}+1\right)\sqrt{\log(4 \ell^2 |\mc{Z}|^2/\delta)}  \\
    &\leq 2(1+\varepsilon) \left(4\sqrt{B^2 + \sigma^2}+1\right)\|z-z'\|_{\E_{X\sim\nu}[\tau \widehat{P}_\ell(X)XX^\top]^{-1}} \sqrt{\log(4 \ell^2 |\mc{Z}|^2/\delta)}\\
    &\leq \epsilon_\ell.
\end{align*}
where we used the property of $\widehat{P}_\ell$ as detailed in Section~\ref{sec:unknown_nu} to conclude.
\end{proof}
\begin{proof}[Proof of Theorem~\ref{thm:upper_bound_tau_input_unknown_nu}]
The total number of labels requested after $L$ rounds is equal to $\sum_{\ell=1}^L \sum_{t=(\ell-1)\tau+1}^{\ell \tau} \widehat{P}_\ell(x_t)$. Again by Freedman's inequality we have that
\begin{align*}
    \sum_{\ell=1}^L \sum_{t=(\ell-1)\tau+1}^{\ell \tau} \widehat{P}_\ell(x_t) \leq 2 \sum_{\ell=1}^L \tau \E_{X \sim \nu}[ \widehat{P}_\ell(X) | \mc{Z}_{\ell} ] + \log(1/\delta)
\end{align*}
From Theorem~\ref{theo:opt}, it holds for any $\ell$ that $\E_{X\sim\nu}[\widehat{P}_\ell(X)] \leq \E_{X\sim\nu}[\widetilde{P}_\ell(X)]+4\sqrt{\mu}$ where $\widetilde{P}_\ell$ is the optimal solution to problem~\eqref{equ:opt_mu_bound}. So now, for some $\widetilde{\tau}$, we want to relate $\E_{X\sim\nu}[\widetilde{\tau}\widetilde{P}_\ell(X)]$ to $\E_{X\sim\nu}[\tau P_\ell(X)]$ where $P_\ell$ is the solution of problem~\eqref{equ:opt_original}. To do so, we rewrite problem~\eqref{equ:opt_original} and problem~\eqref{equ:opt_mu_bound} as
\begin{equation}
    \label{equ:opt_original_rescaled}
    \begin{array}{rl}
        \min_{P} & \E_{X\sim\nu}\Mp{\tau P(X)} \\
        \text{subject to} & y^\top \E_{X\sim\nu}\Mp{\tau P(X)XX^\top}^{-1}y \leq c_\ell^2, \quad\forall y\in\mc{Y}_\ell,\\
        & 0 \leq \tau P(x) \leq \tau,\quad\forall x\in\mc{X}.
    \end{array}
\end{equation}
and
\begin{equation}
	\label{equ:opt_modified_rescaled}
	\begin{array}{rl}
        \min_{P} & \E_{X\sim\nu}\Mp{\widetilde{\tau} P(X)}  \\
        \text{subject to} & y^\top\E_{X\sim\nu}[\widetilde{\tau} P(X) X X^\top]^{-1}y \leq c_\ell^2,\quad\forall y\in\mc{Y}_\ell,\\
        & 0\leq \widetilde{\tau} P(x) \leq \widetilde{\tau}(1-\mu_b),\quad\forall x\in\mc{X}.
    \end{array}
\end{equation}

where problem~\eqref{equ:opt_original_rescaled} is equivalent to problem~\eqref{equ:opt_original} and problem~\eqref{equ:opt_modified_rescaled} is equivalent to problem~\eqref{equ:opt_mu_bound}. Thus taking $\widetilde{\tau} = \frac{\tau}{1-\mu_b}$, problem~\eqref{equ:opt_modified_rescaled} becomes
\begin{equation*}
	\begin{array}{rl}
        \min_{P} & \E_{X\sim\nu}\Mp{\frac{\tau}{1-\mu_b} P(X)}  \\
        \text{subject to} & y^\top\E_{X\sim\nu}[ \frac{\tau}{1-\mu_b} P(X) X X^\top]^{-1}y \leq c_\ell^2,\quad\forall y\in\mc{Y}_\ell,\\
        & 0\leq \frac{\tau}{1-\mu_b} P(x) \leq \tau,\quad\forall x\in\mc{X}.
    \end{array}
\end{equation*}
which, using $Q = \frac{P}{1-\mu_b}$ is equivalent to
\begin{equation}\label{equ:opt_modified_final}
	\begin{array}{rl}
        \min_{Q} & \E_{X\sim\nu}\Mp{\tau Q(X)}  \\
        \text{subject to} & y^\top\E_{X\sim\nu}[ \tau Q(X) X X^\top]^{-1}y \leq c_\ell^2,\quad\forall y\in\mc{Y}_\ell,\\
        & 0\leq \tau Q(x) \leq \tau,\quad\forall x\in\mc{X}.
    \end{array}
\end{equation}
And we can now see that \eqref{equ:opt_modified_final} and \eqref{equ:opt_original_rescaled} are the same optimization problem. And $Q_\ell^*$ the solution of \eqref{equ:opt_modified_final} is equal to $\frac{\widetilde{P}_\ell}{1-\mu_b}$. Thus the result $\E_{X\sim\nu}\Mp{\tilde{\tau} \widetilde{P}_\ell(X)} = \E_{X\sim\nu}\Mp{\tau P_\ell(X)}$.

Remains to bound $\sum_{\ell=1}^{L} \tau \E_{X \sim \nu}[ P_\ell(X)]$ where
\begin{align*}
    &\sum_{\ell=1}^{L} \tau \E_{X \sim \nu}[ P_\ell(X) | \mc{Z}_{\ell} ]  \\
    &= \sum_{\ell=1}^{L} \left[  \min_{P: \mc{X} \rightarrow [0,1]} \tau \E_{X \sim \nu}[P(X)]\quad\text{ subject to }\quad\max_{z,z' \in \mc{\mc{Z}_\ell}} \frac{\|z-z'\|_{\E_{X\sim\nu}[ \tau P(X) X X^\top]^{-1}}^2  }{\epsilon_\ell^2}\beta_{\delta,\ell} \leq  1 \right] ,
\end{align*}
where $\beta_{\delta,\ell}$ is defined in Section~\ref{sec:unknown_nu} as
$$\beta_{\delta,\ell} := 4(1+\varepsilon)^2 \left(4\sqrt{B^2 + \sigma^2}+1\right)^2\log(4 \ell^2 |\mc{Z}|^2/\delta).$$
As in the case where the distribution $\nu$ is known (Section~\ref{sec:proof_upperbound}), we use Lemma~\ref{lmm:signal_to_noise_bound} to bound $\max_{z,z' \in \mc{\mc{Z}_\ell}} \frac{\|z-z'\|_{\E_{X\sim\nu}[ \tau P(X) X X^\top]^{-1}}^2  }{\epsilon_\ell^2}\beta_{\delta,\ell}$ by $\max_{z \in \mc{Z} \setminus z_*} \frac{\|z-z_*\|_{\E_{X\sim\nu}[ \tau P(X) X X^\top]^{-1}}^2  }{\langle z - z_*,\theta_* \rangle^2 } 64\beta_{\delta, L}$.
Last, the reparameterization of Proposition~\ref{prop:reparameterization} also applies here.

In the unlabeled sample complexity, we get an additional $L\kappa = L\lceil 2 K_{\psi_2}^2(\sqrt{d\ln 9/c_1}+\sqrt{\frac{\log(2/\delta)}{c_1}})\max\{1, 20\|\theta_*\|_{\E_{X \sim \nu}[X X^\top]}\}\rceil$ term  from the estimation of the covariance matrix. 
Last, we get an additional $L(K+u)$, where $K$ and $u$ are such that
$$K\geq\widetilde{O}\Sp{\frac{\abs{\mc{Z}}^3\kappa(\Sigma)^2\Norm{\Lambda^*}_2^8M^{16}}{\beta^2\mu_b^6}}\cdot\Sp{\frac{1+\epsilon}{\epsilon}}^2,\quad 
u\geq\widetilde{O}\Sp{\frac{\kappa(\Sigma)^2\Norm{\Lambda^*}_2^6M^{16}}{\beta^2\mu_b^6}}\cdot\Sp{\frac{1+\epsilon}{\epsilon}}^2,$$
from the sample complexity of the subroutine.
\end{proof}

\section{Classification}\label{sec:classification}
In this section we adopt the implementation described in Section~\ref{sec:proof_upperbound}. As described in the text, given a distribution $\pi\in \Delta_{\mc{X}}$, and a class of hypothesis $\mc{H}$, we can reduce classification to linear bandits by setting $\theta^{\ast} = [\theta^{\ast}_x]_{x\in \Delta_{\mc{X}}}$ where $\theta^{\ast}_x = 2\eta(x)-1$, and $\mc{Z} := \{z^{(h)}\}_{h\in \mc{H}}\subset [0,1]^{|\mc{X}|}$ where $z^{(h)}_x = \pi(x)\1\{h(x) = 1\}$. With the quantities computed in Section~\ref{sec:active_classification}, we now prove Theorem~\ref{thm:upper_bound_tau_input_classification}.


\begin{proof}[Proof of Theorem~\ref{thm:upper_bound_tau_input_classification}]
We consider a slightly modified version of Algorithm~\ref{algo:best_arm_meta} where we stop at round $L$ where $L_{\epsilon} = \lceil\log_2(4/\epsilon)\rceil$ and return $\arg\max_{z^{(h)}\in \mc{Z}_{\ell}} \langle z^{(h)}, \widehat{\theta}_\ell \rangle$. By an identical analysis to that in the proof of Theorem 2, we are guaranteed that $h\in \mc{S}_{\ell}$, i.e. $R_{\nu}(h) - R_{\nu}(z^{\ast}) = \langle z^* - z,\theta_* \rangle \leq 4\epsilon_{\ell}$. In addition the analysis of the sample complexity given there immediately gives the first part of the theorem.

It remains to bound the sample complexity in terms of the disagreement coefficient. The total sample complexity is given by,
\begin{align*}
\sum_{\ell=1}^{L}  \left[ \min_{P: \mc{X} \rightarrow [0,1]} \tau \E_{X \sim \nu}[P(X)]\quad\text{ subject to }\quad \max_{z \in \mc{S_\ell}}\frac{ \|z-z_*\|_{\E_{X\sim\nu}[ \tau P(X) X X^\top]^{-1}}^2  }{\epsilon_\ell^2}\beta_{\delta}  \leq  1 \right]
\end{align*}
where we recall $\beta_{\delta} =  2048\log(2L^2|\mc{H}|/\delta)$ since we can take $B=1$ and $\sigma = 1$.

We recall the proof of Theorem 2. From the proof, we see that with probability greater than $1-\delta$, our sample complexity is obtained by summing up to round $L$ 
\begin{align*}
\sum_{\ell=1}^{L}  \left[ \min_{P: \mc{X} \rightarrow [0,1]} \tau \E_{X \sim \nu}[P(X)]\quad\text{ subject to }\quad \max_{z \in \mc{S_\ell}}\frac{ \|z-z_*\|_{\E_{X\sim\nu}[ \tau P(X) X X^\top]^{-1}}^2  }{\epsilon_\ell^2}\beta_{\delta}  \leq  1 \right]
\end{align*}
By proposition 2 this is equivalent to
\begin{align*}
\sum_{\ell=1}^{L}  \left[ \min_{\lambda\in \Delta_X} \rho_\ell(\lambda)\beta_{\delta}\quad\text{ subject to }\quad \left\|\frac{\lambda}{\nu}\right\|_{\infty}\rho_{\ell}(\lambda)\beta_{\delta} \leq  \tau \right],\text{ where } \rho_{\ell}(\lambda) := \max_{z \in \mc{S_\ell}}\frac{ \|z-z_*\|_{\E_{X\sim\lambda}[ X X^\top]^{-1}}^2  }{\epsilon_\ell^2}.
\end{align*}
Define 
\[A_\ell  = \{x \in \mc{X}:\exists h, h(x)\neq h^{\ast}(x), R_{\nu}(h) - R_{\nu}(h^{\ast})\leq 4\epsilon_{\ell}\}, \ell\leq L\]
and let $\displaystyle \lambda_{\ell} = \frac{\1\{x\in A_\ell\}\nu(x)}{\E[\1\{x\in A_{\ell}\}]}$, so $ \displaystyle\left\|\frac{\lambda}{\nu}\right\|_{\infty} = \frac{1}{\E[\1\{x\in A_i\}]}$. 

We first argue that $\lambda_{\ell}$ is feasible for the previous program. Note,
\begin{align*}
   \rho_{\ell}(\lambda_{\ell}) 
    &= \max_{h: R_{\nu}(h) - R_{\nu}(h^{\ast})\leq 4\epsilon_{\ell}}\frac{\E_{X\sim \nu}[\frac{\1\{h(x)\neq h^{\ast}(x)}{\lambda_{\ell}(x)/\nu(x)}\}]}{\epsilon_{\ell}^2}\\
    &\overset{\text{(i)}}{=} \E[\1\{x\in A_{\ell}\}] \max_{h:R_{\nu}(h) - R_{\nu}(h^{\ast})\leq 4\epsilon_{\ell}} \frac{\E_{X\sim \nu }[\1\{h(x) \neq h^{\ast}(x)\}]}{\epsilon_{\ell}^2}\\
    &\leq \E[\1\{x\in A_{\ell}\}] \max_{h:R_{\nu}(h) - R_{\nu}(h^{\ast})\leq 4\epsilon_{\ell}} \frac{16\E_{X\sim \nu }[\1\{h(x) \neq h^{\ast}(x)\}]}{\max\{\epsilon_\ell^2, (R_{\nu}(h) - R_{\nu}(h^{\ast}))^2\}}  \\
    &\leq \E[\1\{x\in A_{\ell}\}] \max_{h:R_{\nu}(h) - R_{\nu}(h^{\ast})\leq 4\epsilon_{\ell}} \frac{16\E_{X\sim \nu }[\1\{h(x) \neq h^{\ast}(x)\}]}{\max\{(4\epsilon_\ell)^2, (R_{\nu}(h) - R_{\nu}(h^{\ast}))^2\}}\\
    &\overset{\text{(ii)}}{\leq} \E[\1\{x\in A_{\ell}\}] \max_{h:R_{\nu}(h) - R_{\nu}(h^{\ast})\leq 4\epsilon_{\ell}} \frac{16\E_{X\sim \nu }[\1\{h(x) \neq h^{\ast}(x)\}]}{\max\{\epsilon^2, (R_{\nu}(h) - R_{\nu}(h^{\ast}))^2\}}\\
    &\leq \E[\1\{x\in A_{\ell}\}] \max_{h\in H} \frac{16\E_{X\sim \nu }[\1\{h(x) \neq h^{\ast}(x)\}]}{\max\{\epsilon^2, (R_{\nu}(h) - R_{\nu}(h^{\ast}))^2\}}\\
    &\leq 16\E[\1\{x\in A_{\ell}\}]\rho(\nu, \epsilon)
\end{align*}
where the equality (i) holds because the following is true when we only consider $h$ such that $R_{\nu}(h) - R_{\nu}(h^{\ast})\leq 4\epsilon_{\ell}$
\[\frac{\1\{h(x)\neq h^{\ast}(x)\}}{\1\{x:\exists h, h(x)\neq h^{\ast}(x), (R_{\nu}(h) - R_{\nu}(h^{\ast}))\leq 4\epsilon_{\ell}\}} = \1\{h(x)\neq h^{\ast}(x)\}.\]
The inequality (ii) above is true because $4\epsilon_{\ell} \geq \epsilon$. Thus we see that $\rho_{\ell}(\lambda_{\ell})\|\lambda/\nu\|_{\infty} \beta_{\delta}\leq 16\rho(\nu, \epsilon)\beta_{\delta} \leq \tau$. It remains to argue about the disagreement coefficient. Firstly note that for any $h$ such that $R_{\nu}(h) - R_{\nu}(h^{\ast}) \leq 4\epsilon_{\ell}$.
\begin{align}\label{eq:01loss}
    d_{\nu}(h, h^{\ast}) = \E_{X\sim \nu}[\1\{h(X)\neq h^{\ast}(X)\}]
    &\leq \E_{X\sim \nu}[\1\{h(X)\neq Y\}] + \E_{X\sim \nu}[\1\{h^{\ast}(X)\neq Y\}]\\
    &\leq R_{\nu}(h) + R_{\nu}(h^{\ast})\\
    &\leq 2R_{\nu}(h^{\ast}) + 4\epsilon_{\ell}
\end{align}

Using this we see that,  
\begin{align*}
    \min_{\lambda\in \Delta} \rho_{\ell}(\lambda) \text{  subject to  }&\rho_{\ell}(\lambda)\|\lambda/\nu\|_{\infty}\beta_{\delta}\leq \tau \\
    &\leq \rho_{\ell}(\lambda_{\ell})\beta_{\delta} \tag{since $\lambda_{\ell}$ is feasible.}\\
    &\leq  \E[\1\{x\in A_{\ell}\}]\max_{h:R_{\nu}(h) - R_{\nu}(h^{\ast})\leq 4\epsilon_{\ell}} \frac{\E_{X\sim \nu }[\1\{h(x) \neq h^{\ast}(x)\}]}{\epsilon_{\ell}^2}\beta_{\delta}\tag{imitating the above computation}\\
    &\leq \frac{(2R(h^{\ast})+4\epsilon_{\ell})\E_{X\sim \nu}[\1\{\exists h: h(X)\neq h^{\ast}(X), d_\nu(h, h^{\ast}) \leq 2R(h^{\ast})+4\epsilon_{\ell}\}]}{\epsilon_{\ell}^2}\beta_{\delta}\tag{Equation~\eqref{eq:01loss}}\\
    &\leq \beta_{\delta}\begin{cases} \frac{9R(h^{\ast})^2}{\epsilon_\ell^2} \frac{\E_{X\sim \nu}[\1\{\exists h: h(X)\neq h^{\ast}(X), d_{\nu}(h, h^{\ast}) \leq 2R(h^{\ast})+4\epsilon_{\ell}\}]}{2R(h^{\ast})+4\epsilon_{\ell}} & 4\epsilon_{\ell}\leq R(h^{\ast})\\ \frac{144\E_{X\sim \nu}[\1\{\exists h: h(X)\neq h^{\ast}(X), d_{\nu}(h, h^{\ast}) \leq 2R(h^{\ast})+4\epsilon_{\ell}\}]}{2R(h^{\ast})+4\epsilon_{\ell}}& 4\epsilon_{\ell} > R(h^{\ast})\end{cases}\\
    &\leq \left(\frac{9R(h^{\ast})^2}{\epsilon_{\ell}^2}+144\right)\frac{\E_{X\sim \nu}[\1\{\exists h: h(X)\neq h^{\ast}(X), d_{\nu}(h, h^{\ast}) \leq 2R(h^{\ast})+4\epsilon_{\ell}\}]}{2R(h^{\ast})+4\epsilon_{\ell}}\beta_{\delta}
\end{align*}
Thus, 
\begin{align*}
    &\sum_{\ell=1}^{L}  \left[ \min_{\lambda\in \Delta_X} \rho_\ell(\lambda)\beta_{\delta}\quad\text{ subject to }\quad \left\|\frac{\lambda}{\nu}\right\|_{\infty}\rho_{\ell}(\lambda) \beta_{\delta}\leq  \tau \right]\\
    &\leq \sum_{\ell=1}^L \rho_{\ell}(\lambda_{\ell})\beta_{\delta}\\
    &\leq \sum_{\ell=1}^L \left(\frac{9R(h^{\ast})^2}{\epsilon_{\ell}^2}+144\right)\frac{\E_{X\sim \nu}[\1\{\exists h: h(X)\neq h^{\ast}(X), d_{\nu}(h, h^{\ast}) \leq 2R(h^{\ast})+4\epsilon_{\ell}\}]}{2R(h^{\ast})+4\epsilon_{\ell}}\beta_{\delta}\\
    &\leq \log_2\left(\frac{4}{\epsilon}\right)\sup_{\ell\leq L} \left(\frac{9R(h^{\ast})^2}{\epsilon_{\ell}^2}+144\right)\frac{\E_{X\sim \nu}[\1\{\exists h: h(X)\neq h^{\ast}(X), d_{\nu}(h, h^{\ast}) \leq 2R(h^{\ast})+\epsilon_{\ell}\}]}{2R(h^{\ast})+4\epsilon_{\ell}}\beta_{\delta}\\
    &\leq \log_2\left(\frac{4}{\epsilon}\right) \left(\frac{36R(h^{\ast})^2}{\epsilon^2}+144\right)\sup_{\ell\leq L}\frac{\E_{X\sim \nu}[\1\{\exists h: h(X)\neq h^{\ast}(X), d_{\nu}(h, h^{\ast}) \leq 2R(h^{\ast})+4\epsilon_{\ell}\}]}{2R(h^{\ast})+4\epsilon_{\ell}}\beta_{\delta}\\
    &\leq 36\log_2\left(\frac{4}{\epsilon}\right) \left(\frac{R(h^{\ast})^2}{\epsilon^2}+4\right)\sup_{\xi\geq \epsilon }\theta^{\ast}(2R(h^{\ast})+\xi, \nu)\beta_{\delta}
\end{align*}
from which the result follows.

\end{proof}


\end{document}